\documentclass[preprint,10pt]{elsarticle}
\usepackage{subcaption}
\usepackage{amsmath,amsfonts}
\usepackage{amssymb}
\usepackage{algorithm}
\usepackage{algorithmic}
\usepackage{array}
\usepackage{textcomp}
\usepackage{stfloats}
\usepackage{url}
\usepackage{verbatim}
\usepackage{graphicx}
\usepackage{float}
\usepackage{bm}
\usepackage{amsthm}
\usepackage{booktabs}
\usepackage{tikz}
\usetikzlibrary{arrows,positioning}
\usetikzlibrary{calc}
\usetikzlibrary{fit,backgrounds}
\usepackage{setspace}
\usepackage{siunitx}
\theoremstyle{remark}

\theoremstyle{}
\newtheorem{theorem}{\textbf{Theorem}}
\newtheorem{lemma}{\textbf{Lemma}}

\newtheorem{proposition}{\textbf{Proposition}}
\usepackage{color}

\newcommand{\tr}{{\rm tr}}
\newcommand{\diag}{{\rm diag}}

\definecolor{red}{rgb}{1,0,0}
\definecolor{green}{rgb}{0,1,0}
\definecolor{blue}{rgb}{0,0,1}
\definecolor{newcolor}{rgb}{0.5,0,1}

\usepackage{epstopdf}
\allowdisplaybreaks[4]

\AtBeginDocument{
    \setlength{\abovedisplayskip}{9pt} 
    \setlength{\belowdisplayskip}{9pt} 
}

\tikzset{
    boxE/.style={draw,rounded corners=2pt,fill=cyan!25,minimum width=1.4cm,minimum height=0.9cm},
    boxC/.style={draw,rounded corners=2pt,fill=orange!25,minimum width=1.4cm,minimum height=0.9cm},
    plus/.style={circle,draw,inner sep=2pt,minimum size=6pt},
    arrow/.style={->,thick}
}

\begin{document}

\begin{frontmatter}
\title{Lightweight Adaptive ReduNet via Hyperspherical Manifold Learning}

\author{Zhenglin Huang}
\ead{zlhuang@my.swjtu.edu.cn}
\author{Qifa Yan\corref{cor1}}
\ead{qifayan@swjtu.edu.cn}
\author{Bin Dai}
\ead{daibin@swjtu.edu.cn}
\author{Xiaohu Tang}
\ead{xhutang@swjtu.edu.cn}

\cortext[cor1]{Corresponding author}

\address{
School of Information Science and Technology,
Southwest Jiaotong University,
Chengdu 610031, China\\
and Information Coding and Transmission Key Laboratory of Sichuan Province,
CSNMT Int. Coop. Res. Centre (MoST),
Southwest Jiaotong University,
Chengdu 611756, China
}





\begin{abstract}
In recent years, a white-box neural network called ReduNet has been proposed, which employs the maximal coding rate reduction (MCR$^2$) principle to transform raw data into low-dimensional discriminative features via a forward layer-wise construction process. Unlike traditional deep networks that rely on backpropagation, ReduNet explicitly derives the parameters of each layer from the features of its preceding layer, offering a mathematically interpretable paradigm. However, this layer-wise construction often requires a large number of layers for the MCR$^2$ objective to reach a stable value, which increases the parameter storage of the unfolded module. To address this issue, we propose LA-ReduNet, a lightweight adaptive architecture that refines the layer-wise update rule and enables discriminative feature representations to be obtained with substantially fewer unfolded layers. Specifically, LA-ReduNet employs hyperspherical manifold learning and adaptive step sizes, thereby reducing by an order of magnitude the number of layers required for the MCR$^2$ objective to reach a stable value. Simulation results demonstrate that, while maintaining comparable classification accuracy, LA-ReduNet requires significantly fewer layers for the MCR$^2$ objective to reach a stable value. Remarkably, under the considered experimental settings, LA-ReduNet requires only approximately $1/29$ of the parameter storage of the unfolded ReduNet module for the MCR$^2$ objective to reach a stable value.
\end{abstract}


\end{frontmatter}

\section{Introduction}
The rapid advancement of artificial intelligence (AI) promotes research on applying AI to other fields \cite{hassija2024interpreting}. A major milestone in this development came in 2016, when the AI program AlphaGo \cite{silver2016mastering, sharma2021role} defeated Lee Sedol, the world champion of Go. Currently, deep learning-based algorithms are applied to a wide range of fields, such as autonomous driving \cite{zhao2025survey}, disease detection \cite{rabie2025review}, image classification \cite{ayyar2021review}, face recognition \cite{amirgaliyev2025review} and semantic communication \cite{getu2025semantic}. However, in certain scenarios, such as autonomous driving, higher interpretability is required for the model. Therefore, research on Explainable Artificial Intelligence (XAI) has attracted widespread attention and developed rapidly \cite{hassija2024interpreting}. Although research on black-box neural networks has made remarkable progress \cite{ribeiro2016should, selvaraju2017grad, ghorbani2019towards}, their interpretability remains limited.

In this context, Chan et al. proposed a novel white-box neural network named ReduNet \cite{chan2022redunet}, providing a perspective distinct from conventional approaches. In contrast to conventional methods that directly interpret the internal mechanisms of black-box neural networks, ReduNet derives the network architecture from the optimization objective, thereby reducing reliance on manual empirical design of the network architecture. More specifically, the network architecture of ReduNet is derived from the principle of Maximal Coding Rate Reduction (MCR$^2$). Different from the traditional cross-entropy (CE) loss function, the principle of MCR$^2$ revises the learning objective to explicitly capture the low-dimensional structures underlying high-dimensional data, rather than primarily focusing on label fitting \cite{yu2020learning}. The MCR$^2$ principle utilizes the rate distortion to measure the compactness of representations. The MCR$^2$ principle aims to maximize the overall coding rate of features while minimizing the within-class coding rate, where the coding rate is computed via multivariate Gaussian rate-distortion (RD) function.

Existing studies on ReduNet can be broadly divided into two categories. The first category focuses on improving ReduNet, such as ESS-ReduNet \cite{yu2024ess}, Multi-ReduNet \cite{limulti} and AR-ReduNet \cite{huang2025simple}. The second direction is to apply ReduNet to other domains, such as SAR target recognition \cite{zhao2024interpretable}, radar jamming recognition \cite{zhu2024updateable} and chromosome classification \cite{zhang2025msd}. Despite these advances, the layer-wise optimization in ReduNet still relies on a Euclidean gradient update followed by normalization. As a result, a fixed step size in the Euclidean space does not directly control the actual angular displacement of the features on the sphere, which may result in small or highly variable angular updates across samples. Motivated by this observation, we propose a lightweight ReduNet with adaptive step sizes based on hyperspherical manifold learning, termed lightweight adaptive ReduNet (LA-ReduNet). Compared with the gradient-ascent-based ReduNet, LA-ReduNet substantially reduces the required number of unfolded layers and, correspondingly, the parameter storage of the unfolded module. Simulation results show that LA-ReduNet not only achieves a more lightweight network architecture but also outperforms gradient-ascent-based ReduNet in classification accuracy.

In particular, LA-ReduNet adopts the same adaptive multivariate Gaussian RD approximation function as AR-ReduNet. AR-ReduNet improves ReduNet by introducing this adaptive approximation function, which provides a more accurate approximation of the coding rate and improves classification performance \cite{huang2025simple}. Rather than further modifying the coding-rate approximation, LA-ReduNet focuses on redesigning the layer-wise update rule of ReduNet under the unit-sphere constraint. Accordingly, AR-ReduNet is included as a baseline in the subsequent comparisons. It should be noted that Riemannian optimization on spheres is well established \cite{absil2008optimization, boumal2023intromanifolds}. Therefore, the novelty of LA-ReduNet does not lie in the use of tangent-space or geodesic updates themselves. Instead, we develop a normalized, truncated, and sample-adaptive Riemannian update tailored to the forward layer-wise MCR$^2$ construction, where each update is directly unfolded into a network layer. The main contributions of this paper are summarized as follows:
\begin{enumerate}
\item LA-ReduNet introduces a redesigned layer-wise update rule for ReduNet under the unit-sphere constraint. Specifically, we observe that the Euclidean update followed by normalization in ReduNet does not directly control the actual angular displacement of each feature on the unit sphere. Accordingly, a truncated and normalized Riemannian update scheme is constructed, in which the feasible update direction is explicitly determined in the tangent space, while the angular step size is adaptively adjusted for each sample according to the cosine similarity between its Euclidean gradient and radial direction. In addition, samples with sufficiently small Riemannian update direction norms are excluded from further updates through a thresholding mechanism. The resulting update scheme is naturally compatible with the layer-wise unfolding of ReduNet and can therefore be directly implemented as a sequence of network layers.

\item Theoretical properties of the proposed Riemannian update scheme are analyzed. In particular, the Lipschitz continuity of the Riemannian update mapping is proved, and a lower-bound inequality for the increment of the MCR$^2$ objective is derived on the product of unit spheres. Building on these results, the finite-termination property of the proposed algorithm is further established under the specified threshold and step-size conditions.

\item Simulation results on the CIFAR-10, CIFAR-100, and CINIC-10 datasets demonstrate that LA-ReduNet requires substantially fewer unfolded layers for both the MCR$^2$ objective and classification accuracy to stabilize. In particular, LA-ReduNet achieves classification-accuracy convergence within approximately 5--10 layers. Moreover, in our experiments, with ReduNet using its original step-size setting and LA-ReduNet adopting a relatively small base step size, LA-ReduNet requires substantially fewer layers for the MCR$^2$ objective to reach a stable value. Specifically, LA-ReduNet achieves such objective convergence in approximately $35$ layers, compared with approximately $1000$ layers for ReduNet under the considered settings, corresponding to approximately $1/29$ of the parameter storage required by ReduNet.

\end{enumerate}

The remainder of this paper is organized as follows: Section 2 reviews the necessary preliminaries of the MCR$^2$ principle, ReduNet, and AR-ReduNet. Section 3 introduces LA-ReduNet and establishes its finite-termination property. Section 4 presents simulation results on multiple datasets, and Section 5 concludes this paper.

\section{Preliminaries}\label{sec:preliminaries}
This section provides a brief introduction to MCR$^2$ \cite{yu2020learning}, ReduNet based on MCR$^2$ principle \cite{chan2022redunet}, and its improved version AR-ReduNet \cite{huang2025simple}.

\subsection{MCR$^2$ Principle and ReduNet Framework}
Consider a sample set $ \bm{X}=[\bm x_{1}, \bm x_{2}, \ldots, \bm x_{m}] \in \mathbb{R}^{n \times m} $, $ \bm{x} \in \mathbb{R}^{n}$ is a sample point. Let $\bm{z}_{i}$ be a transformation of feature $ \bm{x}_{i} $. For the set $ \bm{X} $, $ \bm Z=[\bm z_{1}, \bm z_{2}, ..., \bm z_{m}] \in \mathbb{R}^{n \times m} $ is the feature matrix. According to the MCR$^2$ principle \cite{yu2020learning}, $\bm Z$ is updated by optimizing the following problem:
\begin{align}
&\underset{\bm{Z}}{\rm{maximize}} \quad \Delta R(\bm{Z},  \epsilon, \bm{\Pi})=R(\bm{Z}, \epsilon)- \sum_{j=1}^k R^{\rm c}(\bm{Z}, \epsilon | \bm{\Pi}_j), \label{MCR2}\\
&{\rm{s. t.}}\quad \bm z_1,\ldots,\bm z_m\in\mathbb{S}^{n-1},\notag
\end{align}
where $\mathbb{S}^{n-1}$ denotes the unit sphere in $n$-dimensional space. The notations in \eqref{MCR2} are explained as follows:
    \begin{enumerate}
    \item The memberships of the samples are depicted by a set of $k$ diagonal matrices $\bm{\Pi}=\{\bm{\Pi}_j\}_{j=1}^k$, where $\bm{\Pi}_j$  is the membership matrix of class $j$ defined by\footnote{By definition, the diagonal matrices $\bm{\Pi}=\{\bm{\Pi}_j\}_{j=1}^k$ lie in a simplex $\{\bm{\Pi}:\pi_{ij}\geq 0, \sum_{j=1}^k\bm{\Pi}_j=\bm{I}\}$.}
    \begin{align}
    \bm{\Pi}_j=\diag(\pi_{1,j}, \pi_{2,j}, \ldots, \pi_{m,j})\in \mathbb{R}^{m\times m},
    \end{align}
    with $\pi_{i,j}$ being the label of the $i$-th sample, i.e.,
    \begin{align}
    \pi_{i,j}=\Bigg\{\begin{array}{ll}
    1,&\mbox{if $\bm x_i$ is in class $j$}\\
    0,&\mbox{else}  \label{Pi:training}
    \end{array}.
    \end{align}
    \item The function $R(\bm{Z}, \epsilon)$ is the minimal number of binary bits needed to encode $ \bm{Z} $ such that the expected decoding error is less than $ \epsilon^2 $, and $R^{\rm c}(\bm{Z}, \epsilon|\bm{\Pi}_j)$ is the sum of the minimal number of binary bits for each class. In particular, consider a vector source $ \bm z \in \mathbb{R}^{n} $ from a zero-mean multivariate Gaussian distribution $\bm z \sim \mathcal{N}(0, \bm{\Sigma})$, an approximation of the number of binary bits needed to encode $ \bm Z $ is given by \cite{ma2007segmentation}
    \begin{align}
    R(D) \triangleq \frac{1}{2}\log \det \Big(\bm{I}+\frac{n}{D} \bm{\Sigma} \Big),\label{RD_function}
    \end{align}

    In the objective function \eqref{MCR2}, $R(\bm{Z},\epsilon)$ is the minimum coding length of the whole set of features at distortion $D=\epsilon^2$, approximated by
    \begin{align}
    R(\bm{Z},\epsilon) \triangleq \frac{1}{2}\log \det \Big(\bm{I}+\frac{n}{m\epsilon^2} \bm{ZZ}^{\rm T} \Big),\label{Re}
    \end{align}
    where $ \bm \Sigma $ is replaced by its estimate $ \bm{ZZ}^{\rm T} / m $. While the minimum coding length of the $j$-th class is approximated by
    \begin{align}
    R^{\rm c}(\bm{Z}, \epsilon|\bm{\Pi}_j) \triangleq \frac{\tr(\bm{\Pi}_j)}{2m}\log \det\left(\bm{I}+\frac{n}{\tr(\bm{\Pi}_{j})\epsilon^{2}}\bm{Z}\bm{\Pi}_{j}\bm{Z}^{\rm T}\right).\label{Rc}
    \end{align}

    That is, the approximation of the covariance matrix of the $j$-th class is $ \bm{Z}\bm{\Pi}_{j}\bm{Z}^{\rm T} / \tr(\bm{\Pi}_j) $, and the weight of the $j$-th class is $\tr(\bm{\Pi}_j)/m$.
\end{enumerate}

\begin{figure}[t]
\resizebox{\linewidth}{!}{
\begin{tikzpicture}[
    line width=1.0pt,
    box/.style={
        draw,
        align=center,
        minimum height=28pt,
        inner xsep=10pt,
        inner ysep=6pt,
        rounded corners
    },
    greenbox/.style={
        box,
        fill=lime!30
    },
    bluebox/.style={
        box,
        fill=cyan!25
    },
    redbox/.style={
        box,
        fill=red!25
    },
    yellowcirc/.style={
        circle,
        draw,
        fill=yellow!35,
        inner sep=4pt,
        line width=1.0pt
    }
]


\node[greenbox] (zl)
{$\bm{Z}^{(\ell-1)}$};

\coordinate (split) at ($(zl.east)+(0.9,0)$);

\draw (zl.east) -- (split);

\coordinate (operatorleft) at ($(split)+(0.8,0)$);


\node[
    bluebox,
    anchor=west
] (E) at ($(operatorleft)+(0,3.6)$)
{$\bm{E}^{(\ell)}$};

\node[
    bluebox,
    anchor=west
] (CPi1) at ($(operatorleft)+(0,2.1)$)
{$\bm{C}_{1}^{(\ell)},\,
 \bm{\Pi}_{1}^{(\ell-1)}$};

\node[
    bluebox,
    anchor=west
] (CPi2) at ($(operatorleft)+(0,0.7)$)
{$\bm{C}_{2}^{(\ell)},\,
 \bm{\Pi}_{2}^{(\ell-1)}$};

\node[
    anchor=center
] (CPidots) at ($(CPi2.center)+(0,-0.9)$)
{$\vdots$};

\node[
    bluebox,
    anchor=west
] (CPik) at ($(operatorleft)+(0,-2.1)$)
{$\bm{C}_{k}^{(\ell)},\,
 \bm{\Pi}_{k}^{(\ell-1)}$};

\node[
    bluebox,
    anchor=west
] (I) at ($(operatorleft)+(0,-3.6)$)
{$\bm{I}$};


\draw[->] (split) |- (E.west);

\draw[->] (split) |- (CPi1.west);
\draw[->] (split) |- (CPi2.west);
\draw[->] (split) |- (CPik.west);

\draw[->] (split) |- (I.west);


\coordinate (mergeline) at ($(CPi1.east)+(1.0,0)$);

\coordinate (merge1) at (mergeline |- CPi1.east);
\coordinate (merge2) at (mergeline |- CPi2.east);
\coordinate (mergek) at (mergeline |- CPik.east);

\draw (CPi1.east) -- (merge1);
\draw (CPi2.east) -- (merge2);
\draw (CPik.east) -- (mergek);

\draw (merge1) -- (mergek);


\node[
    yellowcirc,
    right=1.2cm of merge2
] (minus)
{$\bm{-}$};

\node[
    yellowcirc,
    right=1.2cm of minus
] (plus)
{$\bm{+}$};

\draw[->]
    (merge2)
    --
    node[midway,above] {$\eta$}
    (minus.west);

\draw[->]
    (minus.east)
    --
    (plus.west);


\draw[->]
    (E.east)
    --
    node[midway,above] {$\eta$}
    (plus |- E.east)
    --
    (plus.north);


\draw[->]
    (I.east)
    --
    (plus |- I.east)
    --
    (plus.south);


\node[
    redbox,
    right=1.3cm of plus
] (P)
{$\mathcal{P}_{\mathbb{S}^{n-1}}$};

\draw[->]
    (plus.east)
    --
    (P.west);

\node[
    greenbox,
    right=1.3cm of P
] (znext)
{$\bm{Z}^{(\ell)}$};

\draw[->]
    (P.east)
    --
    (znext.west);


\node[
    above=4pt of E
]
{Expansion Operator};

\node[
    above=7pt of CPik
]
{Compression Operators};

\node[
    below=4pt of I
]
{Identity Matrix};

\end{tikzpicture}
}
\caption{Single-Layer Structure of ReduNet.}
\label{fig:1}
\end{figure}
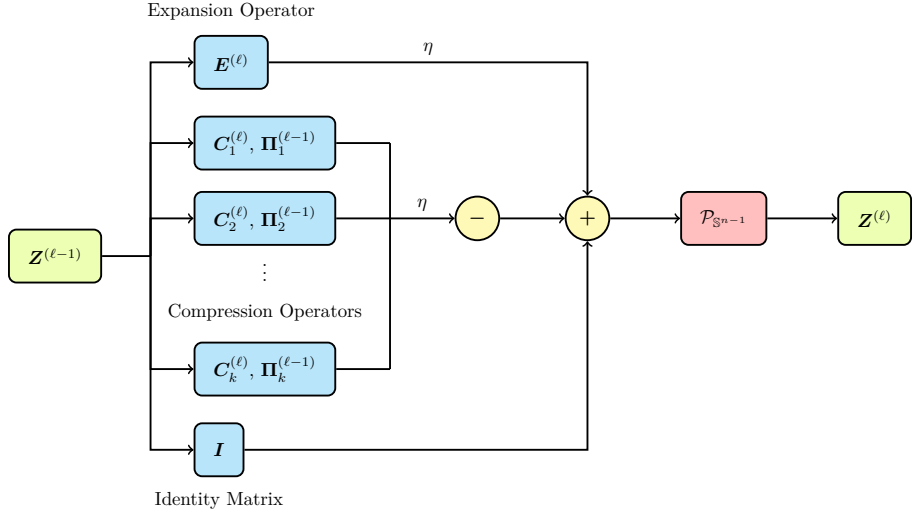

Based on the MCR$^2$ principle, the architecture of ReduNet can be directly derived from objective function \eqref{MCR2}. As illustrated in Fig. \ref{fig:1}, the feature matrix $ \bm{Z} $ is updated via gradient ascent. That is,
\begin{align}
\bm{Z}^{(\ell)}\propto \bm{Z}^{(\ell-1)}+\eta \bm{E}^{(\ell)} \bm{Z}^{(\ell-1)}-\eta\Big(\sum_{j=1}^k\bm{C}_{j}^{(\ell)}\bm{Z}^{(\ell-1)}\bm{\Pi}_j^{(\ell-1)}\Big),\notag\\
\quad {\rm{s.t.}}\quad \bm  z_1^{(\ell)},\ldots, \bm z_m^{(\ell)}\in\mathbb{S}^{n-1},
\end{align}
where $\eta$ is the step size, the matrices $\bm{E}^{(\ell)}$ and $\bm{C}_{j}^{(\ell)}$ are the values of $\bm{E}$ and $\bm{C}_j$ at the $\ell$-th iteration, and $\bm{E}$ and $ \bm{C}_{j} $ are defined by
\begin{align}
\bm{E}&= \frac{n}{m \epsilon^{2}} \left( \bm{I} + \frac{n}{m \epsilon^{2}} \bm{Z} \bm{Z}^{\rm T}\right)^{-1}, \label{cal:E} \\
\bm{C}_{j}&= \frac{n}{m \epsilon^{2}} \left( \bm{I} + \frac{n}{\tr(\bm{\Pi}_{j}) \epsilon^{2}} \bm{Z}\bm{\Pi}_{j}\bm{Z}^{\rm T} \right)^{-1}.  \label{cal:Cj}
\end{align}

In particular, the network is initialized by projecting the samples to $\mathbb{S}^{n-1}$, i.e.,
\begin{align}
\bm{Z}^{(0)}=[\bm z_1^{(0)},\ldots,\bm z_{m}^{(0)}], \quad {\rm s.t.}~\bm z_i^{(0)}=\frac{\bm x_{i}}{||\bm x_{i}||_2}\in\mathbb{S}^{n-1}, \quad i=1, \ldots,m.  \label{ip:Z}
\end{align}
Notice that the membership matrices $\bm{\Pi}_j=\diag(\pi_{1,j}, \pi_{2,j}, \ldots,\pi_{m,j})$ are known during training. However, during the testing phase, $\bm{\Pi}_j$ cannot be obtained from the labels of the test samples. Therefore, the matrix $\bm{\Pi}_j$
is predicted from the data \cite{chan2022redunet}. That is, $\bm{\Pi}_{j}=\diag(\hat{\pi}_{1,j}, \hat{\pi}_{2,j}, \ldots,\hat{\pi}_{m,j})$, where $\hat\pi_{i,j}$ is the estimation of $ \pi_{i,j} $. The estimation of the probability that the $i$-th sample belongs to the $j$-th class $\hat\pi_{i,j}$ is obtained from the softmax function,
\begin{align}
\hat{\pi}_{i,j}=\frac{\exp(-\lambda||\bm{C}_j\bm z_i||_2)}{\sum_{j=1}^k\exp(-\lambda||\bm{C}_j\bm z_i||_2)},
\end{align}
where $ \lambda $ is a hyperparameter that controls the uniformity.

\subsection{AR-ReduNet Based on Improved Approximation of Rate-Distortion Function}  \label{sec:alpha_RD}

Although the analytical expression of the multivariate Gaussian RD function can be derived via reverse water-filling, it remains computationally intractable for high-dimensional data. Therefore, ReduNet adopts approximation \eqref{RD_function} to accurately approximate the multivariate Gaussian RD function. As mentioned above, when the distortion $ D $ is large, the approximation error of \eqref{RD_function} becomes significant. To this end, AR-ReduNet employs a multivariate Gaussian RD approximation with an adaptive regularization term to estimate the inter-class and intra-class coding rates. This improved approximation function enhances the approximation by adding a regularization parameter $\alpha$ to \eqref{RD_function}, i.e.,
    \begin{align}
    R_{\alpha}(D)=\frac{1}{2}\log \det \left(\alpha \bm{I}+\frac{n}{D}\bm{\Sigma}\right), \label{eq:RD_alpha}
    \end{align}
where the parameter $ \alpha \in [0, 1]$ is chosen such that $R_{\alpha}(D)$ achieves 0 with the exact RD function at the same point $D=\tr(\bm{\Sigma})$, i.e., $\alpha=\alpha^*$, where $\alpha^*$ is the unique value satisfying $R_{\alpha^*}(\tr(\bm{\Sigma}))=0$ \cite{huang2025simple}.

Compared to ReduNet, AR-ReduNet computes the MCR$^2$ principle by using the approximation function \eqref{eq:RD_alpha}, which approaches the exact RD function more closely than function \eqref{RD_function}. The objective function of AR-ReduNet is defined as
    \begin{align}
    \Delta R(\bm{Z},  \epsilon,\bm{\Pi})=&\frac{1}{2}\log \det \left(\alpha \bm{I}+\frac{n}{m\epsilon^{2}}\bm{ZZ}^{\rm T}\right)   \notag \\
    &-\sum_{j=1}^k \frac{\tr(\bm{\Pi}_j)}{2m}\log \det\left(\alpha_{j} \bm{I}+\frac{n}{\tr(\bm{\Pi}_{j})\epsilon^{2}}\bm{Z}\bm{\Pi}_{j}\bm{Z}^{\rm T}\right),     \label{MCR2_alpha}
    \end{align}
the parameters $ \alpha $ and $ \alpha_j $ are updated according to
    \begin{subequations}
    \begin{align}
    \log \det \left(\alpha \bm{I}+\frac{n}{\tr(\bm{ZZ}^{\rm T})}\bm{ZZ}^{\rm T}\right)&=0, \label{compute:alpha} \\
    \log \det\left(\alpha_{j} \bm{I}+\frac{n}{\tr(\bm{Z}\bm{\Pi}_{j}\bm{Z}^{\rm T})}\bm{Z}\bm{\Pi}_{j}\bm{Z}^{\rm T}\right)&=0, \label{compute:alphaj}
    \end{align}
    \end{subequations}
where $ \alpha $ and $ \alpha_j $ are determined by $\bm{ZZ}^{\rm T}/\tr(\bm{ZZ}^{\rm T})$ and $\bm{Z}\bm{\Pi}_{j}\bm{Z}^{\rm T}/\tr(\bm{Z}\bm{\Pi}_{j}\bm{Z}^{\rm T})$, respectively. Since these matrices are updated at each iteration, $ \alpha $ and $ \alpha_j $ are re-evaluated accordingly. In addition, when updating the parameters $\bm{E}$ and $\bm{C}_j$ of AR-ReduNet, $ \alpha $ and $ \alpha_j $ are considered as constants.

Since LA-ReduNet also adopts the multivariate Gaussian RD approximation function \eqref{eq:RD_alpha} to compute the coding rate, we follow AR-ReduNet and employ binary search to solve for the parameters $ \alpha $ and $ \alpha_j $. As shown in Algorithm \ref{alg:bisection}, given an error threshold $ \delta $, the unique values of $ \alpha $ and $ \alpha_j $ satisfying the condition can be efficiently obtained.

\begin{algorithm}[t]
\setstretch{1.15}
\normalsize
\caption{BinarySearch $( \bm{\Sigma} , \delta )$}\label{alg:bisection}
\begin{algorithmic}[1]\label{alg}
\STATE $\alpha_{\rm L}\leftarrow 0, \alpha_{\rm R}\leftarrow 1$;
\STATE $ \alpha^{*}\leftarrow\frac{\alpha_{\rm L}+\alpha_{\rm R}}{2}$;
\WHILE {($ |R_{\alpha}(\tr(\bm{\Sigma})) | > \delta $)}
\IF{$ R_{\alpha}(\tr(\bm{\Sigma})) < 0 $}
\STATE $ \alpha_{\rm L} \leftarrow \alpha^{*} $;
\ELSE
\STATE $ \alpha_{\rm R} \leftarrow \alpha^{*} $;
\ENDIF
\STATE $\alpha^{*} \leftarrow \frac{\alpha_{\rm L} + \alpha_{\rm R}}{2}$;
\ENDWHILE
\RETURN $\alpha^{*} $;
\end{algorithmic}
\end{algorithm}

\section{Lightweight Adaptive ReduNet (LA-ReduNet)}\label{sec:method}
\subsection{MCR$^2$ Optimization Objective Based on Hyperspherical Manifold Learning}
By reviewing optimization objective \eqref{MCR2}, we consider an optimization problem on the unit sphere. In ReduNet or AR-ReduNet, projected gradient ascent is employed, with projection onto the unit-sphere guaranteeing that $\|\bm{z}\|_2=1$. Although the normalization operation enforces the unit-sphere constraint, the angular displacement induced by a fixed Euclidean update is controlled only indirectly and depends on both the radial and tangential components of the update. Therefore, we construct the update directly in the tangent space and evolve the feature $\bm{z}$ along the corresponding geodesic on the unit sphere.

Consider the feature vector $\bm{z}^{(\ell)}$ at the $\ell$-th layer. According to the optimization objective in \eqref{MCR2_alpha}, the implicit parameters $\alpha$ and $\{\alpha_j\}_{j=1}^{k}$ are first computed from the current features and then kept fixed when forming the update direction with respect to $\bm{z}^{(\ell)}$. Under this frozen-parameter setting, the Euclidean update direction can be decomposed into a radial component parallel to $\bm{z}^{(\ell)}$ and a tangential component lying in the tangent space of the unit sphere at $\bm{z}^{(\ell)}$, i.e.,
\begin{align}
\frac{\partial \Delta R}{\partial \bm{z}}\Big|_{\bm{z}=\bm{z}^{(\ell)}} = \bm{g} = \bm{g}_{\rm{T}}+ \bm{g}_{\rm{R}} \in \mathbb{R}^{n},
\end{align}
where $\bm{g}_{\rm T}$ and $\bm{g}_{\rm R}$ denote the tangential and radial components of the Euclidean update direction, respectively. $ \bm{g}_{\rm{T}} $ can be expressed as
    \begin{align}
    \bm{g}_{\rm{T}} \triangleq \bm{g}-(\bm{g}^{\rm T}\bm{z}^{(\ell)})\bm{z}^{(\ell)} \in \mathbb{R}^{n}.  \label{gt_formula}
    \end{align}

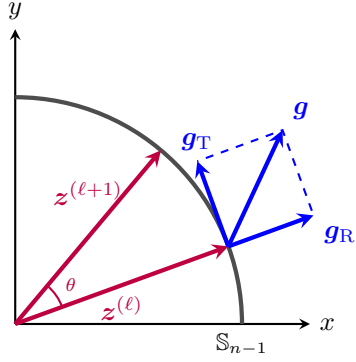
\begin{figure}
\centering
\begin{tikzpicture}[
    >=stealth,
    scale=3
]

\draw[->, thick] (0,0) -- (1.3,0) node[right] {$x$};
\draw[->, thick] (0,0) -- (0,1.3) node[above] {$y$};

\draw[ultra thick, black!70] (1,0) arc (0:90:1);
\node[below, font=\small] at (1,0) {$\mathbb{S}_{n-1}$};

\draw[->, ultra thick, purple] (0,0) -- (20:1)
    node[midway, below, purple] {$\bm{z}^{(\ell)}$};
\draw[->, ultra thick, purple] (0,0) -- (50:1)
    node[midway, above=8pt, purple] {$\bm{z}^{(\ell+1)}$};

\draw[thick, purple] (20:0.22)
    arc[start angle=20, end angle=50, radius=0.22];
\node[purple, font=\scriptsize] at (35:0.30) {$\theta$};

\coordinate (P) at (20:1);
\coordinate (gr) at ($(P)+(20:0.4)$);
\coordinate (gt) at ($(P)+(110:0.4)$);
\coordinate (g)  at ($(P)+(20:0.4)+(110:0.4)$);

\draw[->, ultra thick, blue] (P) -- (g)  node[above right] {$\bm{g}$};
\draw[->, ultra thick, blue] (P) -- (gr) node[below right] {$\bm{g}_{\rm{R}}$};
\draw[->, ultra thick, blue] (P) -- (gt) node[above] {$\bm{g}_{\rm{T}}$};

\draw[thick, dashed, blue] (P) -- (gr) -- (g) -- (gt);

\end{tikzpicture}
\caption{Orthogonal decomposition of the iterative update on the sphere $\mathbb{S}_{n-1}$.}
\label{fig:2}
\end{figure}

Next, by rotating through a small angle $ \theta $ at point $ \bm{z}^{(\ell)} $, the update of $ \bm{z}^{(\ell+1)} $ is given by
    \begin{align}
    \bm{z}^{(\ell+1)} = \cos\theta \cdot \bm{z}^{(\ell)} + \sin\theta \cdot \frac{\bm{g}_{\rm{T}}}{||\bm{g}_{\rm{T}}||_2},  \label{unit_update}
    \end{align}
for all $ \bm{g}_{\rm{T}} $ with $ ||\bm{g}_{\rm{T}}|| \neq 0 $. As illustrated in Fig. \ref{fig:2}, when optimizing on the sphere, only the tangential component of the gradient is useful for updating the direction, while the radial component is ineffective and will be discarded. In addition, to avoid repeated evaluations of the $\sin(\cdot)$ and $\cos(\cdot)$ functions \cite{meng2019spherical}, we reparameterize the angle $ \theta $ using a parameter $t_0$ by defining $ \tan(\theta / 2) = t_0 $. Thus, the update formula \eqref{unit_update} can be rewritten as
    \begin{align}
    \bm{z}^{(\ell+1)}=\Bigg\{\begin{array}{ll}
    \frac{1 - t^{2}}{1 + t^{2}} \cdot \bm{z}^{(\ell)} + \frac{2t}{1 + t^{2}} \cdot \bm{g}_{\rm{T}}^{*},&\mbox{$||\bm{g}_{\rm{T}}||_2 > \tau$}\\
    \bm{z}^{(\ell)},&\mbox{$||\bm{g}_{\rm{T}}||_2 \leq \tau$}
    \end{array}.   \label{update:sphere}
    \end{align}

\begin{figure}[t]
\centering

\begin{subfigure}[b]{0.45\textwidth}
\centering
\begin{tikzpicture}[>=stealth, scale=2]

\draw[->, thick] (-1.3,0) -- (1.3,0) node[right] {$x$};
\draw[->, thick] (0,-1.3) -- (0,1.3) node[above] {$y$};

\draw[ultra thick, black!70] (0,0) circle (1);
\node[below left] at (-1,0) {$\mathbb{S}^1$};

\coordinate (z) at (10:1);
\draw[->, ultra thick, purple] (0,0) -- (z) node[midway, above, yshift=2pt] {$\bm{z}^{(\ell)}$};


\coordinate (grend) at ($(z)+(0.6,0)$);
\coordinate (gtend) at ($(z)+(0,0.127)$);
\coordinate (gend) at ($(z)+(0.6,0.127)$);

\draw[->, ultra thick, blue] (z) -- (gend) node[above right] {$\bm{g}$};
\draw[->, ultra thick, blue] (z) -- (grend) node[below] {$\bm{g}_{\rm{R}}$};
\draw[->, ultra thick, blue] (z) -- (gtend) node[above, xshift=2pt] {$\bm{g}_{\rm{T}}$};

\draw[thick, dashed, blue] (grend) -- (gend) -- (gtend);
\draw[thick, dashed, blue] (z) -- (gend);

\end{tikzpicture}
\caption*{(a)}
\end{subfigure}
\hfill
\begin{subfigure}[b]{0.45\textwidth}
\centering
\begin{tikzpicture}[>=stealth, scale=2]

\draw[->, thick] (-1.3,0) -- (1.3,0) node[right] {$x$};
\draw[->, thick] (0,-1.3) -- (0,1.3) node[above] {$y$};

\draw[ultra thick, black!70] (0,0) circle (1);
\node[below left] at (-1,0) {$\mathbb{S}^1$};

\coordinate (z) at (10:1);
\draw[->, ultra thick, purple] (0,0) -- (z) node[midway, above, yshift=2pt] {$\bm{z}^{(\ell)}$};

\draw[->, ultra thick, orange!80!yellow] (0,0) -- (45:1) node[midway, above, xshift=-2pt, yshift=2pt] {$\bm{z}^{(\ell+1)}$};

\coordinate (Aend) at ($(z)+(0,0.85)$);
\draw[->, ultra thick, green!70!black] (z) -- (Aend) node[above] {$ \frac{2t}{1 + t^{2}} \cdot \bm{g}_{\rm{T}}^{*} $};

\coordinate (grend) at ($(z)+(0.155,0)$);
\coordinate (gtend) at ($(z)+(0,0.580)$);
\coordinate (gend) at ($(z)+(0.155,0.580)$);

\draw[->, ultra thick, blue] (z) -- (gend) node[above right] {$\bm{g}$};
\draw[->, ultra thick, blue] (z) -- (grend) node[right, text=blue] {$\bm{g}_{\rm{R}}$};
\draw[->, ultra thick, blue] (z) -- (gtend) node[left, xshift=1pt] {$\bm{g}_{\rm{T}}$};

\draw[thick, dashed, blue] (grend) -- (gend) -- (gtend);
\draw[thick, dashed, blue] (z) -- (gend);

\end{tikzpicture}
\caption*{(b)}
\end{subfigure}

\caption{Example of gradient-based update on the unit sphere $ \mathbb{S}^1 $: (a) when $ ||\bm{g}_{\rm{T}}||_2 \leq \tau $, the update magnitude in the direction of $ \bm{g}_{\rm{T}} $ becomes extremely small, thus we terminate the update of $\bm{z}^{(\ell)}$; (b) when $ ||\bm{g}_{\rm{T}}||_2 > \tau $, the update of feature $\bm{z}^{(\ell)}$ is accelerated by $ t $, where the adaptive step-size parameter lies in $[t_0, t_0(1 + \beta)]$.}
\label{fig:2_2}
\end{figure}
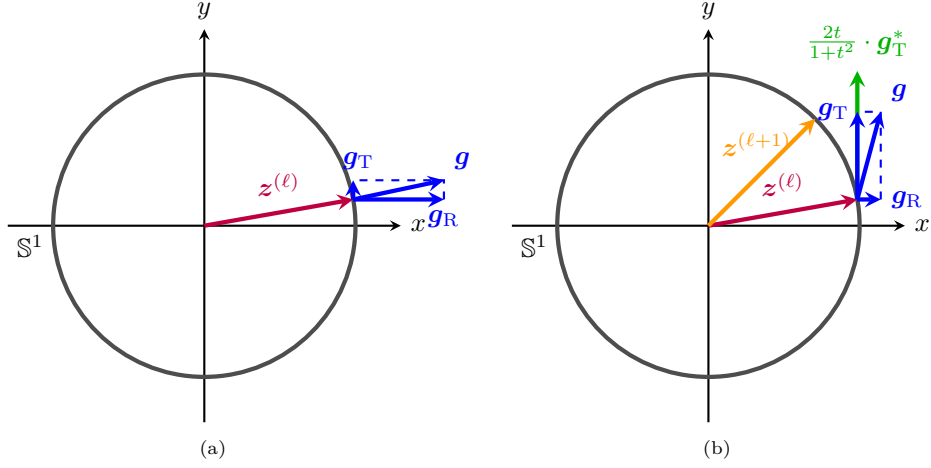

where $ \bm{g}_{\rm{T}}^{*}= \bm{g}_{\rm{T}} / ||\bm{g}_{\rm{T}}||_2 $, and $ t $ represents the adaptive step-size parameter. Since only the tangential component contributes to a feasible update, the degree to which the gradient aligns with the tangent space reflects its effectiveness under the spherical constraint. Inspired by the cosine-similarity-based strategy in \cite{meng2019spherical}, the quantity $1 - |\bm g^{\rm T} \bm z^{(\ell)}| / \|\bm g\|_2$ is adopted to characterize the tangential effectiveness of the Euclidean gradient. It approaches zero when the gradient is nearly radial and increases as the gradient becomes more aligned with the tangent space. The absolute value makes the measure invariant to the sign of the radial component, since inward and outward radial components are both removed by tangent-space projection. The adaptive step-size parameter $ t $ can be expressed as
    \begin{align}
    t = t_0 \cdot \left(1 + \beta \cdot \left( 1 - \left|\frac{\bm{g}^{\rm T}\bm{z}^{(\ell)}}{||\bm{g}||_2}\right|\right)\right), \quad \beta \geq 0,   \label{Ad_t}
    \end{align}
where $ \beta $ is a hyperparameter that controls the maximum expansion factor of the step size. Notice that the norm of the Riemannian update direction is not used to directly determine the update magnitude. As illustrated in Fig. \ref{fig:2_2}, the actual update is jointly controlled by the threshold $ \tau $ and the adaptive parameter $t$. Specifically, when $ ||\bm{g}_{\rm{T}}||_2 \leq \tau $, the effective update gradient of $ \bm{z}^{(\ell)} $ is regarded as sufficiently small, and the corresponding sample is no longer updated. When $ ||\bm{g}_{\rm{T}}||_2 > \tau $, the sample remains active and is updated along the normalized Riemannian update direction, with the update magnitude controlled by the adaptive parameter $t$. This design corresponds to the subsequent theoretical analysis of finite-termination, in which the inequality characterizing the variation of the MCR$^2$ objective guarantees a uniform positive lower bound on the objective increment for each active sample under the specified step-size condition. The algorithm terminates when no sample remains active at a given iteration. To theoretically justify this stopping criterion, the Riemannian update mapping associated with the MCR$^2$ objective is analyzed next, and the finite-termination property of the proposed algorithm is subsequently established.

\subsection{Riemannian Update Analysis and Finite-Termination of MCR$^2$}
In this subsection, we establish a finite-termination guarantee for the proposed Riemannian update scheme. Specifically, we first derive the eigenvalue bounds for the matrices involved in the MCR$^2$ objective, and then establish the Lipschitz continuity of the implicit parameter mappings and the corresponding inverse-matrix mappings. Based on these auxiliary results, we establish the Lipschitz continuity of the Riemannian update mapping and subsequently derive an inequality characterizing the variation of the MCR$^2$ objective. This inequality provides the theoretical foundation for analyzing the proposed geodesic update scheme and proving that the proposed scheme satisfies the prescribed stopping criterion within a finite number of iterations under suitable threshold and step-size conditions. According to \eqref{gt_formula}, the Riemannian update mapping can be written as
\begin{align}
\bm G_{\rm T}(\bm Z)
=
\bigl[
\bm g_{\rm{T},1},
\bm g_{\rm{T},2},
\ldots,
\bm g_{\rm{T},m}
\bigr]
\in\mathbb R^{n\times m}.
\label{Gt_column_form}
\end{align}
This yields the proposition below.

\begin{proposition} \label{pro:1}
The Riemannian update mapping $\bm{G}_{\rm T}$ is Lipschitz continuous on $\mathcal M=(\mathbb{S}^{n-1})^m$ with Lipschitz constant
\begin{align}
L_{\rm grad}
=
(m+1)(k+1)L_{\rm E}
+
\frac{2n(n+1)}{\epsilon^2}
\left(
1+\frac{1}{\sqrt m}
\right),
\end{align}
where
\begin{align}
L_{\rm E}
=
\frac{n(n+1)}{m\epsilon^2}
\biggl[
1
+
(n+1)
\left(
\pi\sqrt{n}
+
\frac{2n}{\epsilon^2}
\right)
\biggr].
\end{align}
Moreover, $\bm G_{\rm T}$ satisfies
\begin{align}
\|\bm G_{\rm T}(\bm Z)\|_F
\leq
\frac{2n(n+1)}
{\sqrt m\,\epsilon^2},
\qquad
\forall\,\bm Z\in\mathcal M.
\end{align}
Define
\begin{align}
F(\bm Z)=\Delta R\left(\bm Z,\epsilon,\bm\Pi;\alpha(\bm Z),\{\alpha_j(\bm Z^j)\}_{j=1}^{k}\right),
\label{composite_objective}
\end{align}
where the implicit parameters $\alpha(\bm Z)$ and $\alpha_j(\bm Z^j)$ are recomputed according to the current feature matrix $\bm Z$. Then, there exists a constant $L_{\rm s}>0$ such that, for any $\bm Z\in\mathcal M$ and any tangent vector $\bm\xi\in T_{\bm Z}\mathcal M$,
\begin{align}
F\bigl(\operatorname{Geo}(\bm Z,\bm\xi)\bigr)-F(\bm Z)
\geq
\left\langle
\bm G_{\rm T}(\bm Z),\bm\xi
\right\rangle_F
-C_{\alpha}\|\bm\xi\|_F
-\frac{L_{\rm s}}{2}\|\bm\xi\|_F^2.
\label{Riemannian_smoothness}
\end{align}
Here, $\bm Z=[\bm z_1,\ldots,\bm z_m]$ and $\bm\xi=[\bm\xi_1,\ldots,\bm\xi_m]$. The constant $C_{\alpha}$ is given by
\begin{align}
C_{\alpha}=\frac{2n\sqrt{n+1}}{\sqrt m}.
\label{C_alpha_definition}
\end{align}
The mapping $\operatorname{Geo}(\bm Z,\bm\xi)$ can be expressed as
\begin{align}
\operatorname{Geo}(\bm Z,\bm\xi)
&=
\Biggl[
\cos\bigl(\|\bm\xi_1\|_2\bigr)\bm z_1
+
\sin\bigl(\|\bm\xi_1\|_2\bigr)
\frac{\bm\xi_1}{\|\bm\xi_1\|_2},
\ \ldots,
\notag\\
&\qquad
\cos\bigl(\|\bm\xi_m\|_2\bigr)\bm z_m
+
\sin\bigl(\|\bm\xi_m\|_2\bigr)
\frac{\bm\xi_m}{\|\bm\xi_m\|_2}
\Biggr].
\label{Geodesic_mapping}
\end{align}
If $\bm\xi_i = \boldsymbol{0}$, the corresponding column is defined as $\bm{z}_i$. In particular, one may take
\begin{align}
L_{\rm s}=L_{\rm grad}+\frac{2n(n+1)}{\sqrt m\,\epsilon^2}. \label{Lsm_definition}
\end{align}
\end{proposition}
Detailed proofs of the supporting results are provided from \ref{app:proof-lemma-first} to \ref{app:proof-lemma-fourth}, while the
proof of Proposition \ref{pro:1} is given in \ref{app:proof-proposition}. The inequality established above characterizes the variation of the MCR$^2$ objective under the proposed Riemannian update scheme, where the update direction at each iteration is formed with the current implicit parameters treated as fixed. Based on this result, a finite-termination guarantee for the proposed truncated Riemannian update is established in the following theorem.

\begin{theorem}\label{thm:1}
Let $\{\bm Z^{(\ell)}\}_{\ell\geq0}$ be the sequence generated by
\eqref{update:sphere} from $\bm Z^{(0)}\in\mathcal M$, where
$\bm Z^{(0)}$ denotes the initial sample matrix.
Suppose that $t_0>0$, $\beta\geq0$, and
$\tau>C_{\alpha}$, and that
\begin{align}
\arctan\!\bigl(t_0(1+\beta)\bigr)
<
\frac{\tau-C_{\alpha}}{L_{\rm s}},
\end{align}
where $C_{\alpha}$ and $L_{\rm s}$ are defined in
\eqref{C_alpha_definition} and \eqref{Lsm_definition}, respectively.
Define
\begin{align}
\delta_\tau
=
2\arctan(t_0)
\left[
\tau-C_{\alpha}
-
L_{\rm s}\arctan\!\bigl(t_0(1+\beta)\bigr)
\right].
\label{delta_tau_definition}
\end{align}
Then $\delta_\tau>0$, and the proposed truncated Riemannian update
scheme terminates no later than iteration
\begin{align}
T_{\max}
=
\left\lceil
\frac{
F_{\max}-F(\bm Z^{(0)})
}{
\delta_\tau
}
\right\rceil,
\end{align}
where $F_{\max}$ denotes a finite upper bound on $F(\bm Z)$ for $\bm Z\in\mathcal M$.
\end{theorem}
\begin{proof}
Suppose that the stopping criterion is not satisfied at the beginning
of iteration $\ell$. Define the set of effective column indices as
\begin{align}
\mathcal I_{\ell}=\left\{i:\|\bm g_{\mathrm{T},i}^{(\ell)}\|_2>\tau\right\},
\end{align}
where $\bm g_{{\rm T},i}^{(\ell)}$ denotes the $i$-th column of the Riemannian update direction at iteration $\ell$. Then $\mathcal I_{\ell} \neq \varnothing$. For each $i\in\mathcal I_{\ell}$, the update rule for the $i$-th column is given by \eqref{update:sphere}. Let
\begin{align}
\theta_i^{(\ell)}=2\arctan\!\left(t_i\right),
\end{align}
where $t_i$ denotes the adaptive step-size parameter of the $i$-th column, and
$\theta_i^{(\ell)}$ represents the geodesic distance traveled by that
column on the unit sphere. Since $\|\bm z_i^{(\ell)}\|_2=1$, we have
\begin{align}
0\leq\left|\frac{\bm g_i^{\rm T}\bm z_i^{(\ell)}}{\|\bm g_i\|_2}\right|\leq 1.
\end{align}
Thus, it follows from \eqref{Ad_t} that
\begin{align}
t_0\leq t_i \leq t_0(1+\beta). \label{adaptive_step_bound}
\end{align}
Since $\arctan(\cdot)$ is strictly increasing, we further have
\begin{align}
2\arctan(t_0) \leq \theta_i^{(\ell)} \leq 2\arctan\!\bigl(t_0(1+\beta)\bigr).
\label{theta_bound}
\end{align}

At the $\ell$-th iteration, define the tangent update vector $\bm\xi^{(\ell)}\in T_{\bm Z^{(\ell)}}\mathcal M$ by specifying its
$i$-th column as
\begin{align}
\bm\xi_i^{(\ell)}
=\begin{cases}
\displaystyle
\theta_i^{(\ell)}
\frac{\bm{g}_{\mathrm{T},i}^{(\ell)}}
{\|\bm{g}_{\mathrm{T},i}^{(\ell)}\|_2},
& i\in\mathcal I_{\ell}, \\[3mm]
\bm 0,
&i\notin\mathcal I_{\ell}.
\end{cases}
\label{xi_definition}
\end{align}
By Proposition \ref{pro:1} and the definition of $\bm G_{\rm T}(\bm Z)$ in \eqref{Gt_column_form}, the inequality in \eqref{Riemannian_smoothness} gives
\begin{align}
&F(\bm Z^{(\ell+1)})-F(\bm Z^{(\ell)})
\notag\\
&\geq
\left\langle
\bm G_{\rm T}(\bm Z^{(\ell)}),
\bm\xi^{(\ell)}
\right\rangle_F
-
C_{\alpha}\|\bm\xi^{(\ell)}\|_F
-
\frac{L_{\rm s}}{2}
\|\bm\xi^{(\ell)}\|_F^2
\label{Riemannian_smoothness_inequality}\\
&=
\sum_{i\in\mathcal I_{\ell}}
\theta_i^{(\ell)}
\|\bm g_{{\rm T},i}^{(\ell)}\|_2
-
C_{\alpha}
\left[
\sum_{i\in\mathcal I_{\ell}}
\bigl(\theta_i^{(\ell)}\bigr)^2
\right]^{1/2}
-
\frac{L_{\rm s}}{2}
\sum_{i\in\mathcal I_{\ell}}
\bigl(\theta_i^{(\ell)}\bigr)^2
\label{ell_gt_1}  \\
&\geq
\sum_{i\in\mathcal I_{\ell}}
\left[
\theta_i^{(\ell)}
\|\bm g_{{\rm T},i}^{(\ell)}\|_2
-
C_{\alpha}\theta_i^{(\ell)}
-
\frac{L_{\rm s}}{2}
\bigl(\theta_i^{(\ell)}\bigr)^2
\right]
\label{ell_gt_2} \\
&>
\sum_{i\in\mathcal I_{\ell}}
\left[
(\tau-C_{\alpha})\theta_i^{(\ell)}
-
\frac{L_{\rm s}}{2}
\bigl(\theta_i^{(\ell)}\bigr)^2
\right].
\label{ell_gt}
\end{align}
The equality in \eqref{ell_gt_1} follows from \eqref{Gt_column_form} and \eqref{xi_definition}, while the subsequent
inequality follows from $\|\bm\xi^{(\ell)}\|_F \leq\sum_{i\in\mathcal I_{\ell}}\theta_i^{(\ell)}$, since $\theta_i^{(\ell)}\geq0$.
Finally, the strict inequality in \eqref{ell_gt} follows from $\|\bm g_{{\rm T},i}^{(\ell)}\|_2>\tau$ for every $i\in\mathcal I_{\ell}$. For each $i\in\mathcal I_{\ell}$, using \eqref{theta_bound} together with the step-size condition, we obtain
\begin{align}
(\tau-C_{\alpha})\theta_i^{(\ell)}
-\frac{L_{\rm s}}{2}\bigl(\theta_i^{(\ell)}\bigr)^2
&=\theta_i^{(\ell)}
\left(
\tau-C_{\alpha}
-\frac{L_{\rm s}}{2}\theta_i^{(\ell)}
\right)
\notag\\
&\geq
2\arctan(t_0)
\left[
\tau-C_{\alpha}
-L_{\rm s}\arctan\!\bigl(t_0(1+\beta)\bigr)
\right].
\label{delta_tau_lower_bound}
\end{align}
By \eqref{delta_tau_definition}, the right-hand side of \eqref{delta_tau_lower_bound} is equal to $\delta_\tau$. Moreover, the step-size condition ensures that $\delta_\tau>0$. Since $\mathcal I_{\ell}\neq\varnothing$, combining \eqref{ell_gt} and \eqref{delta_tau_lower_bound} yields
\begin{align}
F(\bm Z^{(\ell+1)})-F(\bm Z^{(\ell)})>|\mathcal I_{\ell}|\delta_\tau\geq\delta_\tau. \label{uniform_increment}
\end{align}
Finally, suppose that the algorithm proceeds through $T$ consecutive iterations without satisfying the stopping criterion. Summing \eqref{uniform_increment} over $\ell=0,\ldots,T-1$ gives
\begin{align}
F(\bm Z^{(T)})-F(\bm Z^{(0)})>T\delta_\tau.  \label{cumulative_increment}
\end{align}
Since the implicit parameter mappings are continuous by Lemma \ref{lem:2}, the objective $F$ is continuous on the compact manifold $\mathcal M$. Therefore, $F$ attains a finite maximum $F_{\max}$ on $\mathcal M$. Combining this with \eqref{cumulative_increment}, we obtain
\begin{align}
T\delta_\tau<F_{\max}-F(\bm Z^{(0)}),
\end{align}
which implies
\begin{align}
T<\frac{F_{\max}-F(\bm Z^{(0)})}{\delta_\tau}.
\end{align}
Consequently, the algorithm terminates no later than iteration
\begin{align}
T_{\max}=\left\lceil\frac{F_{\max}-F(\bm Z^{(0)})}{\delta_\tau}\right\rceil,
\end{align}
where $\lceil\cdot\rceil$ denotes the ceiling function. This completes the proof.
\end{proof}
Theorem \ref{thm:1} establishes a finite-termination guarantee for the proposed Riemannian update scheme under the prescribed threshold and step-size conditions. It should be noted that these conditions are sufficient but not necessary. Moreover, since the constants $C_{\alpha}$ and $L_{\rm s}$ are derived from global worst-case bounds over the entire feasible manifold $\mathcal M$, the resulting sufficient conditions can be relatively conservative in practice. Nevertheless, when these conditions are satisfied, they provide a finite-termination guarantee even in the worst case. Therefore, these conditions are intended primarily as theoretical guarantees rather than practical parameter-selection rules.

\subsection{LA-ReduNet Architecture}

\begin{figure}[t]
\resizebox{\linewidth}{!}{
\begin{tikzpicture}[
    line width=1.0pt,
    box/.style={
        draw,
        align=center,
        minimum height=28pt,
        inner xsep=10pt,
        inner ysep=6pt,
        rounded corners
    },
    greenbox/.style={
        box,
        fill=lime!30
    },
    bluebox/.style={
        box,
        fill=cyan!25
    },
    redbox/.style={
        box,
        fill=red!25
    },
    yellowcirc/.style={
        circle,
        draw,
        fill=yellow!35,
        inner sep=4pt,
        line width=1.0pt
    },
    orangecirc/.style={
        circle,
        draw,
        fill=orange!30,
        inner sep=4pt,
        line width=1.0pt
    },
    purplecirc/.style={
        circle,
        draw,
        fill=violet!25,
        inner sep=4pt,
        line width=1.0pt
    }
]


\node[greenbox] (zl)
{$\bm{z}^{(0)} \in \mathbb{R}^{d}$};

\node[
    redbox,
    left=1.3cm of zl
] (P)
{$\mathcal{P}_{\mathbb{S}^{n-1}}$};

\draw[->] (P.east) -- (zl.west);

\node[
    redbox,
    left=1.3cm of P,
    minimum width=3.2cm,
    font=\bfseries
] (dim)
{$\mathcal{D}:\mathbb{R}^n \rightarrow \mathbb{R}^d$};

\draw[->] (dim.east) -- (P.west);

\node[
    greenbox,
    left=1.3cm of dim
] (x)
{$\bm{x} \in \mathbb{R}^{n}$};

\draw[->] (x.east) -- (dim.west);

\node[
    draw,
    dashed,
    thick,
    inner sep=0.5cm,
    fit=(dim) (P)
] (dm) {};

\node[above=2pt of dm]
{\textbf{Dimension-reduction Module}};


\coordinate (split) at ($(zl.east)+(0.8,0)$);

\draw (zl.east) -- (split);

\coordinate (operatorleft) at ($(split)+(0.8,0)$);


\node[
    bluebox,
    anchor=west
] (E) at ($(operatorleft)+(0,3.6)$)
{$\bm{E}^{(\ell)}$};


\node[
    bluebox,
    anchor=west
] (CPi1) at ($(operatorleft)+(0,2.1)$)
{$\bm{C}_{1}^{(\ell)},\,
  \bm{\Pi}_{1}^{(\ell-1)}$};

\node[
    bluebox,
    anchor=west
] (CPi2) at ($(operatorleft)+(0,0.7)$)
{$\bm{C}_{2}^{(\ell)},\,
  \bm{\Pi}_{2}^{(\ell-1)}$};

\node[
    anchor=center
] (CPidots) at ($(CPi2.center)+(0,-0.9)$)
{$\vdots$};

\node[
    bluebox,
    anchor=west
] (CPik) at ($(operatorleft)+(0,-2.1)$)
{$\bm{C}_{k}^{(\ell)},\,
  \bm{\Pi}_{k}^{(\ell-1)}$};


\node[
    bluebox,
    anchor=west
] (I) at ($(operatorleft)+(0,-3.6)$)
{$\bm{I}$};


\draw[->] (split) |- (E.west);

\draw[->] (split) |- (CPi1.west);
\draw[->] (split) |- (CPi2.west);
\draw[->] (split) |- (CPik.west);

\draw[->] (split) |- (I.west);


\coordinate (mergeline) at ($(CPi1.east)+(1.0,0)$);

\coordinate (merge1) at (mergeline |- CPi1.east);
\coordinate (merge2) at (mergeline |- CPi2.east);
\coordinate (mergek) at (mergeline |- CPik.east);

\draw (CPi1.east) -- (merge1);
\draw (CPi2.east) -- (merge2);
\draw (CPik.east) -- (mergek);

\draw (merge1) -- (mergek);


\node[
    yellowcirc,
    right=1.2cm of merge2
] (minus)
{$\bm{-}$};

\node[
    yellowcirc,
    right=1.2cm of minus
] (plus)
{$\bm{+}$};

\draw[->]
    (merge2)
    --
    node[midway,above] {$\eta$}
    (minus.west);

\draw[->]
    (minus.east)
    --
    (plus.west);


\draw[->]
    (E.east)
    --
    node[midway,above] {$\eta$}
    (plus |- E.east)
    --
    (plus.north);


\node[
    purplecirc,
    right=1.2cm of plus
] (grad)
{$\bm{g}_{\rm T}^*$};

\draw[->]
    (plus.east)
    --
    node[midway,above]
    {$\bm{g}_{\rm T} \mapsto \bm{g}_{\rm T}^*$}
    (grad.west);

\node[
    orangecirc,
    right=1.2cm of grad
] (gate)
{$\tau$};

\draw[->]
    (grad.east)
    --
    (gate.west);

\node[
    greenbox,
    right=2.0cm of gate
] (znext)
{$\bm{z}^{(L)} \in \mathbb{R}^{d}$};

\draw[->]
    (gate.east)
    --
    (znext.west);


\draw[->]
    (I.east)
    --
    (gate |- I.east)
    --
    (gate.south);


\node[
    draw,
    dashed,
    thick,
    inner sep=1.2cm,
    fit=(I) (E)
        (CPi1) (CPi2) (CPik)
        (minus) (plus)
        (grad) (gate)
] (am) {};

\node[above=2pt of am]
{\textbf{$L$-Layer Structure of LA-ReduNet}};


\node[above=4pt of E]
{Expansion Operator};

\node[above=7pt of CPik]
{Compression Operators};

\node[below=4pt of I]
{Identity Matrix};

\node[above=4pt of gate]
{Threshold-Based Spherical Update};

\end{tikzpicture}
}
\caption{Feature Extraction Process of LA-ReduNet.}
\label{fig:LA_ReduNet}
\end{figure}
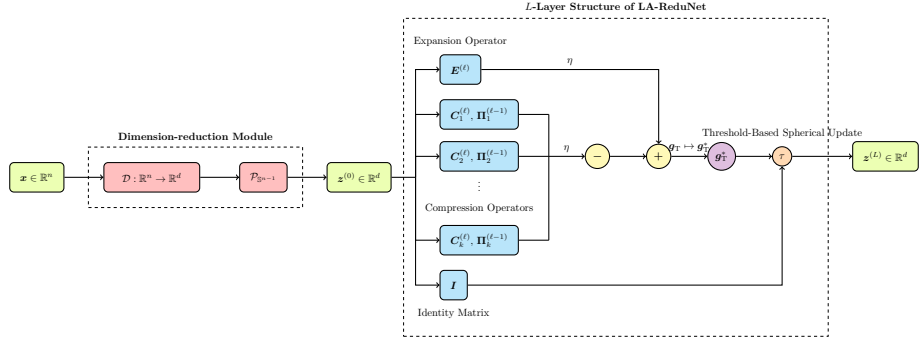

\begin{figure}[!t]
    \centering
    \includegraphics[width=\linewidth]{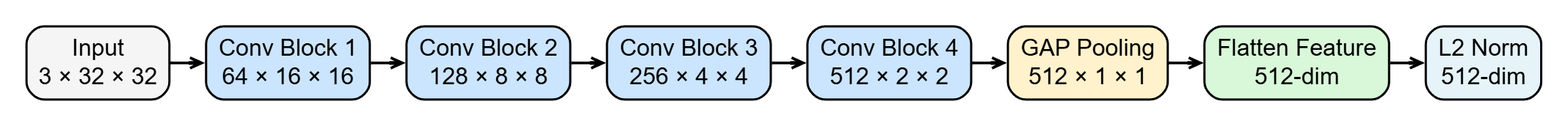}
    \caption{Structure of the dimension-reduction module.}
    \label{fig:6}
\end{figure}

In this subsection, we present the architecture of LA-ReduNet. As discussed above, the parameter storage of a ReduNet-type unfolded module depends on the network depth $L$, the number of classes $k$, and the feature dimension $n$, with an asymptotic complexity of $\textrm{O}(L(k+1)n^2)$. For the experiments, the same lightweight dimension-reduction module $ \mathcal{D} $ is used as a common front-end for ReduNet, AR-ReduNet, and LA-ReduNet to produce lower-dimensional normalized inputs and ensure a controlled comparison. As illustrated in Fig. \ref{fig:LA_ReduNet}, the feature dimension is reduced from $n$ to $d$, where $ d \ll n $. Accordingly, the parameter storage of the subsequent ReduNet-type module has an asymptotic complexity of $\textrm{O}(L(k+1)d^2)$. Since the same front-end and feature dimension are used for all three variants, the reduction in the parameter storage of the unfolded LA-ReduNet module relative to ReduNet and AR-ReduNet mainly results from its substantially smaller network depth. Algorithm \ref{alg:LwAda-ReduNet:1} and Algorithm \ref{alg:LwAda-ReduNet:2} respectively illustrate the training and testing procedures of LA-ReduNet.

\begin{algorithm}[!t]
\setstretch{1.15}
\caption{Training of LA-ReduNet ($\bm X, \bm \Pi, \epsilon, t_0, \tau, \beta$)}\label{alg:LwAda-ReduNet:1}
\begin{algorithmic}[1]
\STATE  $ \bm Z^{(0)}\leftarrow \Big[\frac{\bm x_1}{||\bm x_1||_2},\ldots,\frac{\bm x_m}{||\bm x_m||_2}\Big] $;

\FOR{$\ell=1$ to $L$}
\STATE  $ \bm E^{(\ell)}\leftarrow\frac{d}{m\epsilon^{2}}\big(\alpha \bm I+\frac{d}{m\epsilon^{2}} \bm Z^{(\ell-1)}(\bm Z^{(\ell-1)})^{\rm T}\big)^{-1}$;  \\ // $ \alpha $ is given by Algorithm \ref{alg:bisection}.
\FOR{$j=1$ to $k$}
\STATE $\bm C_{j}^{(\ell)} \leftarrow \frac{d}{m\epsilon^{2}}\big(\alpha_j \bm I + \frac{d}{\tr(\bm \Pi_j^{(\ell-1)})\epsilon^{2}} \bm Z^{(\ell-1)}\bm \Pi_j^{(\ell-1)}(\bm Z^{(\ell-1)})^{\rm T}\big)^{-1} $;  \\ // $ \alpha_{j} $ is given by Algorithm \ref{alg:bisection}.
\ENDFOR
\STATE $ \bm{G}^{(\ell)} \leftarrow \bm{E}^{(\ell)} \bm{Z}^{(\ell-1)}-\Big(\sum_{j=1}^k\bm{C}_{j}^{(\ell)}\bm{Z}^{(\ell-1)}\bm{\Pi}_j^{(\ell-1)}\Big);
$ \\ // $\bm{G}^{(\ell)}=[\bm{g}^{(\ell)}_{1}, \bm{g}^{(\ell)}_{2}, \ldots, \bm{g}^{(\ell)}_{m}] \in \mathbb{R}^{d \times m}$ is the gradient matrix corresponding to $ \bm Z^{(\ell-1)} $.
\FOR{$p=1$ to $m$}
\STATE $ (\bm{g}^{(\ell)}_{\rm T})_{p} \leftarrow \bm{g}^{(\ell)}_{p}-((\bm{g}^{(\ell)}_{p})^{\rm T}\bm{z}^{(\ell-1)}_{p})\bm{z}^{(\ell-1)}_{p} $;
\IF{$\|(\bm{g}^{(\ell)}_{\rm T})_{p}\| > \tau$}
    \STATE $ t = t_0 \cdot \left(1 + \beta \cdot \left( 1 - \left|\frac{(\bm{g}_{p}^{(\ell)})^{\rm T}\bm{z}_{p}^{(\ell-1)}}{||\bm{g}_{p}^{(\ell)}||}\right|\right)\right) $;
    \STATE $\bm{z}^{(\ell)}_{p} \leftarrow \displaystyle \frac{1 - t^{2}}{1 + t^{2}} \bm{z}^{(\ell-1)}_{p} + \frac{2t}{1 + t^{2}} \frac{(\bm{g}^{(\ell)}_{\rm T})_{p}}{||(\bm{g}^{(\ell)}_{\rm T})_{p}||}$;
\ELSE
    \STATE $\bm{z}^{(\ell)}_{p} \leftarrow \bm{z}^{(\ell-1)}_{p}$;
\ENDIF
\ENDFOR
\STATE $\bm Z^{(\ell)}=[\bm{z}^{(\ell)}_{1}, \bm{z}^{(\ell)}_{2}, \ldots, \bm{z}^{(\ell)}_{m}]$;
\ENDFOR
\RETURN $ \bm Z^{(L)}, \{ \bm{E}^{(\ell)}, \bm{C}_{1}^{(\ell)}, \bm{C}_{2}^{(\ell)}, \ldots, \bm{C}_{k}^{(\ell)}\}_{\ell=1}^{L}$.
\end{algorithmic}
\end{algorithm}

\begin{algorithm}[!t]
\setstretch{1.15}
\caption{Testing of LA-ReduNet ($ \widetilde{\bm{X}}, \{\bm{E}^{(\ell)}, \bm{C}_{1}^{(\ell)}, \ldots, \bm{C}_{k}^{(\ell)}\}_{\ell=1}^{L}, t_0, \tau, \beta, \lambda $)}\label{alg:LwAda-ReduNet:2}
\begin{algorithmic}[1]
\STATE  $ \widetilde{\bm Z}^{(0)}\leftarrow \Big[\frac{\widetilde{\bm x}_1}{||\widetilde{\bm x}_1||_2},\ldots,\frac{\widetilde{\bm x}_m}{||\widetilde{\bm x}_m||_2}\Big]$;\quad // $\widetilde{\bm x}_i$ is the $i$-th column of data matrix $\widetilde{\bm X}$.
\FOR{$\ell=1$ to $L$}
\FOR{$j=1$ to $k$}
\FOR{$i=1$ to $m$}
\STATE  $ \hat{\pi}_{i,j}^{(\ell-1)}\leftarrow \frac{\exp(-\lambda||\bm C^{(\ell)}_j\bm{\widetilde{z}}_i^{(\ell-1)}||)}{\sum_{j=1}^k\exp(-\lambda||\bm C^{(\ell)}_j\bm{\widetilde{z}}_i^{(\ell-1)}||)} $;
\ENDFOR
\STATE $\bm \Pi_j^{(\ell-1)}\leftarrow\diag(\hat{\pi}_{1,j}^{(\ell-1)},\ldots,\hat{\pi}_{m,j}^{(\ell-1)})$;
\ENDFOR
\STATE $ \widetilde{\bm{G}}^{(\ell)} \leftarrow \bm{E}^{(\ell)} \widetilde{\bm Z}^{(\ell-1)}-\Big(\sum_{j=1}^k\bm{C}_{j}^{(\ell)}\widetilde{\bm Z}^{(\ell-1)}\bm{\Pi}_j^{(\ell-1)}\Big);$
\\ // $\widetilde{\bm{G}}^{(\ell)}=[\widetilde{\bm{g}}^{(\ell)}_{1}, \widetilde{\bm{g}}^{(\ell)}_{2}, \ldots, \widetilde{\bm{g}}^{(\ell)}_{m}] \in \mathbb{R}^{d \times m}$ is the gradient matrix corresponding to $ \widetilde{\bm Z}^{(\ell-1)} $.
\FOR{$p=1$ to $m$}
\STATE $ (\widetilde{\bm{g}}^{(\ell)}_{\rm T})_{p} \leftarrow \widetilde{\bm{g}}^{(\ell)}_{p}-((\widetilde{\bm{g}}^{(\ell)}_{p})^{\rm T}\widetilde{\bm{z}}^{(\ell-1)}_{p})\widetilde{\bm{z}}^{(\ell-1)}_{p} $;
\IF{$\|(\widetilde{\bm{g}}^{(\ell)}_{\rm T})_{p}\| > \tau$}
    \STATE $ t = t_0 \cdot \left(1 + \beta \cdot \left( 1 - \left|\frac{(\widetilde{\bm{g}}_{p}^{(\ell)})^{\rm T}\widetilde{\bm{z}}_{p}^{(\ell-1)}}{||\widetilde{\bm{g}}_{p}^{(\ell)}||}\right|\right)\right) $;
    \STATE $\widetilde{\bm{z}}^{(\ell)}_{p} \leftarrow \displaystyle \frac{1 - t^{2}}{1 + t^{2}} \widetilde{\bm{z}}^{(\ell-1)}_{p} + \frac{2t}{1 + t^{2}} \frac{(\widetilde{\bm{g}}^{(\ell)}_{\rm T})_{p}}{||(\widetilde{\bm{g}}^{(\ell)}_{\rm T})_{p}||}$;
\ELSE
    \STATE $\widetilde{\bm{z}}^{(\ell)}_{p} \leftarrow \widetilde{\bm{z}}^{(\ell-1)}_{p}$;
\ENDIF
\ENDFOR
\STATE $\widetilde{\bm Z}^{(\ell)}=[\widetilde{\bm{z}}^{(\ell)}_{1}, \widetilde{\bm{z}}^{(\ell)}_{2}, \ldots, \widetilde{\bm{z}}^{(\ell)}_{m}]$;
\ENDFOR
\RETURN $ \widetilde{\bm Z}^{(L)} $.
\end{algorithmic}
\end{algorithm}

The dimension-reduction module is implemented using a simple convolutional architecture and is shared by all ReduNet variants in the experiments. Its role is to generate lower-dimensional normalized features for the subsequent ReduNet-type modules. This convolutional front-end is trained separately through backpropagation, whereas the subsequent LA-ReduNet module retains the white-box construction of ReduNet. As illustrated in Fig. \ref{fig:6}, the module consists of four cascaded convolutional blocks, followed by global average pooling (GAP) and a flattening operation, producing a 512-dimensional feature vector. Since the subsequent ReduNet-type modules require unit-norm input features, the output feature vectors are normalized before being fed into ReduNet, AR-ReduNet, or LA-ReduNet.

For training the common convolutional front-end, the same combined loss is used for all ReduNet variants. Since the downstream ReduNet-type models are constructed based on the MCR$^2$ objective, cross-entropy (CE) and MCR$^2$ are jointly employed to encourage the learned low-dimensional features to be compatible with the subsequent modules. The combined loss is defined as
\begin{align}
\mathcal{L}_{\text{combined}} = \mathcal{L}_{\text{CE}} -\mu \Delta R,
\end{align}
where $ \mathcal{L}_{\text{CE}} $ denotes the CE loss, $ \Delta R $ is derived from \eqref{MCR2_alpha}, and $ \mu $ represents weight coefficient. Since the output features are normalized to unit norm, scaled cosine similarities are used as the classification logits \cite{liu2017sphereface}. In this way, classification is determined by the angular similarity between the feature vectors and the weight vectors, making the classifier compatible with the normalized feature space. The corresponding cosine-based softmax CE loss is formulated as
\begin{align}
\mathcal{L}_{\text{CE}} = -\frac{1}{m}\sum_{i=1}^m \log \frac{e^{a\cos\theta_{y_i}}}{e^{a\cos\theta_{y_i}} + \sum_{j \neq y_i}^k e^{a\cos\theta_j}},
\end{align}
where $\theta_{y_i}$ denotes the angle between the feature of the $i$-th sample and the weight vector of its ground-truth class $y_i$, and $ a $ is a scaling factor.

\section{Simulation Results}
In this section, we evaluate the performance of LA-ReduNet across multiple datasets. Compared with the baseline algorithms ReduNet and AR-ReduNet, LA-ReduNet requires substantially fewer unfolded layers. Correspondingly, the parameter storage of the unfolded LA-ReduNet module is also significantly smaller than that of the corresponding ReduNet and AR-ReduNet modules. Following the idea of ReduNet \cite{chan2022redunet}, we select ReduNet, AR-ReduNet, and LA-ReduNet as feature extractors, and Nearest Subspace (NS) as the classifier. For a given feature $ \bm{z}_{\textrm{test}} $ of the test sample, the NS classifier provides the predicted label $ j_{\textrm{pred}} $, which satisfies:
\begin{align}
j_{\mathrm{pred}} = \arg\min_{t \in \{1, \dots, k\}} \left\| \left( \bm{I} - \bm{U}_t \bm{U}_t^{\mathrm{T}} \right) \bm{z}_{\mathrm{test}} \right\|_2^2,
\end{align}
where $\bm{U}_t \in \mathbb{R}^{n \times n}$ is the orthogonal matrix composed of left singular vectors of $\bm{Z}_j$, the
extracted features of training data in class $j$ by the feature extractor. We evaluate the performance of three different feature extractors on the CIFAR-10, CIFAR-100 \cite{krizhevsky2009learning}, and CINIC-10 datasets \cite{darlow2018cinic}. The detailed information of these datasets is presented in Table \ref{tab:0}.

\begin{table}[t!]
  \centering
  \begin{tabular}{lccc}  
    \toprule
     & \multicolumn{1}{c}{CIFAR-10} & \multicolumn{1}{c}{CIFAR-100} & \multicolumn{1}{c}{CINIC-10} \\  
    \midrule
    Training set       & 50000 & 50000 & 90000 \\
    Test set     & 10000  & 10000 & 90000 \\
    Number of classes  & 10 & 100 & 10 \\
    Data dimension    & 32 $\times$ 32 $\times$ 3 & 32 $\times$ 32 $\times$ 3 & 32 $\times$ 32 $\times$ 3 \\
    \bottomrule
  \end{tabular}
  \caption{Properties of the evaluated datasets.}
  \label{tab:0}
\end{table}

%

\subsection{Effect of Network Layers on Feature Visualization and Classification Accuracy}
In this section, we investigate the effect of the number of iterative layers $L$ on the feature extraction capability of different models. We consider two baseline models and our proposed method in the experiments. To distinguish different convergence behaviors observed in the experiments, we consider two notions of convergence. Objective convergence refers to the stage at which the MCR$^2$ objective becomes stable as the number of network layers increases, whereas classification-accuracy convergence refers to the stage at which the classification accuracy reaches and remains near a stable level. In general, classification-accuracy convergence may occur earlier than objective convergence.

\subsubsection{CIFAR-10 and CIFAR-100}
To ensure a fair comparison, ReduNet and AR-ReduNet use the same dimension reduction module as LA-ReduNet, ensuring that all three models receive identical input features during forward propagation. In addition, for consistency, parameter $a$ is set to 16 and $\mu$ is set to $10^{-3}$ across all experiments\footnote{The choice of $\mu=10^{-3}$ is validated by ablation studies (Table \ref{tab:2} and Fig. \ref{fig:13}), where it achieves a favorable balance between classification accuracy and feature representation quality compared to other tested values.}. In our training configuration, the AdamW optimizer is employed with an initial learning rate of $10^{-3}$ and weight decay of $10^{-2}$. The learning rate is dynamically adjusted using the cosine annealing scheduler with a period of 50 epochs and a minimum learning rate of $10^{-5}$. All models are trained for a total of 30 epochs. Meanwhile, the parameters of ReduNet and AR-ReduNet include: $ \epsilon =0.3 $, $ \eta=0.5 $ and maximum number of iterative layers $ L_{\rm{\max}}=1000 $. The step size $\eta=0.5$ is adopted following the original ReduNet setting, so as to retain the standard configuration of the baseline methods \cite{chan2022redunet}. For LA-ReduNet, the parameters include: $\epsilon=0.3$, the adaptive step-size parameter $t_0=0.05$, threshold $\tau=10^{-8}$, scaling coefficient $\beta=1$, and maximum number of iterative layers $L_{\rm{max}}=1000$. The threshold $\tau=10^{-8}$ is adopted to truncate samples whose Riemannian update direction norms have become negligible, while avoiding premature truncation of samples that still admit effective updates. This practical choice is not intended to satisfy the global sufficient conditions in Theorem \ref{thm:1}. Regarding the parameter $t_0$, a relatively small value $t_0=0.05$ is adopted, rather than choosing an aggressively large value for $t_0$ to accelerate the iteration.

For CIFAR-10 and CIFAR-100, as illustrated in Fig. \ref{fig:7}, we observe that LA-ReduNet attains its highest classification accuracy at approximately 5 layers and outperforms ReduNet and AR-ReduNet. By contrast, both ReduNet and AR-ReduNet fail to achieve stable accuracy within the 50-layer range. Furthermore, as shown in Fig. \ref{fig:7_1}, LA-ReduNet achieves objective convergence at approximately 35 layers, whereas ReduNet and AR-ReduNet require between 900 and 1000 layers to achieve objective convergence. After objective convergence is reached, the MCR$^2$ objective values of LA-ReduNet occasionally exhibit slight layer-to-layer decreases (first occurring at layer 35 on CIFAR-10 and layer 216 on CIFAR-100). These observations show that, although the experimental setting of LA-ReduNet does not satisfy the sufficient conditions derived from global worst-case bounds in Theorem \ref{thm:1}, its MCR$^2$ objective remains stable in practice after convergence. It should also be noted that AR-ReduNet and LA-ReduNet compute the coding rate using a modified multivariate Gaussian rate-distortion approximation function, and therefore their MCR$^2$ values are not directly comparable to that of ReduNet. According to the objective function \eqref{MCR2_alpha}, owing to the adoption of $\alpha $ and $\alpha_j $, the MCR$^2$ objective values solved by AR-ReduNet and LA-ReduNet are lower than those computed by ReduNet.

\begin{figure}[t]
    \centering
    \captionsetup[subfigure]{labelformat=parens}

    \begin{subfigure}{0.49\textwidth}
    \centering
    \includegraphics[width=\linewidth]{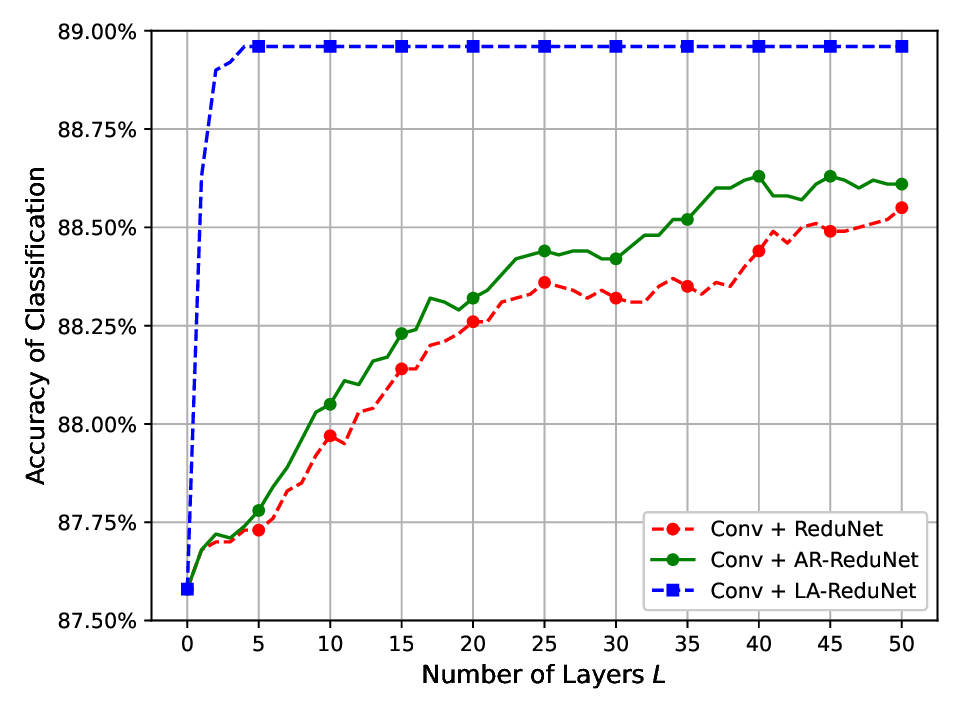}
    \caption{}  
    \label{fig:7_a}
    \end{subfigure}
    \hfill
    \begin{subfigure}{0.49\textwidth}
    \centering
    \includegraphics[width=\linewidth]{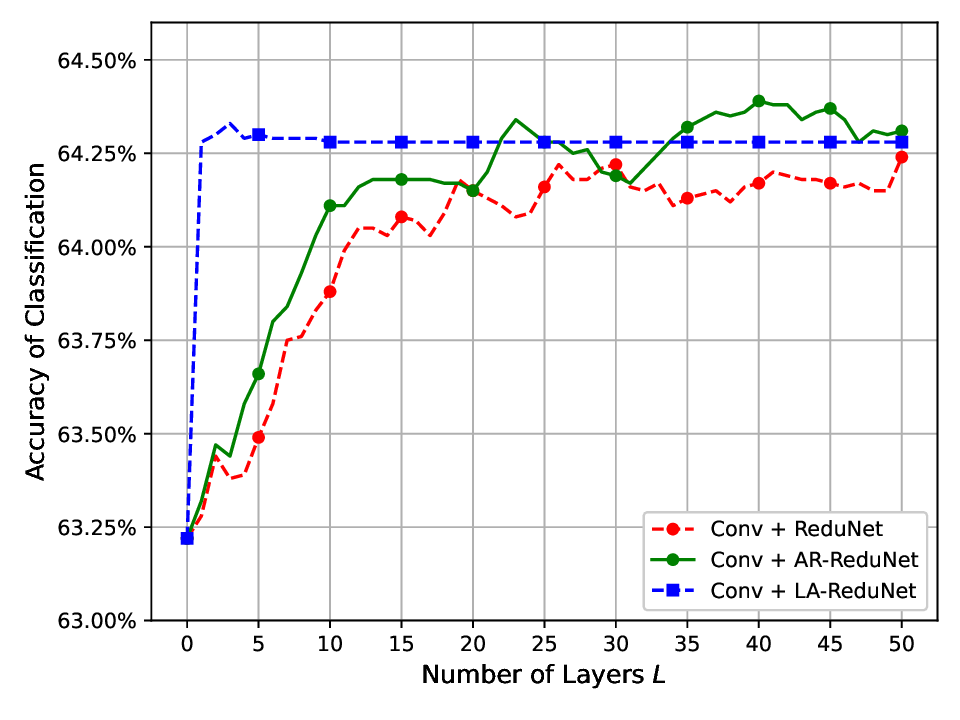}
    \caption{}  
    \label{fig:7_b}
    \end{subfigure}

    \vspace{0.5em}
    \caption{Trends of classification accuracy with the number of iterative layers $L$ for ReduNet, AR-ReduNet, and LA-ReduNet: (a) CIFAR-10; (b) CIFAR-100. Notice that the accuracy at $L=0$ corresponds to the result obtained by convolutional module.}
    \label{fig:7}
\end{figure}

\begin{figure}[t]
    \centering
    \captionsetup[subfigure]{labelformat=parens}

    \begin{subfigure}{0.49\textwidth}
    \centering
    \includegraphics[width=\linewidth]{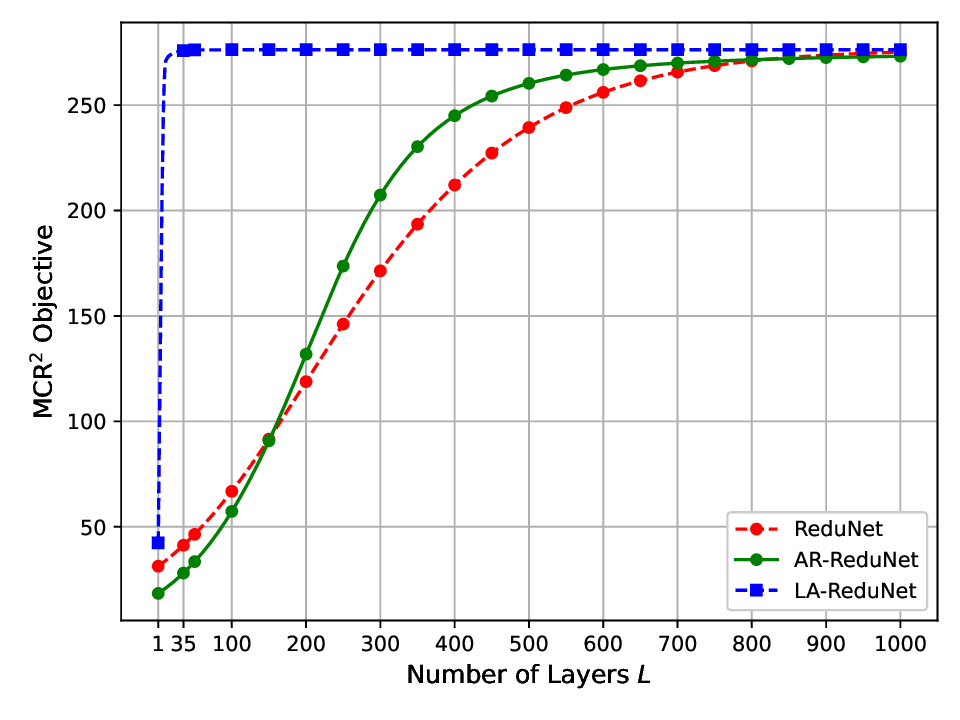}
    \caption{}  
    \label{fig:7_1_a}
    \end{subfigure}
    \hfill
    \begin{subfigure}{0.49\textwidth}
    \centering
    \includegraphics[width=\linewidth]{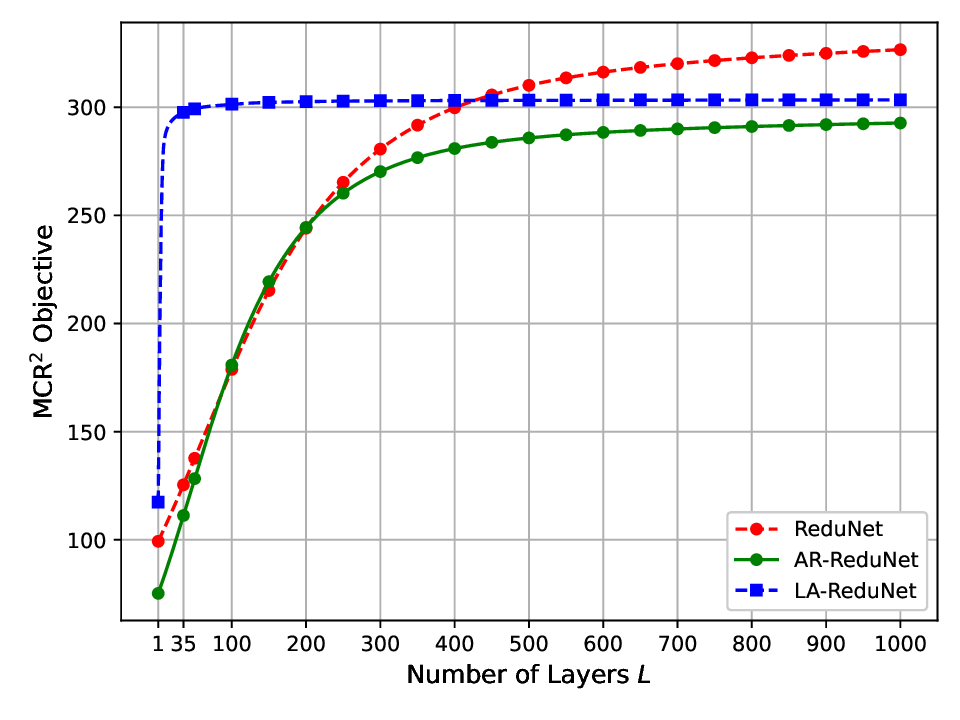}
    \caption{}  
    \label{fig:7_1_b}
    \end{subfigure}

    \vspace{0.5em}
    \caption{Curves of MCR$^2$ objective values during training: (a) CIFAR-10; (b) CIFAR-100. }
    \label{fig:7_1}
\end{figure}

Although the convolutional module improves classification accuracy on CIFAR-10 and CIFAR-100, Fig. \ref{fig:7_2_a} and Fig. \ref{fig:7_3_a} demonstrate that the features extracted by this module still fail to fully separate some classes. As illustrated in Fig. \ref{fig:7_2_f} and Fig. \ref{fig:7_3_f}, it can be observed that LA-ReduNet effectively learns more discriminative features. Compared with features obtained solely by convolutional modules, the features extracted by ReduNet ($L=1000$), AR-ReduNet ($L=1000$), and LA-ReduNet ($L=35$) exhibit much clearer boundaries. However, when the network has fewer layers, the features extracted by LA-ReduNet are more discriminative than those learned by ReduNet and AR-ReduNet. This demonstrates that LA-ReduNet can extract highly discriminative features with very few layers, greatly reducing the parameter storage relative to ReduNet and AR-ReduNet.

\begin{figure}[t!]
    \centering
    \captionsetup[subfigure]{labelformat=parens}

    \begin{subfigure}{0.49\textwidth}
    \centering
    \includegraphics[width=\linewidth]{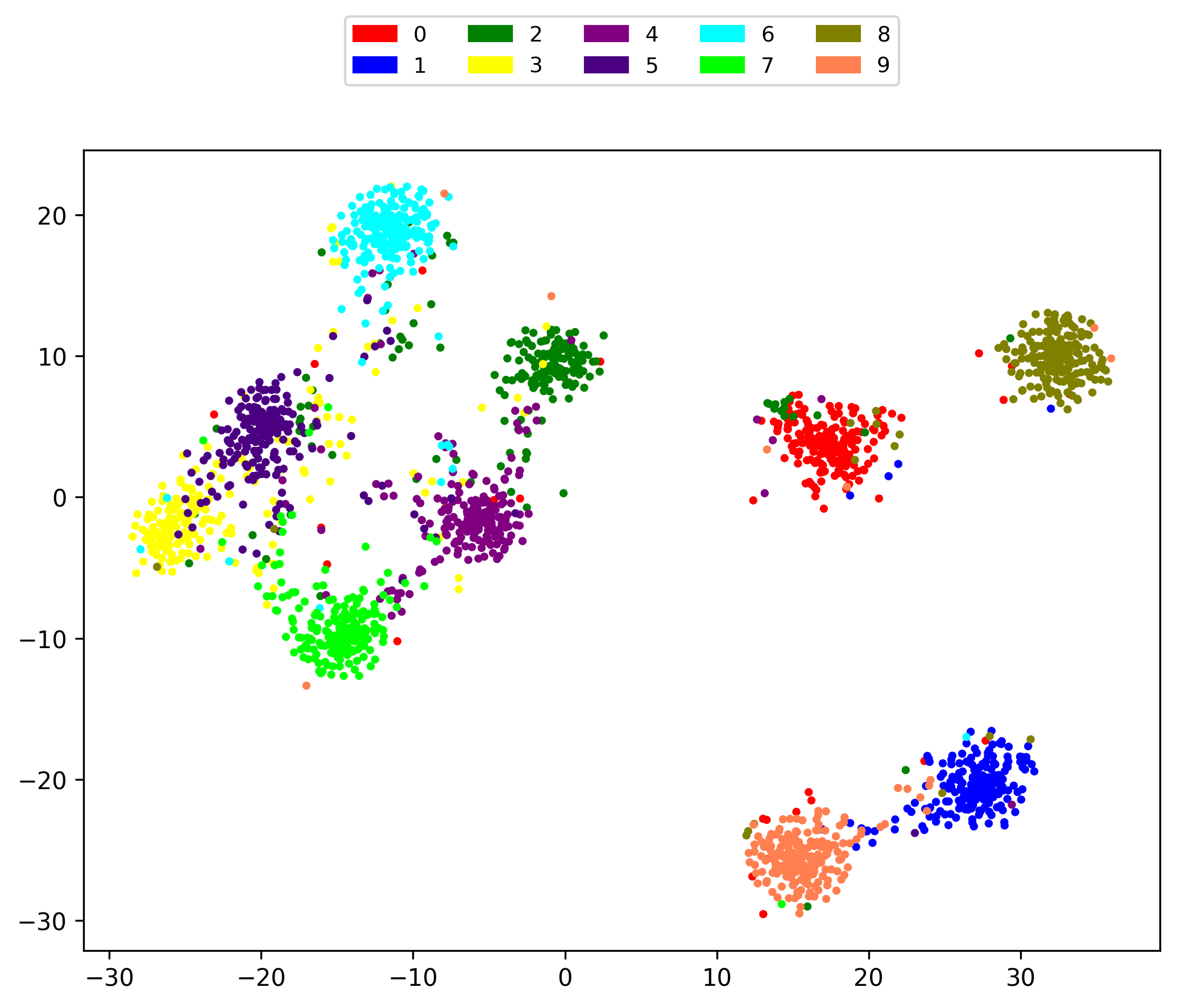}
    \caption{Conv only}
    \label{fig:7_2_a}
    \end{subfigure}
    \hfill
    \begin{subfigure}{0.49\textwidth}
    \centering
    \includegraphics[width=\linewidth]{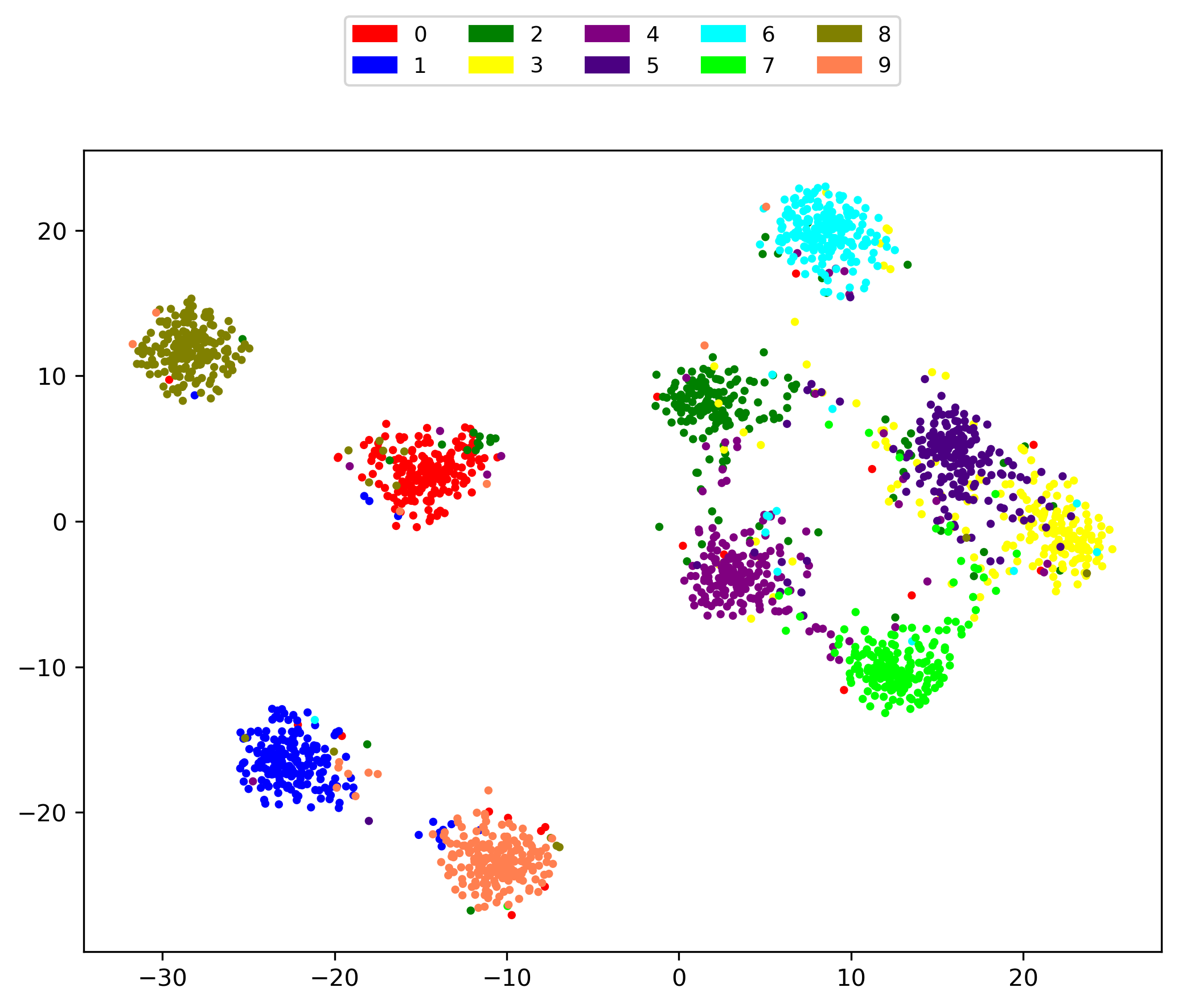}
    \caption{Conv + ReduNet (35 layers)}
    \label{fig:7_2_b}
    \end{subfigure}

    \vspace{0.5em}

    \begin{subfigure}{0.49\textwidth}
    \centering
    \includegraphics[width=\linewidth]{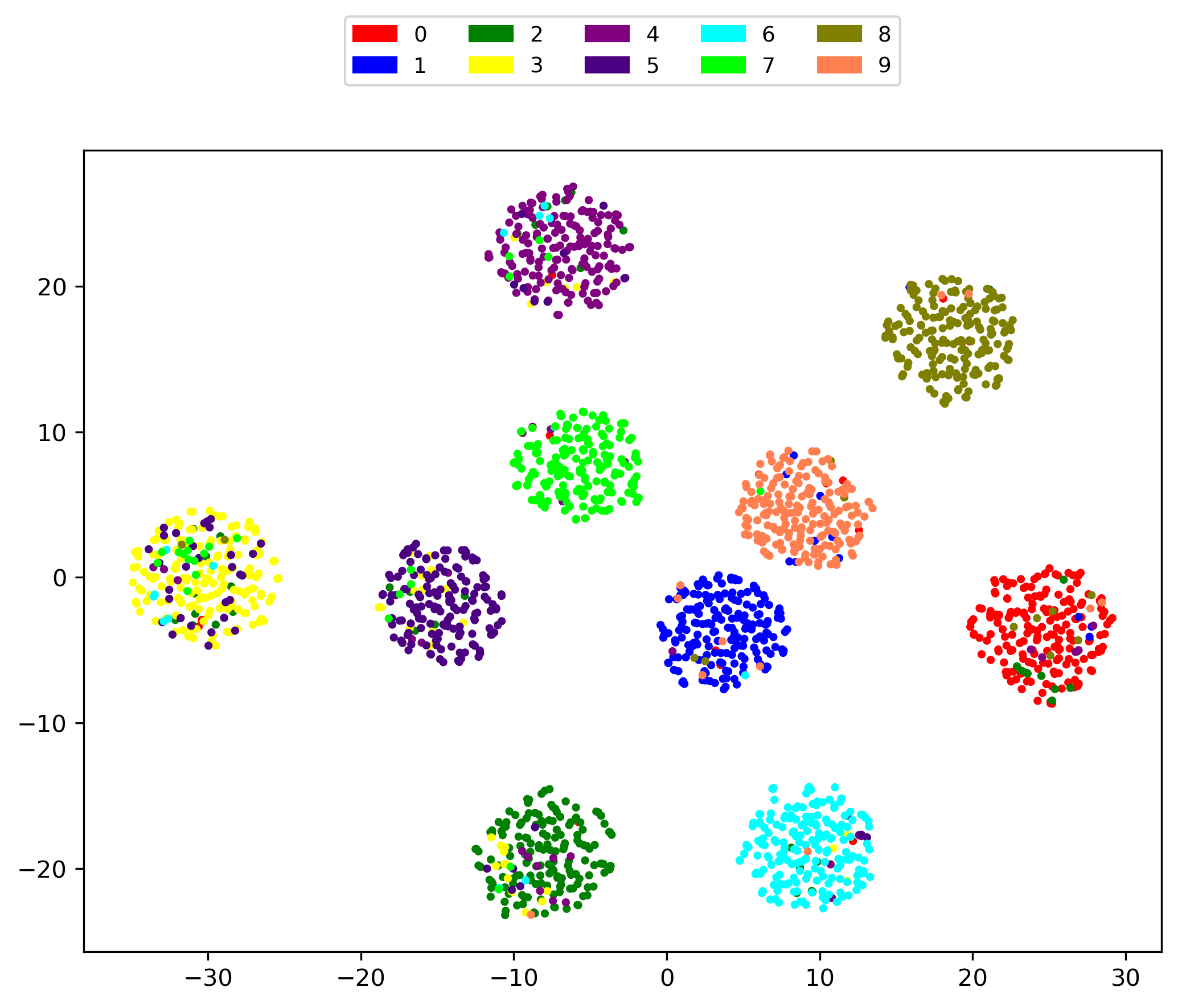}
    \caption{Conv + ReduNet (1000 layers)}
    \label{fig:7_2_c}
    \end{subfigure}
    \hfill
    \begin{subfigure}{0.49\textwidth}
    \centering
    \includegraphics[width=\linewidth]{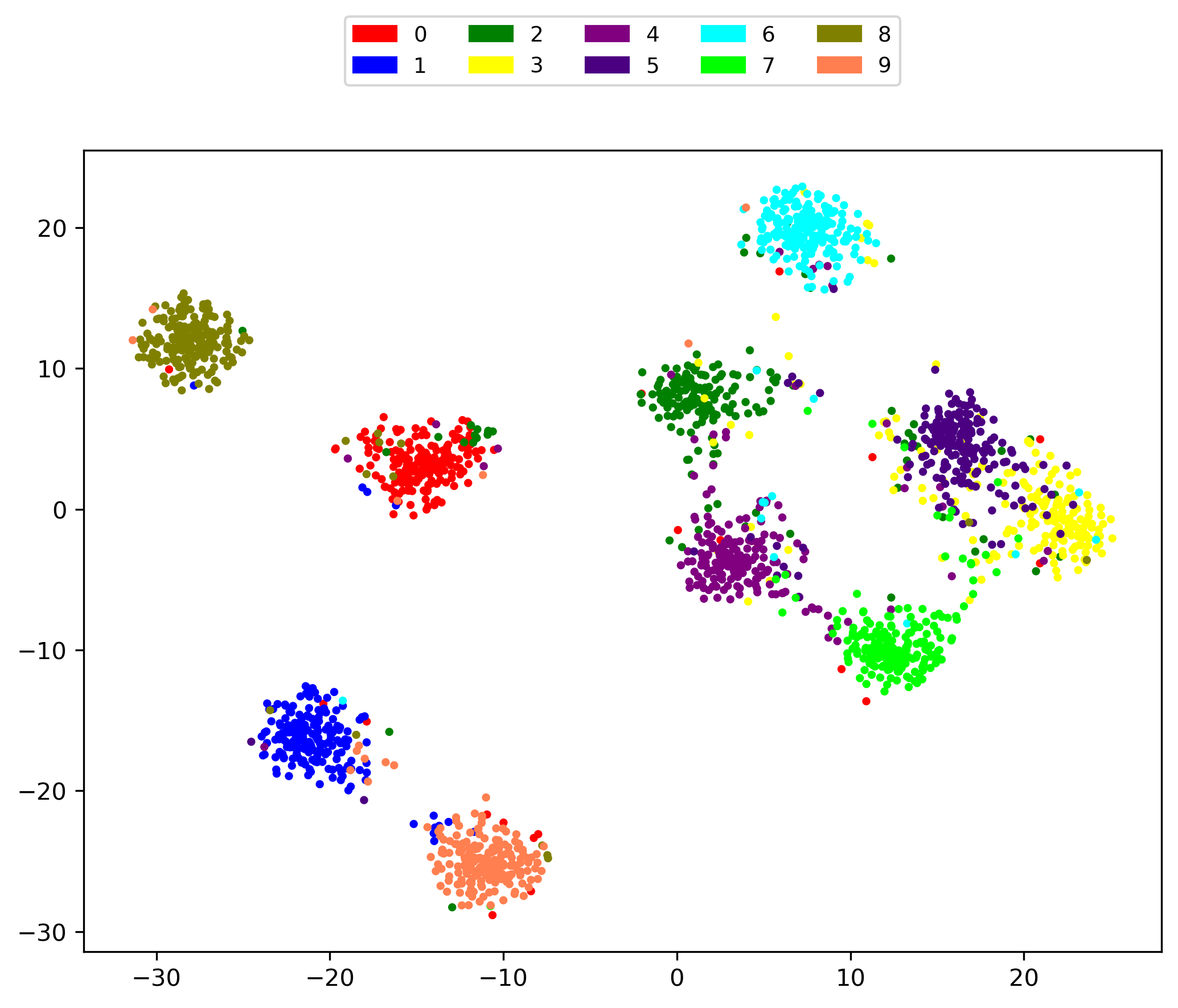}
    \caption{Conv + AR-ReduNet (35 layers)}
    \label{fig:7_2_d}
    \end{subfigure}

    \vspace{0.5em}

    \begin{subfigure}{0.49\textwidth}
    \centering
    \includegraphics[width=\linewidth]{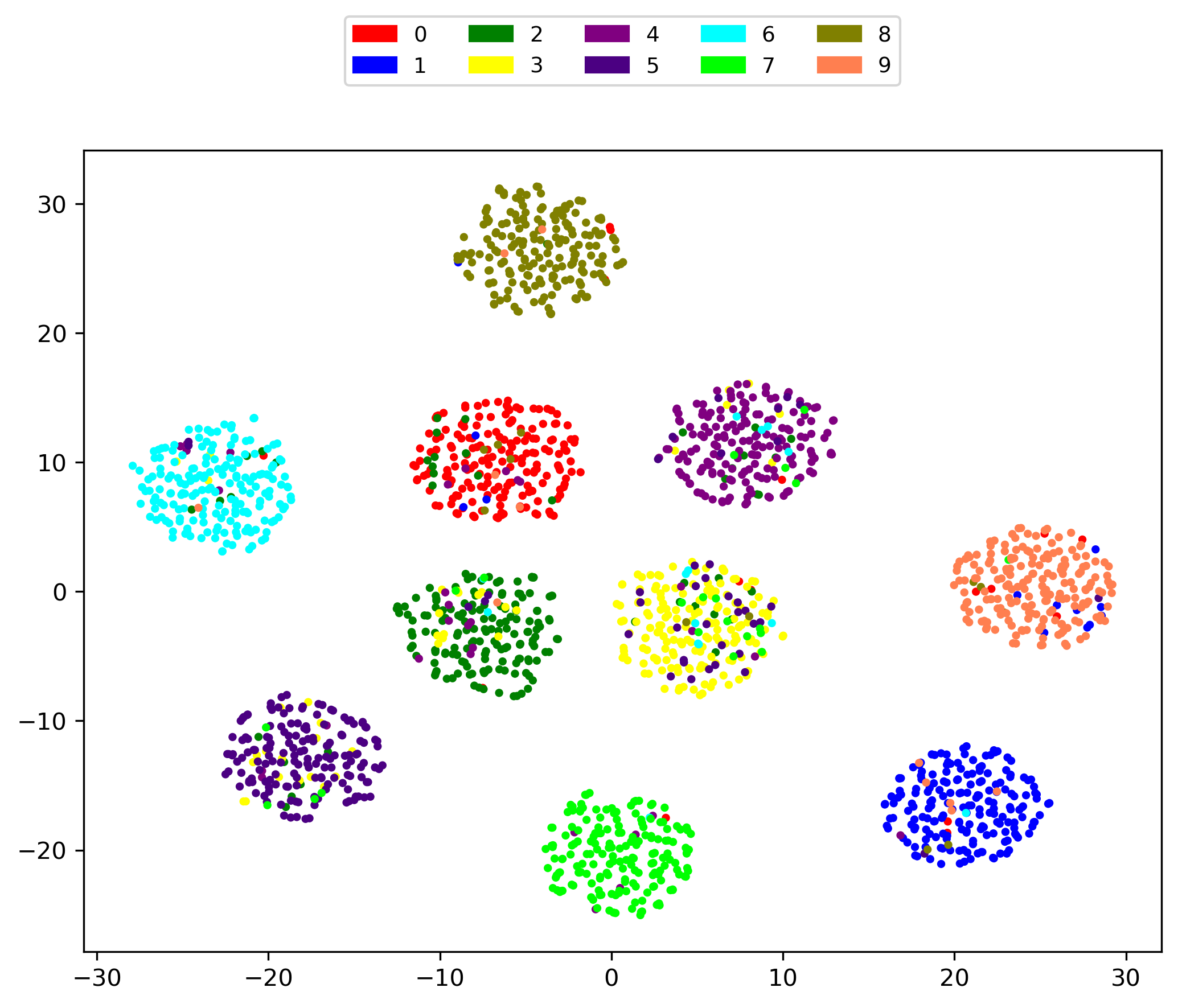}
    \caption{Conv + AR-ReduNet (1000 layers)}
    \label{fig:7_2_e}
    \end{subfigure}
    \hfill
    \begin{subfigure}{0.49\textwidth}
    \centering
    \includegraphics[width=\linewidth]{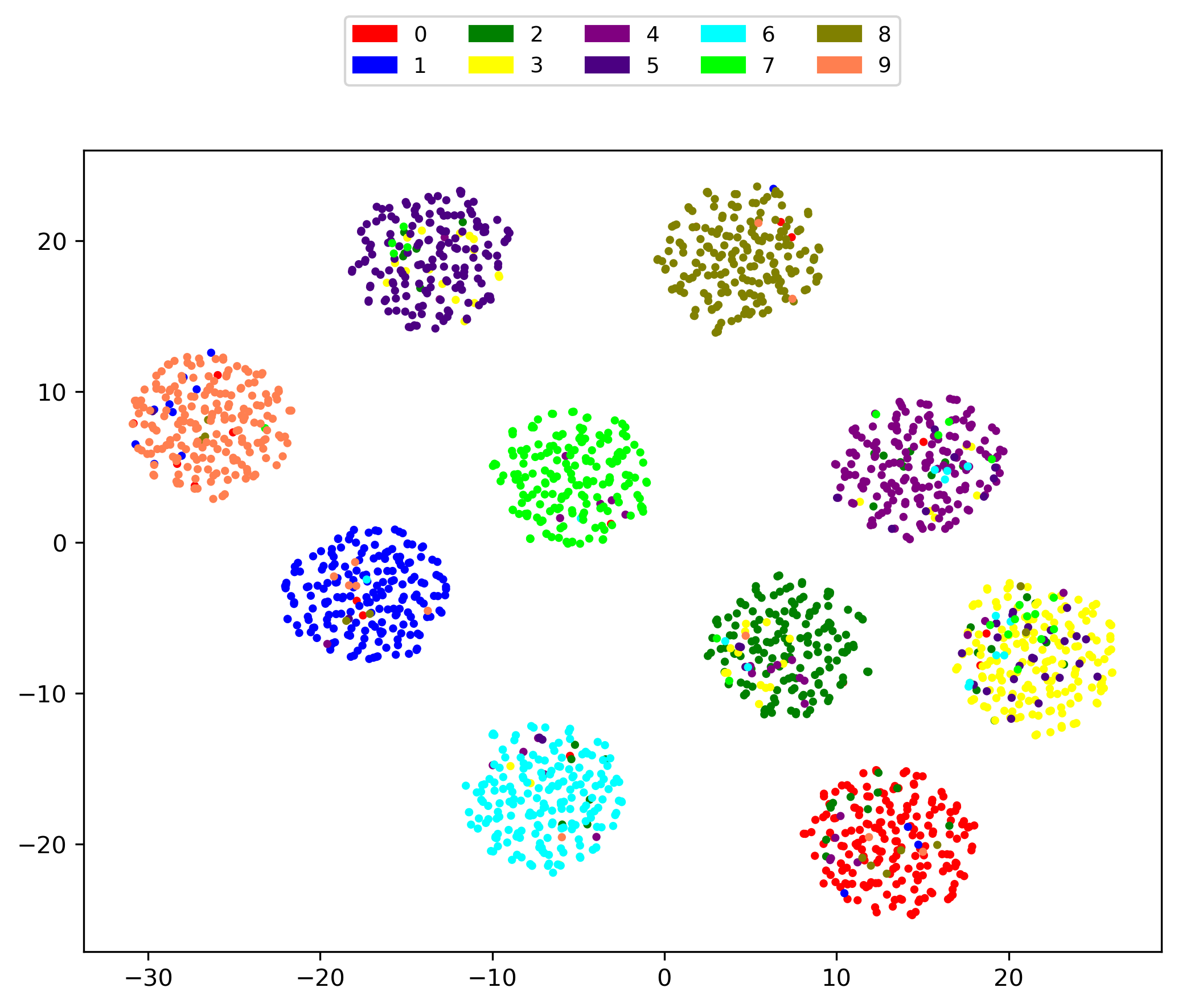}
    \caption{Conv + LA-ReduNet (35 layers)}
    \label{fig:7_2_f}
    \end{subfigure}

    \caption{t-SNE visualization of learned features on the CIFAR-10 dataset.}
    \label{fig:7_2}
\end{figure}

\begin{figure}[p]
    \centering
    \captionsetup[subfigure]{labelformat=parens}

    \begin{subfigure}{0.49\textwidth}
    \centering
    \includegraphics[width=\linewidth]{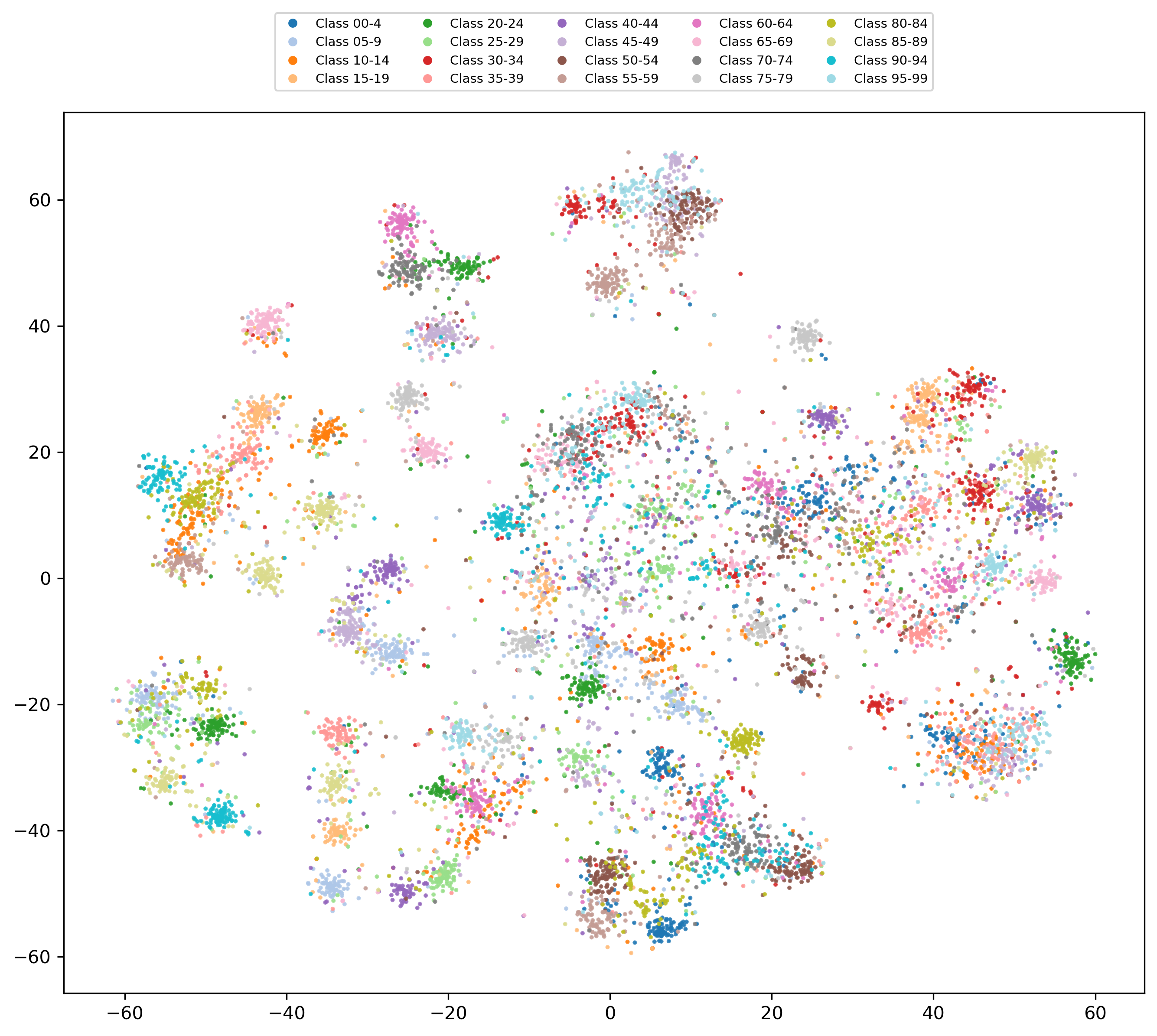}
    \caption{Conv only}
    \label{fig:7_3_a}
    \end{subfigure}
    \hfill
    \begin{subfigure}{0.49\textwidth}
    \centering
    \includegraphics[width=\linewidth]{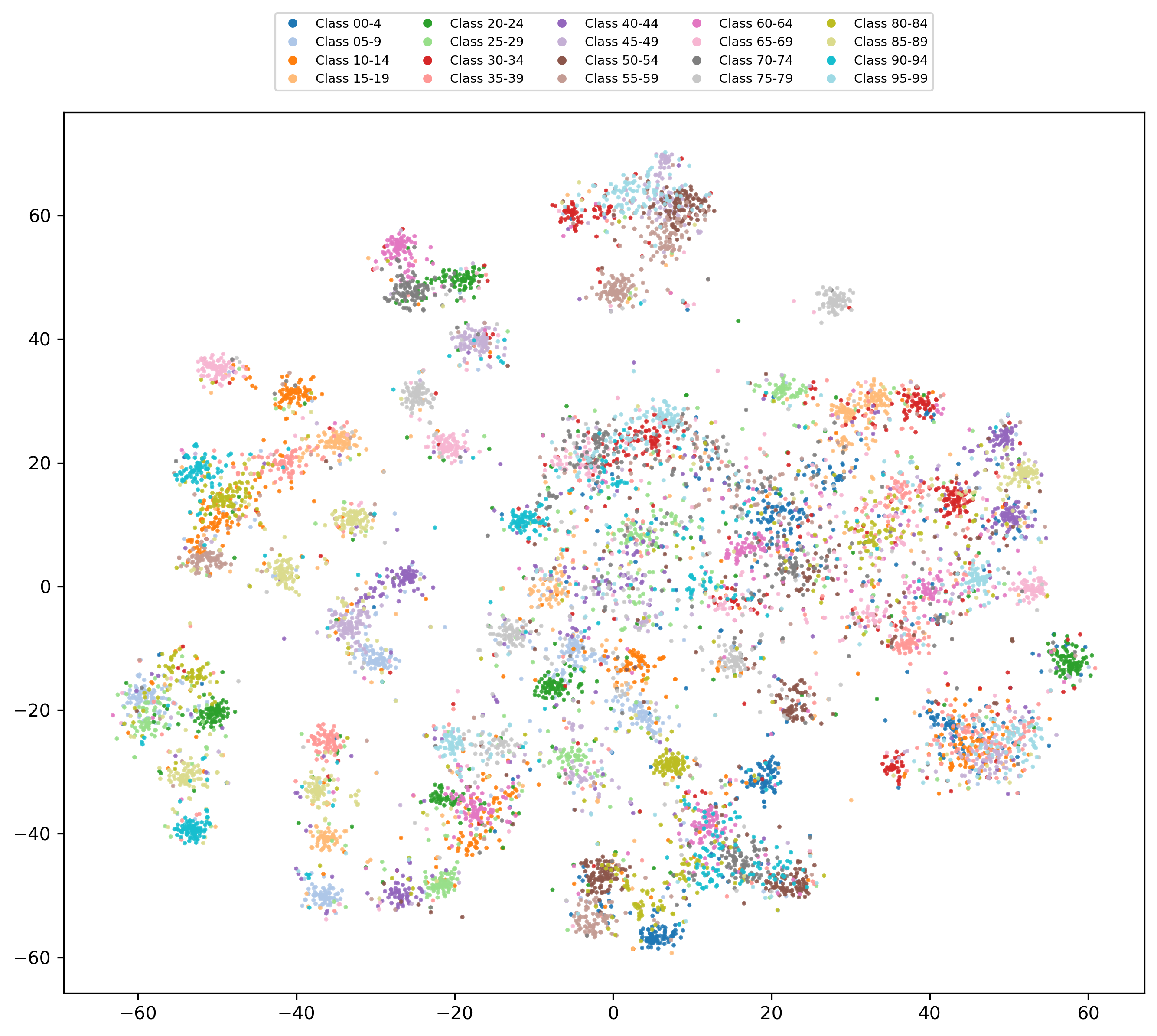}
    \caption{Conv + ReduNet (35 layers)}
    \label{fig:7_3_b}
    \end{subfigure}

    \vspace{0.5em}

    \begin{subfigure}{0.49\textwidth}
    \centering
    \includegraphics[width=\linewidth]{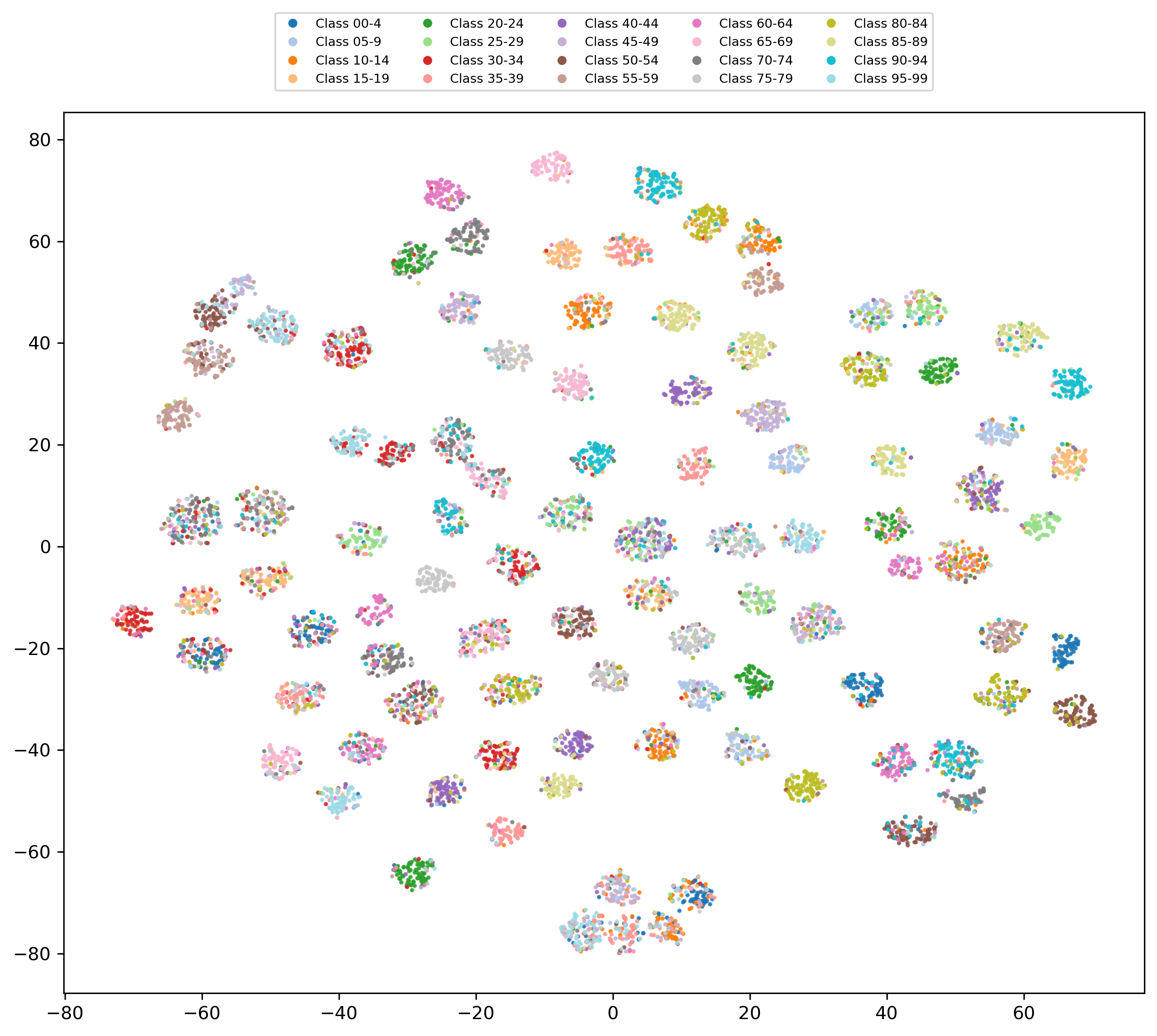}
    \caption{Conv + ReduNet (1000 layers)}
    \label{fig:7_3_c}
    \end{subfigure}
    \hfill
    \begin{subfigure}{0.49\textwidth}
    \centering
    \includegraphics[width=\linewidth]{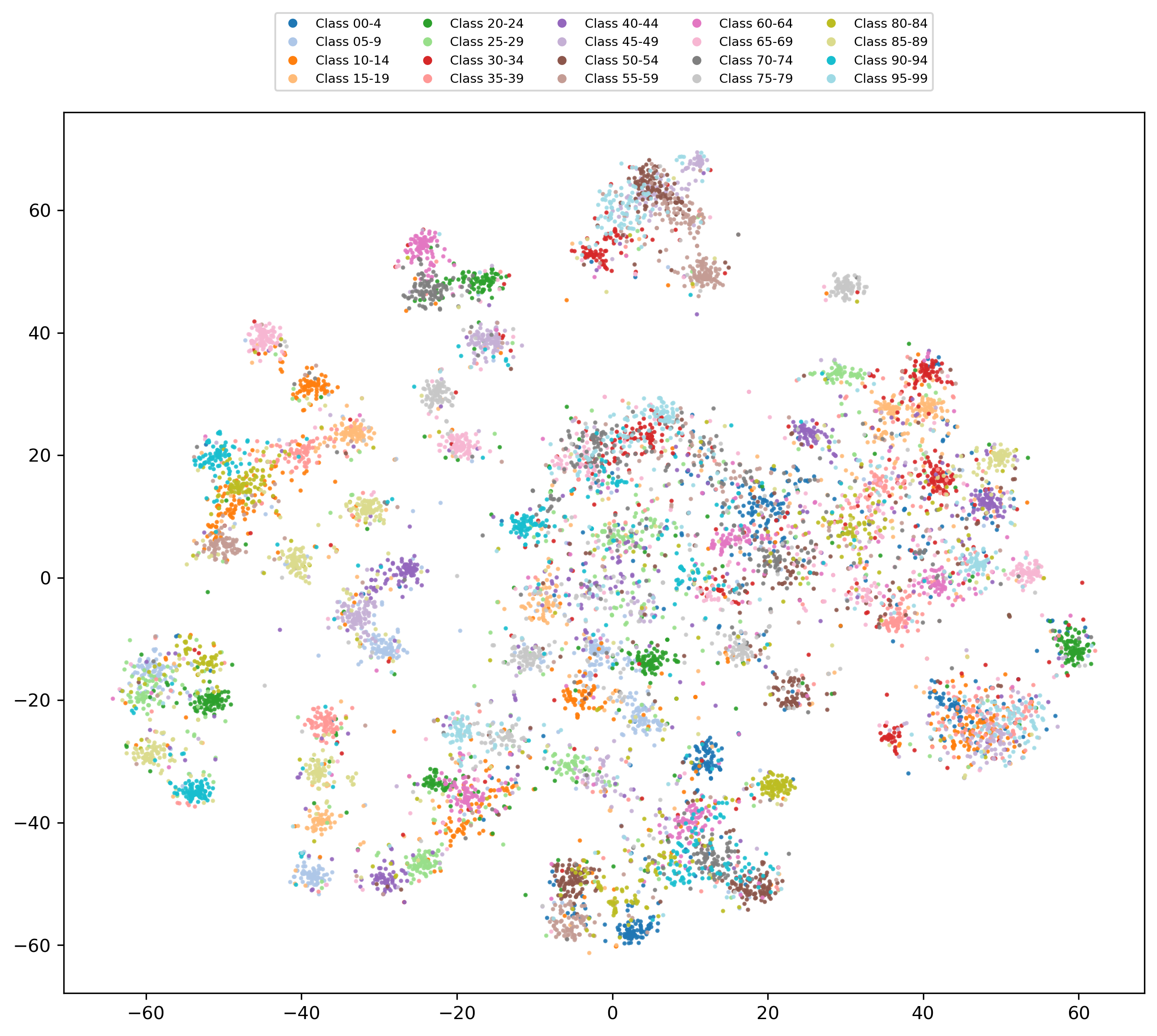}
    \caption{Conv + AR-ReduNet (35 layers)}
    \label{fig:7_3_d}
    \end{subfigure}

    \vspace{0.5em}

    \begin{subfigure}{0.49\textwidth}
    \centering
    \includegraphics[width=\linewidth]{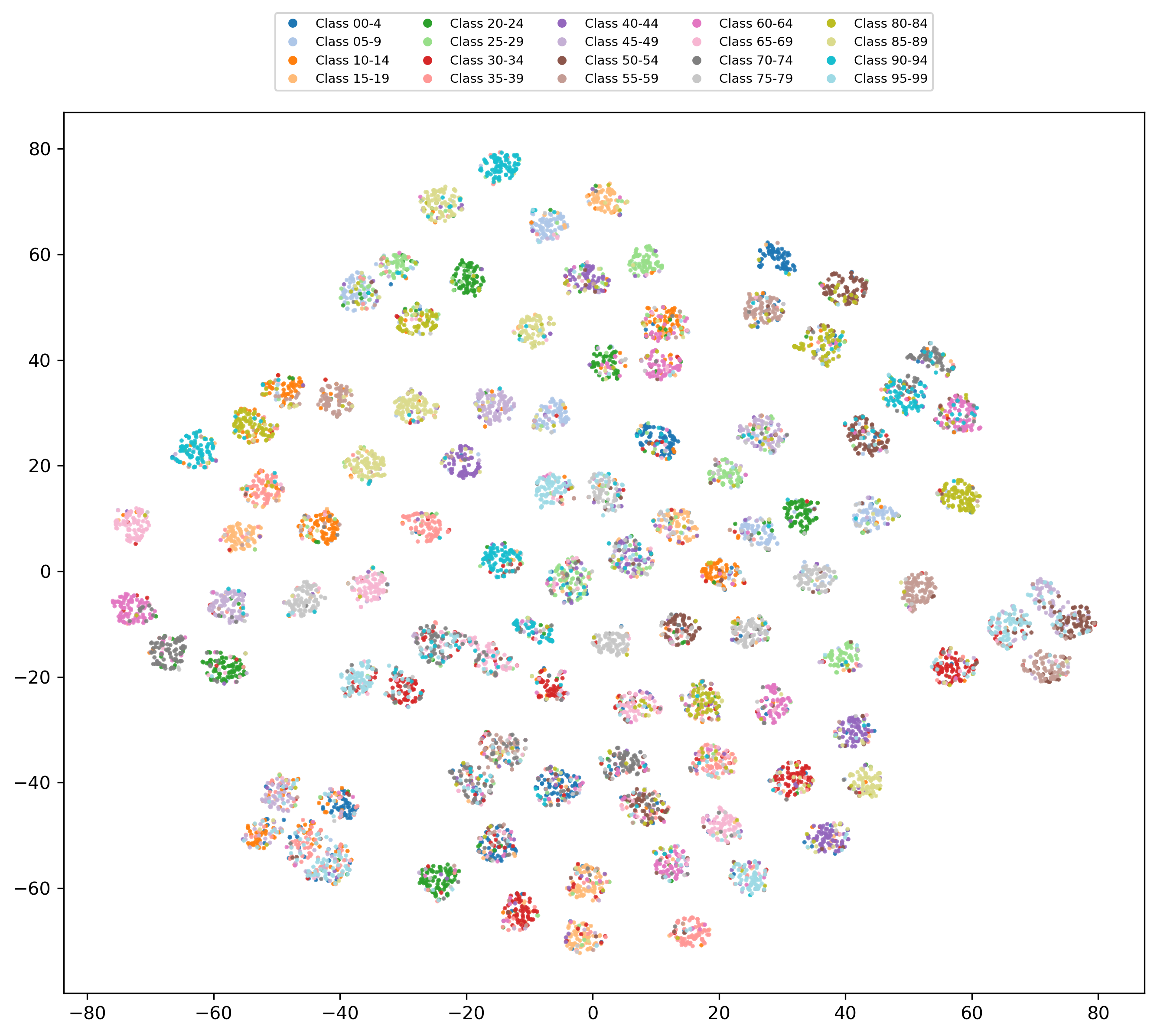}
    \caption{Conv + AR-ReduNet (1000 layers)}
    \label{fig:7_3_e}
    \end{subfigure}
    \hfill
    \begin{subfigure}{0.49\textwidth}
    \centering
    \includegraphics[width=\linewidth]{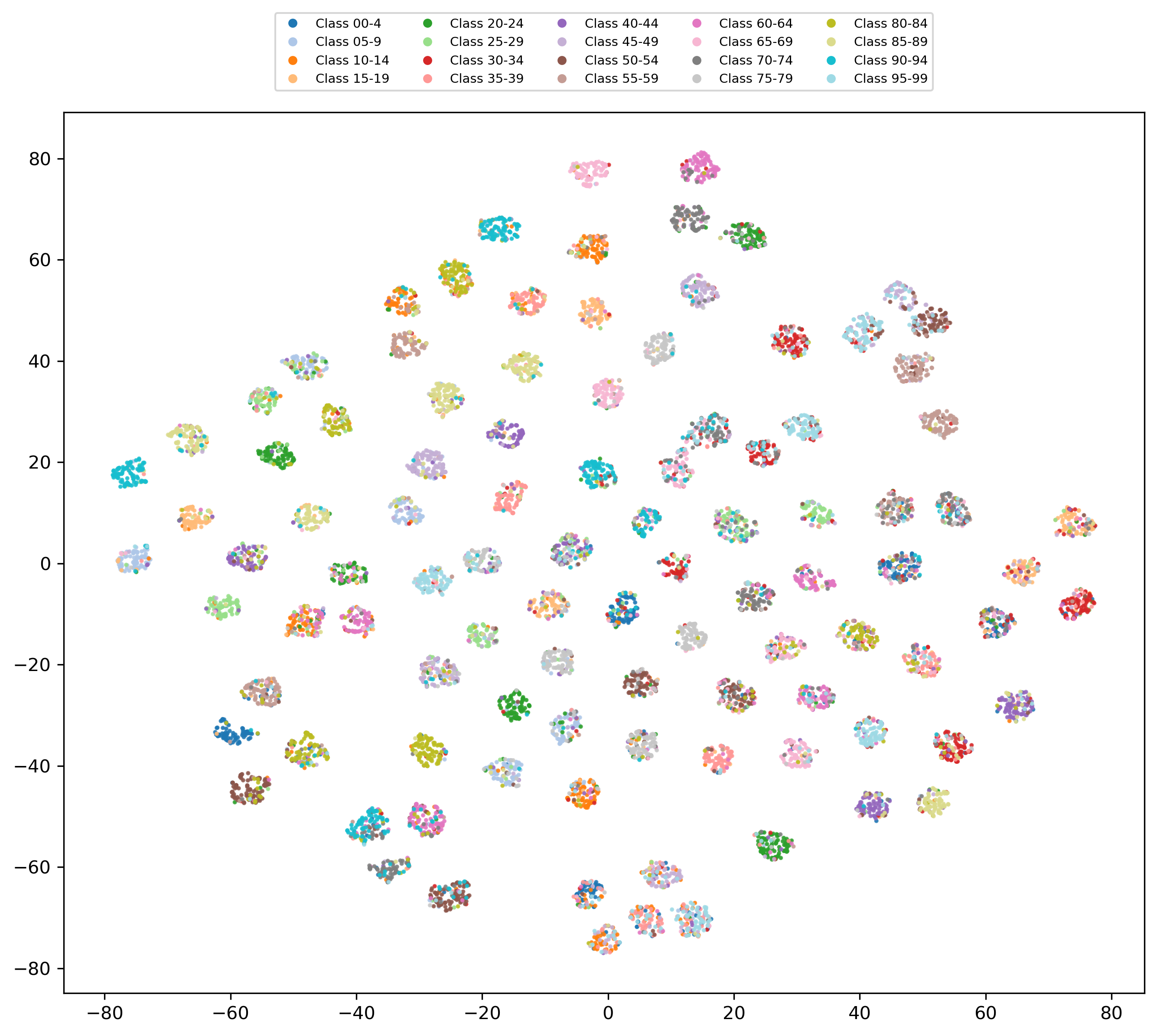}
    \caption{Conv + LA-ReduNet (35 layers)}
    \label{fig:7_3_f}
    \end{subfigure}

    \caption{t-SNE visualization of learned features on the CIFAR-100 dataset.}
    \label{fig:7_3}
\end{figure}

\subsubsection{CINIC-10}
Consistent with the experiments on CIFAR-10 and CIFAR-100, we evaluate the features extracted by ReduNet, AR-ReduNet and LA-ReduNet on CINIC-10. The hyperparameters are set to the values specified above. As illustrated in Fig. \ref{fig:5}, LA-ReduNet achieves substantially faster classification-accuracy convergence than both ReduNet and AR-ReduNet, while also attaining a higher classification accuracy. In particular, our method reaches an accuracy above 78.6$\%$ within only three iterations, whereas the baseline models remain at approximately 77$\%$ even after 50 iterations. In addition, Fig. \ref{fig:5_1} shows that LA-ReduNet achieves objective convergence after approximately 35 layers. Similarly, after objective convergence is reached, the MCR$^2$ objective values of LA-ReduNet occasionally exhibit slight layer-to-layer decreases (first occurring at layer 68). In comparison, ReduNet and AR-ReduNet require nearly 1000 layers to achieve objective convergence. Furthermore, consistent with our experiments on other datasets, we conduct t-SNE visualization for features extracted by all methods on the CINIC-10 dataset to intuitively compare inter-class separability and intra-class compactness. As illustrated in Fig. \ref{fig:5_2}, similar to the experimental results on CIFAR-10 and CIFAR-100, LA-ReduNet is capable of extracting features with superior separability. Notably, unlike CIFAR-10, CINIC-10 contains both CIFAR-10 images and downsampled ImageNet images. Consequently, the features extracted solely by the convolutional layers exhibit weaker inter-class separability. Nevertheless, LA-ReduNet can still transform these features into clearly separable representations using substantially fewer layers than ReduNet and AR-ReduNet.

\begin{figure}[t]
    \centering
    \includegraphics[scale=0.6]{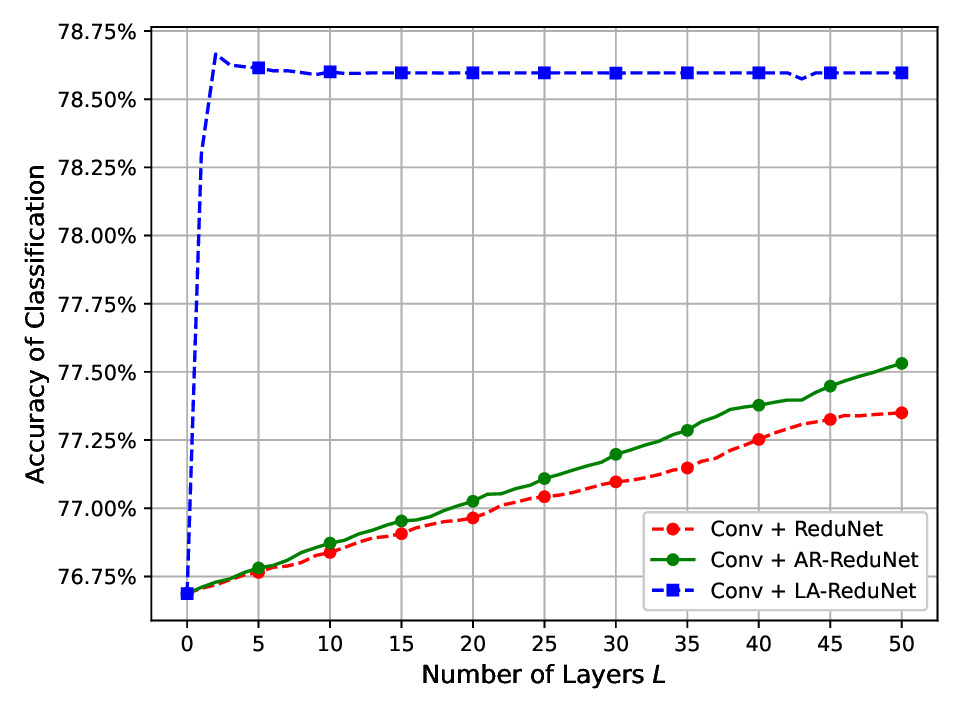}
    \caption{The trends of classification accuracy versus the number of iterative layers L for ReduNet, AR-ReduNet, and LA-ReduNet on the CINIC-10 dataset.}
    \label{fig:5}
\end{figure}

\begin{figure}[t!]
    \centering
    \includegraphics[scale=0.6]{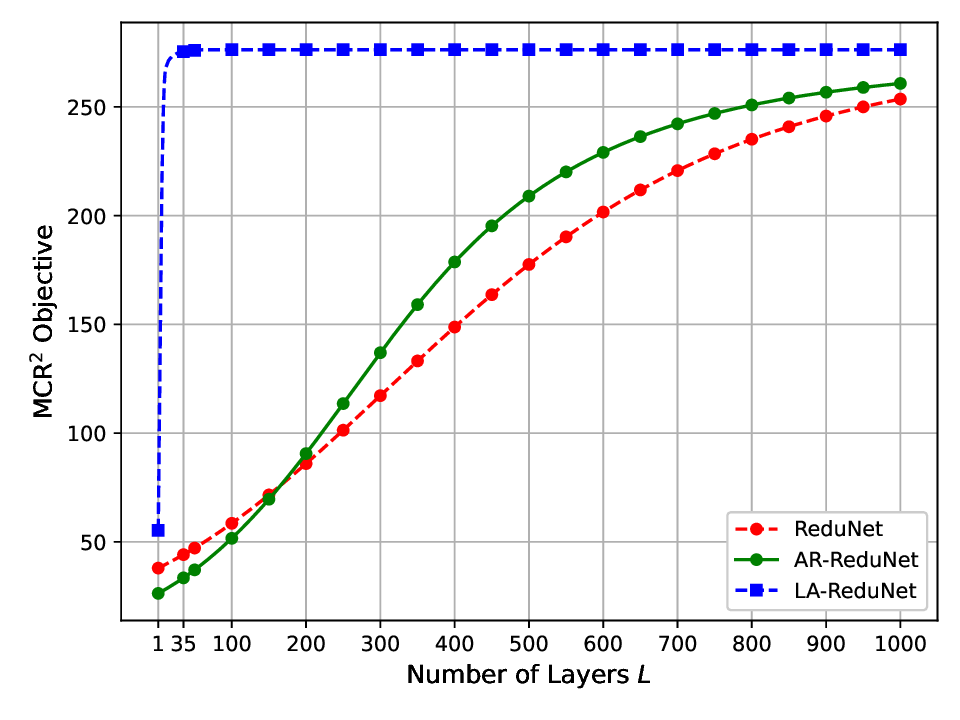}
    \caption{Curves of MCR$^2$ objective values during training on CINIC-10.}
    \label{fig:5_1}
\end{figure}

\begin{figure}[t!]
    \centering
    \captionsetup[subfigure]{labelformat=parens}

    \begin{subfigure}{0.49\textwidth}
    \centering
    \includegraphics[width=\linewidth]{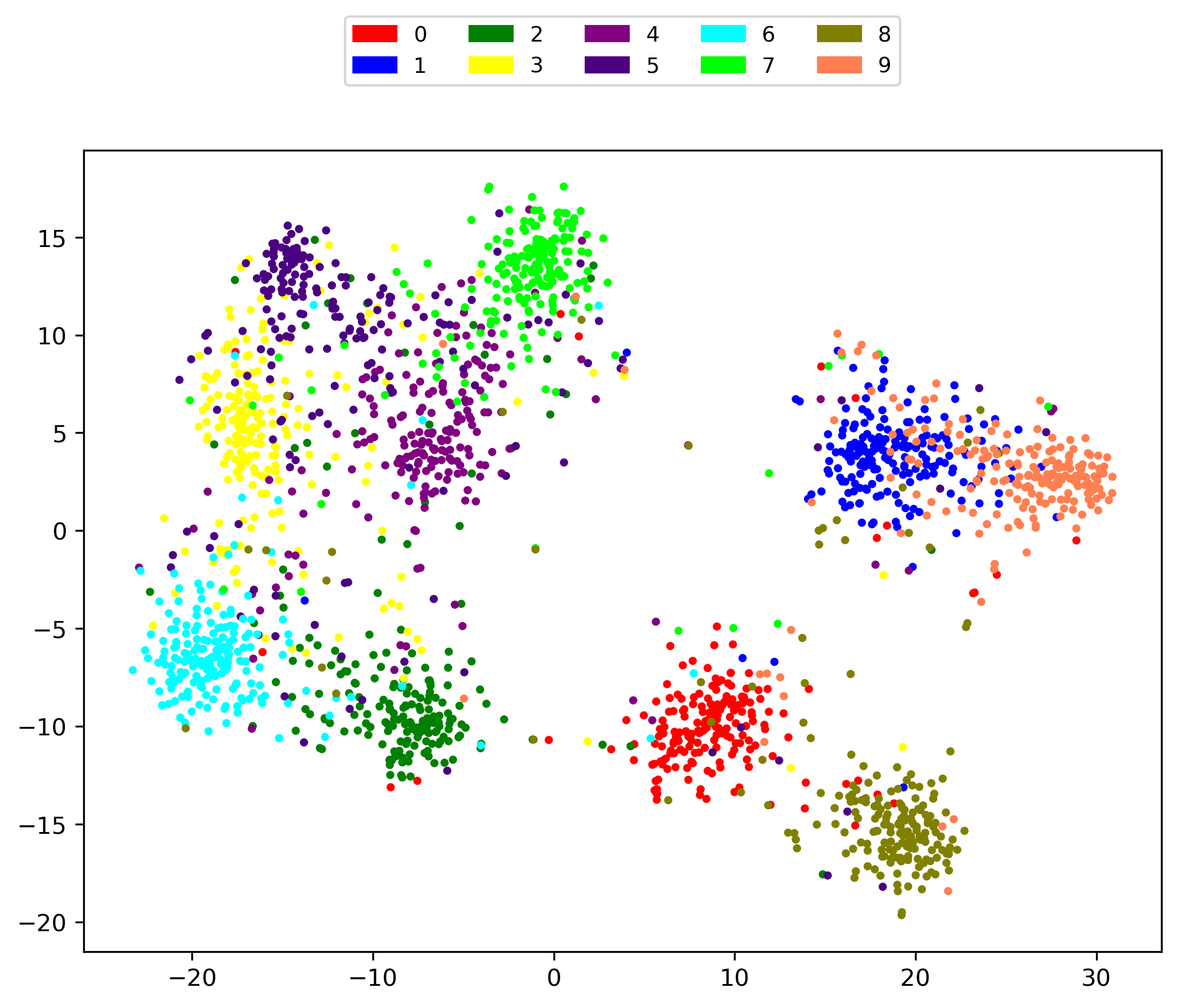}
    \caption{Conv only}
    \label{fig:5_2_a}
    \end{subfigure}
    \hfill
    \begin{subfigure}{0.49\textwidth}
    \centering
    \includegraphics[width=\linewidth]{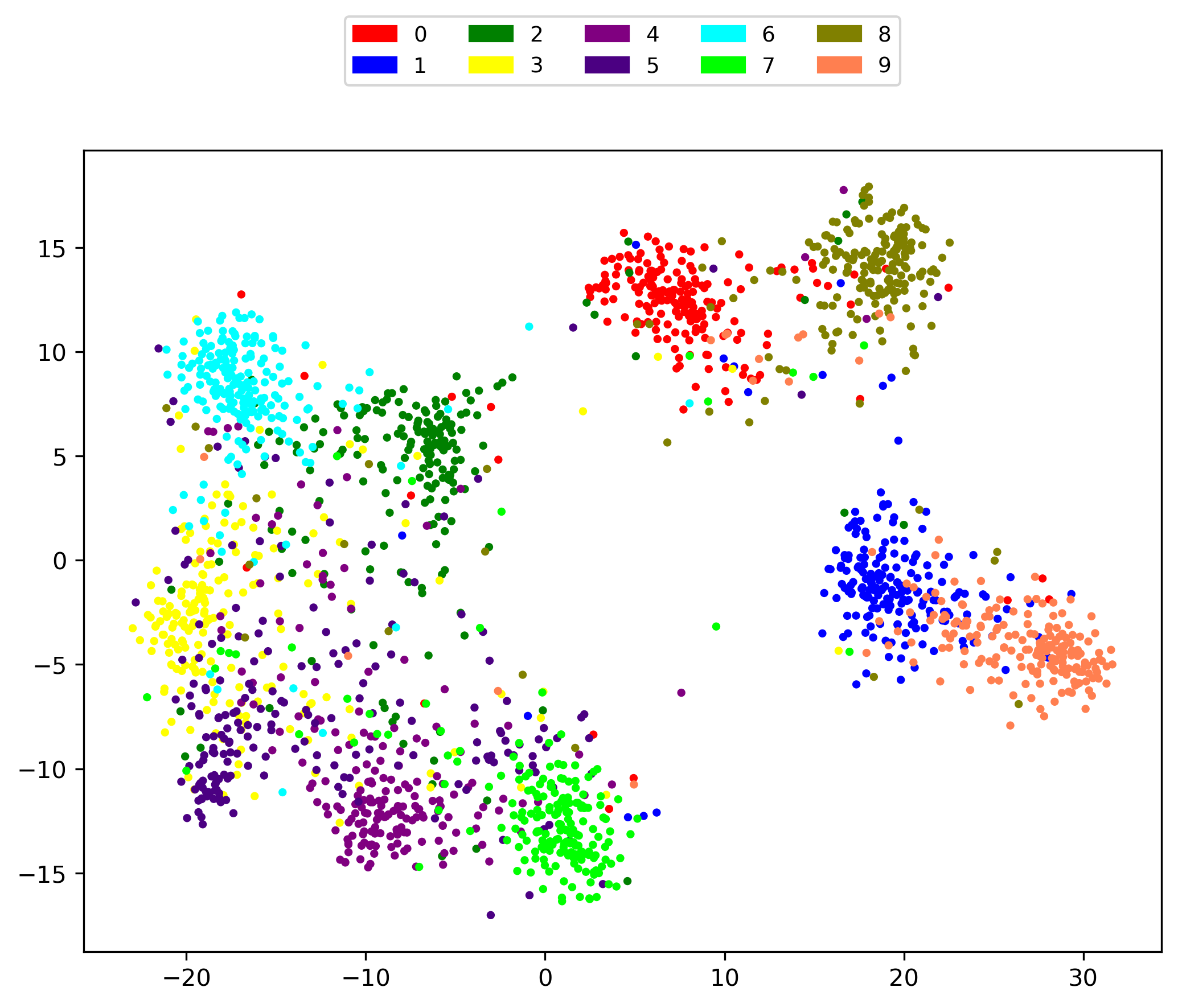}
    \caption{Conv + ReduNet (35 layers)}
    \label{fig:5_2_b}
    \end{subfigure}

    \vspace{0.5em}

    \begin{subfigure}{0.49\textwidth}
    \centering
    \includegraphics[width=\linewidth]{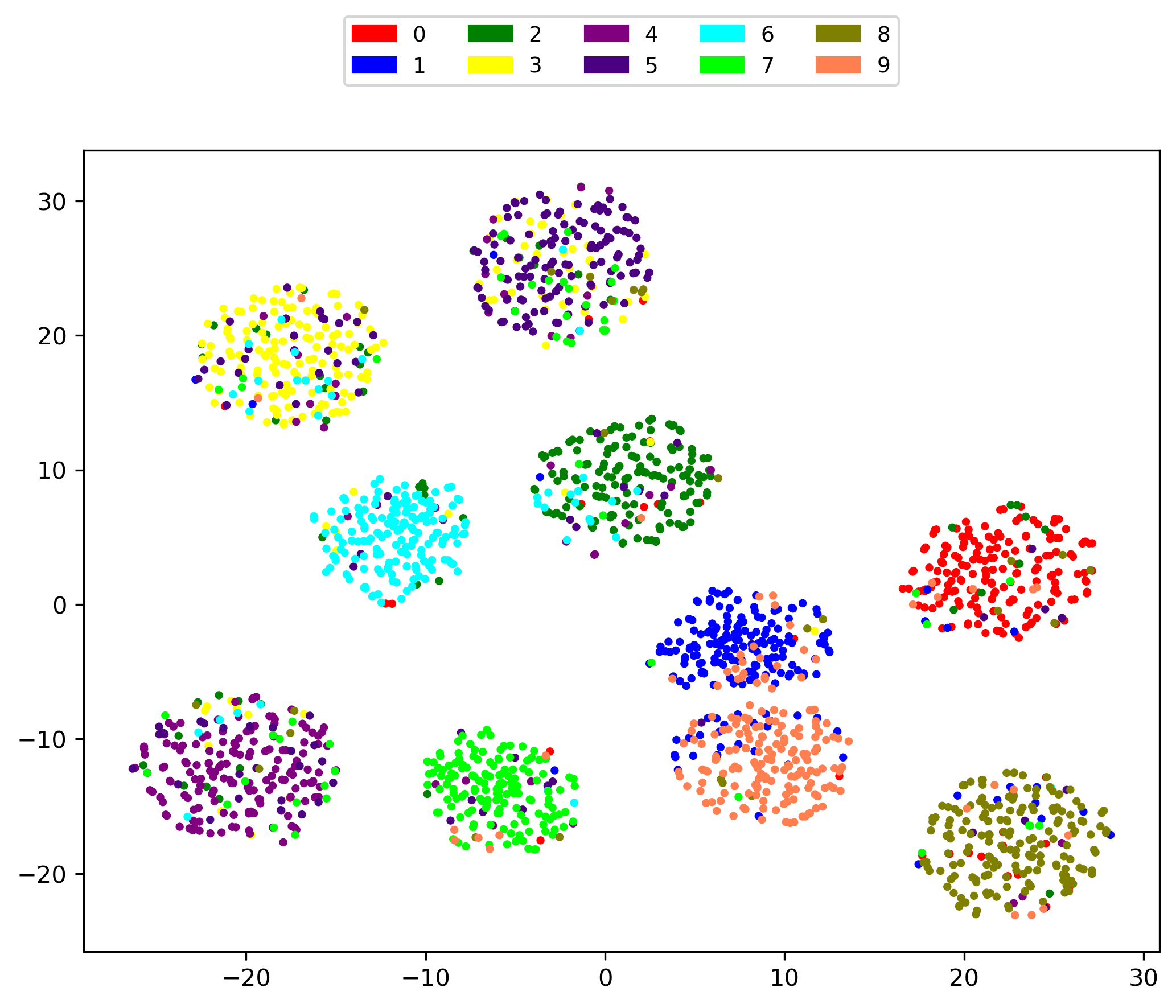}
    \caption{Conv + ReduNet (1000 layers)}
    \label{fig:5_2_c}
    \end{subfigure}
    \hfill
    \begin{subfigure}{0.49\textwidth}
    \centering
    \includegraphics[width=\linewidth]{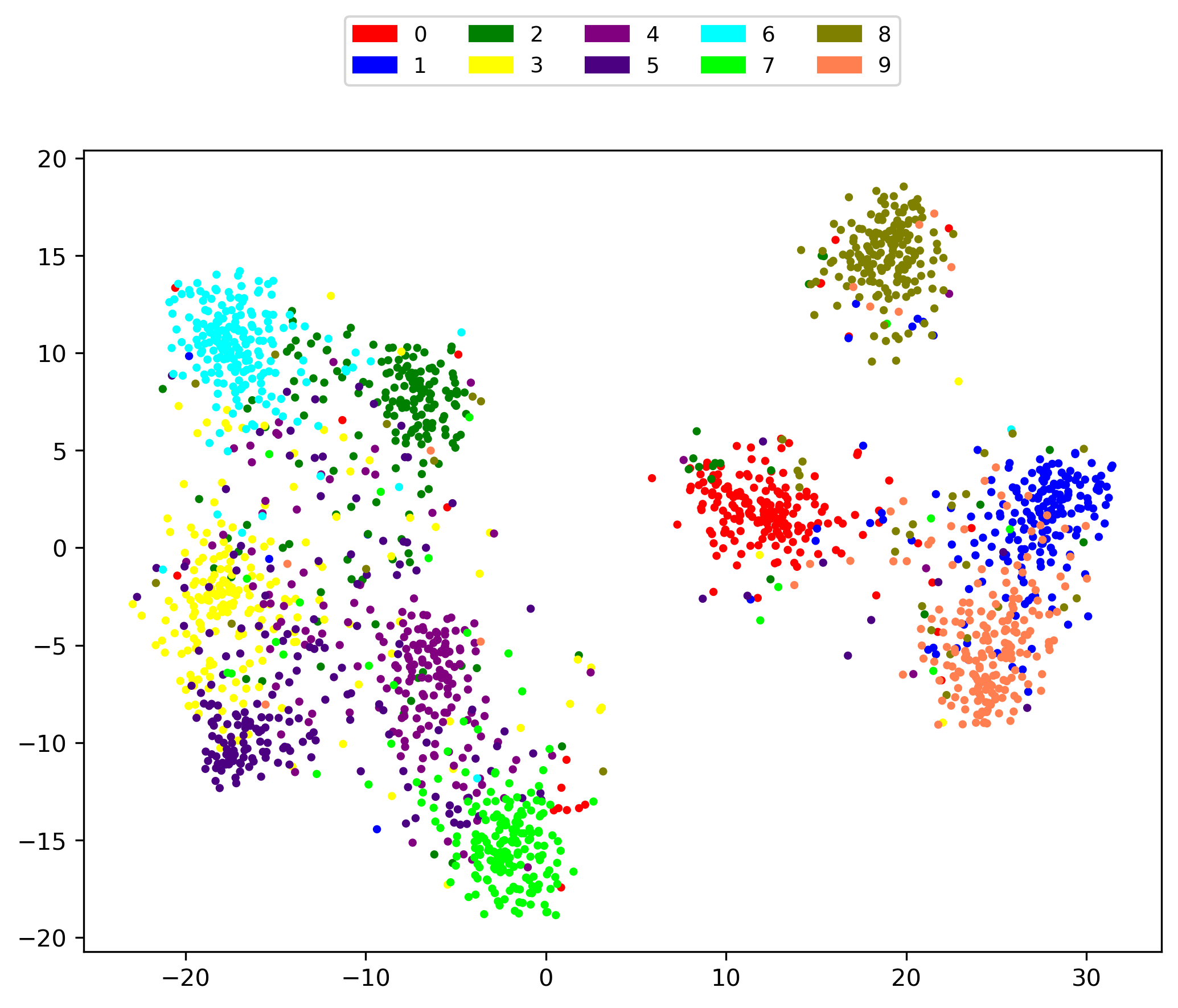}
    \caption{Conv + AR-ReduNet (35 layers)}
    \label{fig:5_2_d}
    \end{subfigure}

    \vspace{0.5em}

    \begin{subfigure}{0.49\textwidth}
    \centering
    \includegraphics[width=\linewidth]{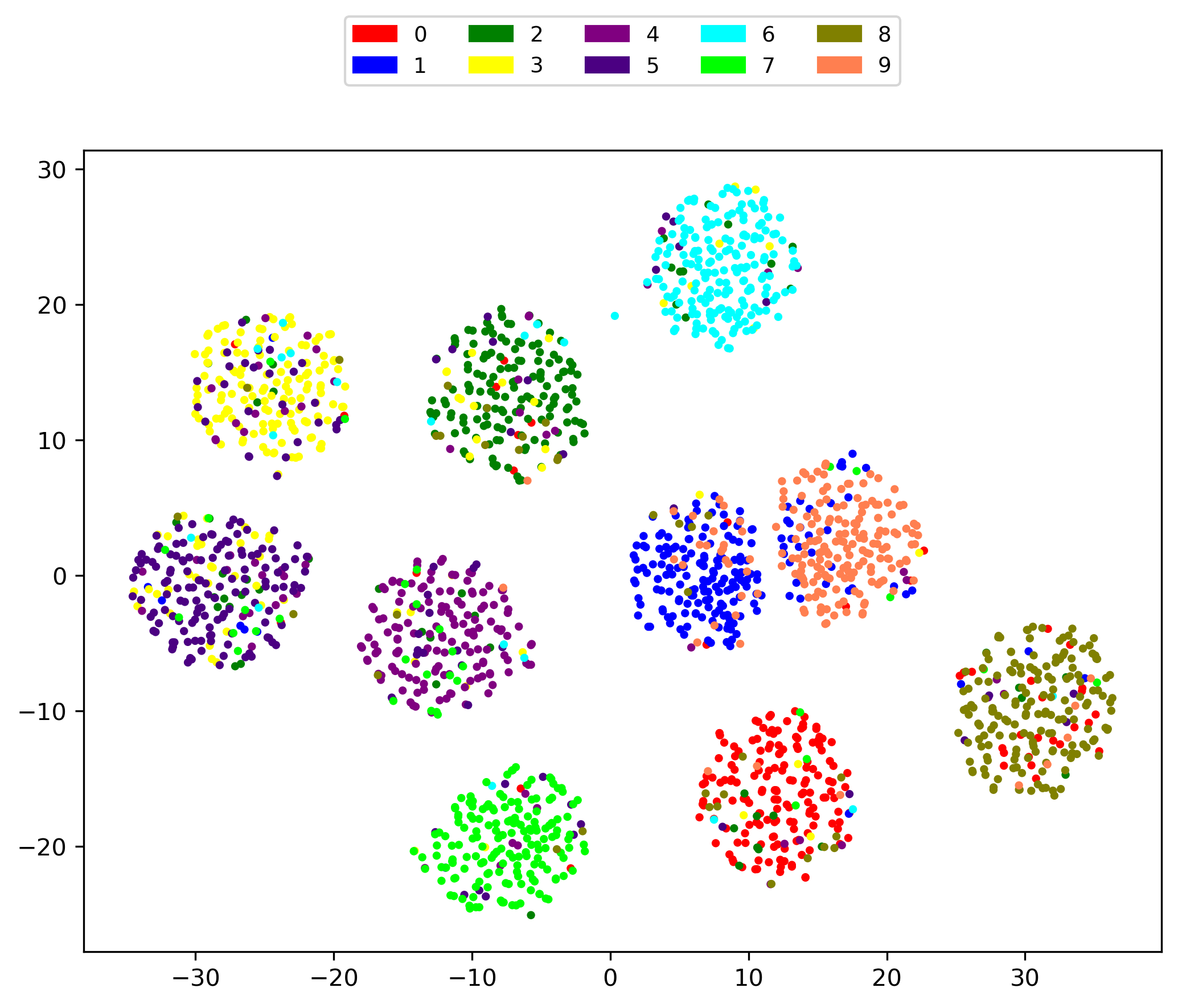}
    \caption{Conv + AR-ReduNet (1000 layers)}
    \label{fig:5_2_e}
    \end{subfigure}
    \hfill
    \begin{subfigure}{0.49\textwidth}
    \centering
    \includegraphics[width=\linewidth]{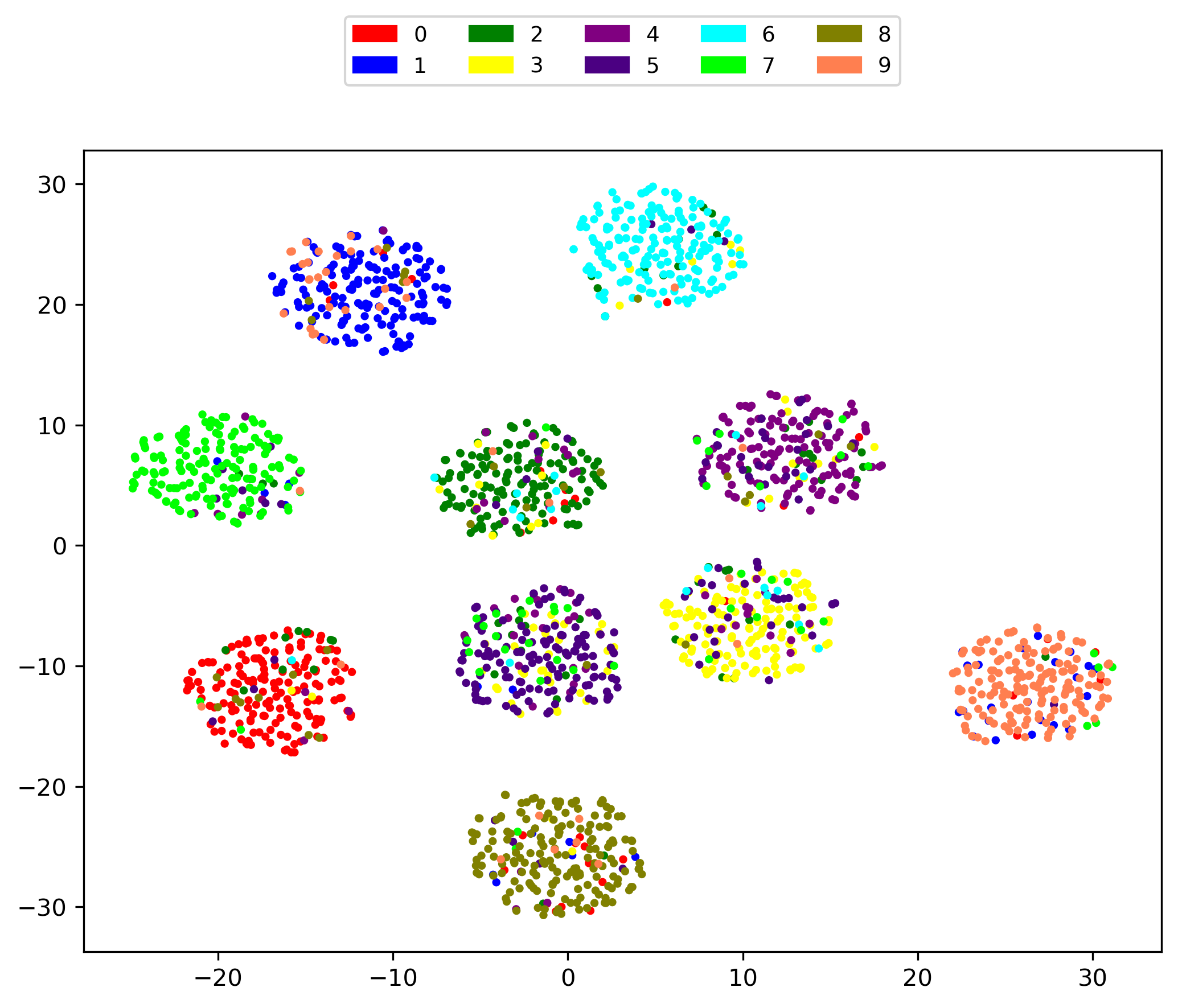}
    \caption{Conv + LA-ReduNet (35 layers)}
    \label{fig:5_2_f}
    \end{subfigure}

    \caption{t-SNE visualization of learned features on the CINIC-10 dataset.}
    \label{fig:5_2}
\end{figure}

In summary, we compare the parameter storage of the unfolded modules of LA-ReduNet, ReduNet, and AR-ReduNet, while excluding the convolutional front-end, which is shared by all three methods. For each layer, the stored parameters correspond to the expansion and compression matrices defined in \eqref{cal:E} and \eqref{cal:Cj}. In the experiments, since the common convolutional front-end reduces the feature dimension to $d$, the corresponding matrices in the unfolded modules are of size $d\times d$, and the three methods have the same parameter storage per layer. The scalar parameters $\alpha$ and $\alpha_j$ are recomputed from the current feature matrix $\bm Z$ at each layer and therefore do not need to be stored as model parameters. As a result, the total parameter storage of the unfolded module is mainly determined by the number of layers. As shown in Fig. \ref{fig:7_1} and Fig. \ref{fig:5_1}, the MCR$^2$ objective of LA-ReduNet reaches a stable value after approximately $35$ layers. Under the original parameter settings of ReduNet and AR-ReduNet, the two baseline methods require approximately $1000$ layers for their respective MCR$^2$ objectives to reach stable values. Consequently, under these settings, the parameter storage of LA-ReduNet is approximately $35/1000\approx 1/29$ of that required by the corresponding unfolded baselines. It is also worth noting that classification-accuracy convergence is generally achieved earlier, typically within approximately 5--10 layers for LA-ReduNet.

\subsection{Ablation Studies and Hyperparameter Analysis}

\begin{figure}[t]
\centering
\captionsetup[subfigure]{labelformat=parens}

\begin{subfigure}{0.47\textwidth}
\centering
\includegraphics[width=\linewidth]{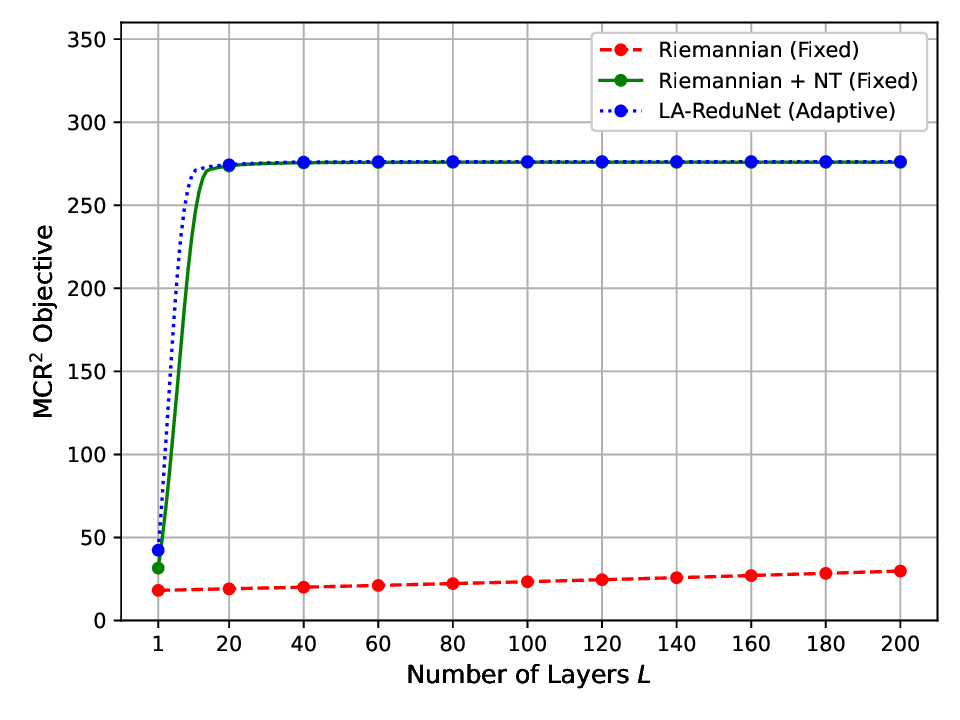}
\caption{CIFAR-10}
\label{fig:15_a}
\end{subfigure}
\hfill
\begin{subfigure}{0.47\textwidth}
\centering
\includegraphics[width=\linewidth]{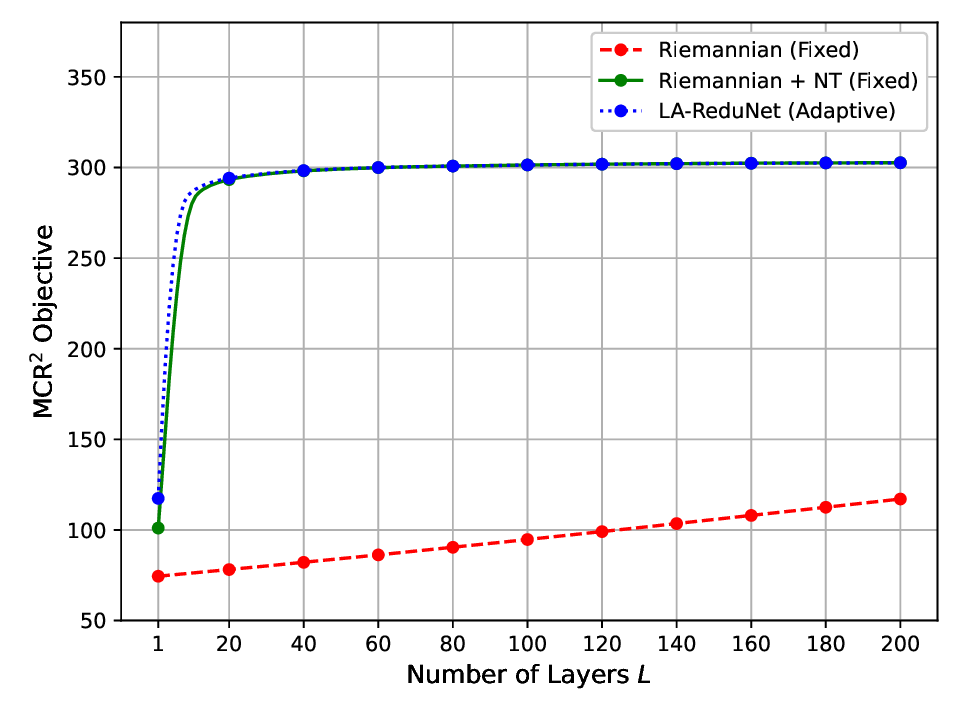}
\caption{CIFAR-100}
\label{fig:15_b}
\end{subfigure}

\vspace{0.5em}

\begin{subfigure}{0.47\textwidth}
\centering
\includegraphics[width=\linewidth]{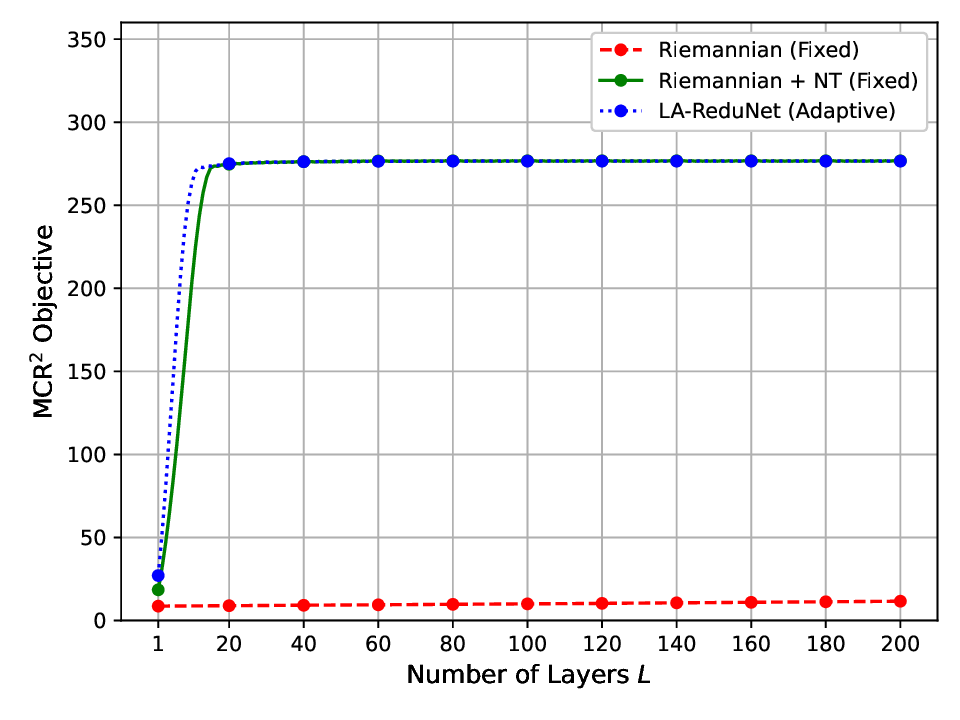}
\caption{CINIC-10}
\label{fig:15_c}
\end{subfigure}
\caption{Effect of different components on the convergence of the MCR$^2$ objective.}
\label{fig:15}
\end{figure}

In this subsection, we first conduct component-wise ablation studies to evaluate the effects of the Riemannian update, the normalization-threshold (NT) mechanism, and the adaptive step-size strategy on the objective convergence of LA-ReduNet. We then investigate the effects of the LA-ReduNet parameters $t_0$ and $\beta$ on objective convergence, and analyze the influence of the loss-weighting parameter $\mu$ in the convolutional module on classification accuracy.

\begin{table}[t!]
  \centering
  \begin{tabular}{lccc}
    \toprule
    Variant
    & Update direction
    & NT
    & Step-size parameter \\
    \midrule
    Riemannian (Fixed)
    & $\bm g_{\mathrm{T}}$
    & No
    & $t=t_0$ \\

    Riemannian + NT (Fixed)
    & $\bm g_{\mathrm{T}}/\|\bm g_{\mathrm{T}}\|_2$
    & Yes
    & $t=t_0$ \\

    LA-ReduNet (Adaptive)
    & $\bm g_{\mathrm{T}}/\|\bm g_{\mathrm{T}}\|_2$
    & Yes
    & Eq. \eqref{Ad_t} \\
    \bottomrule
  \end{tabular}
  \caption{Configurations of the variants used in the component-wise ablation study.}
  \label{tab:ablation_variants}
\end{table}

\subsubsection{Component-wise Ablation Study of LA-ReduNet}

To evaluate the contributions of the key components in LA-ReduNet, we construct three variants by progressively introducing the NT mechanism and the adaptive step-size strategy on top of the basic Riemannian update. The configurations of these variants are summarized in Table \ref{tab:ablation_variants}. As illustrated in Fig. \ref{fig:15}, the basic Riemannian update without the NT mechanism exhibits considerably slower objective convergence. Introducing the NT mechanism markedly improves the convergence behavior by reducing the dependence of the spherical angular progress on the tangent-update magnitude, while the threshold retains the stopping information carried by this magnitude. With the NT mechanism, the objective eventually reaches a similar level to that of the full LA-ReduNet, while the adaptive step-size strategy provides further improvement in the early layers. These results indicate that the NT mechanism provides the primary improvement in objective convergence, whereas the adaptive step-size strategy further improves the progress in the early layers. Notice that a relatively conservative value $\beta=1$ is adopted in these experiments, which limits the variation range of the adaptive step size and therefore results in relatively close convergence curves between LA-ReduNet and the Riemannian + NT variant in Fig. \ref{fig:15}. As further shown in the subsequent analysis of $\beta$, the advantage of the adaptive step-size strategy becomes more pronounced when a larger value of $\beta$ is used.

\begin{figure}[t]
\centering
\captionsetup[subfigure]{labelformat=parens}

\begin{subfigure}{0.47\textwidth}
\centering
\includegraphics[width=\linewidth]{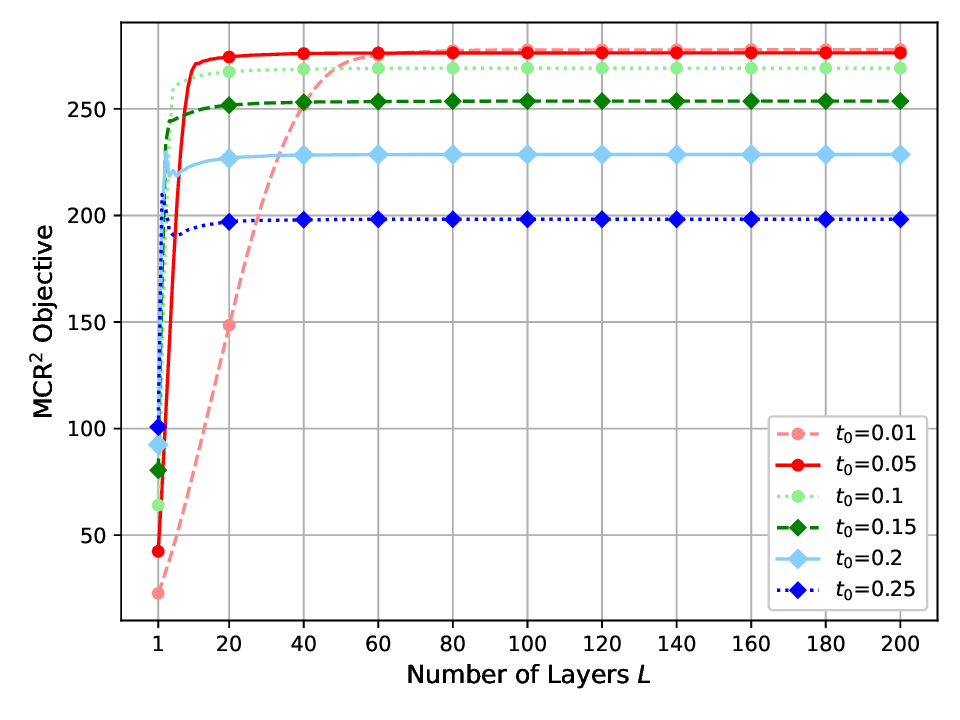}
\caption{CIFAR-10}
\label{fig:11_b}
\end{subfigure}
\hfill
\begin{subfigure}{0.47\textwidth}
\centering
\includegraphics[width=\linewidth]{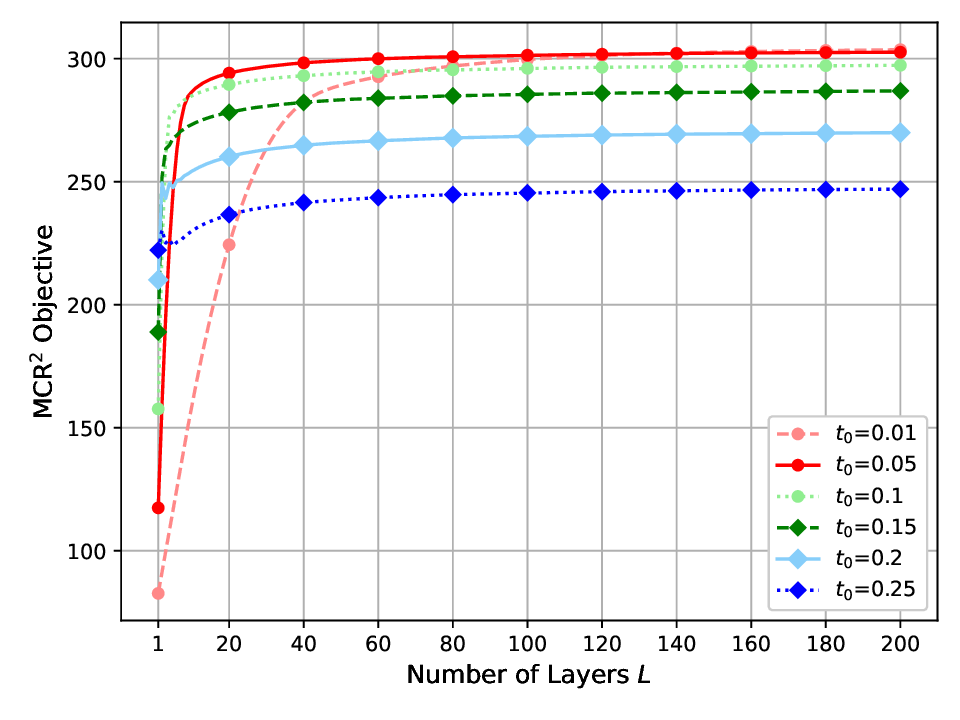}
\caption{CIFAR-100}
\label{fig:11_c}
\end{subfigure}

\vspace{0.5em}

\begin{subfigure}{0.47\textwidth}
\centering
\includegraphics[width=\linewidth]{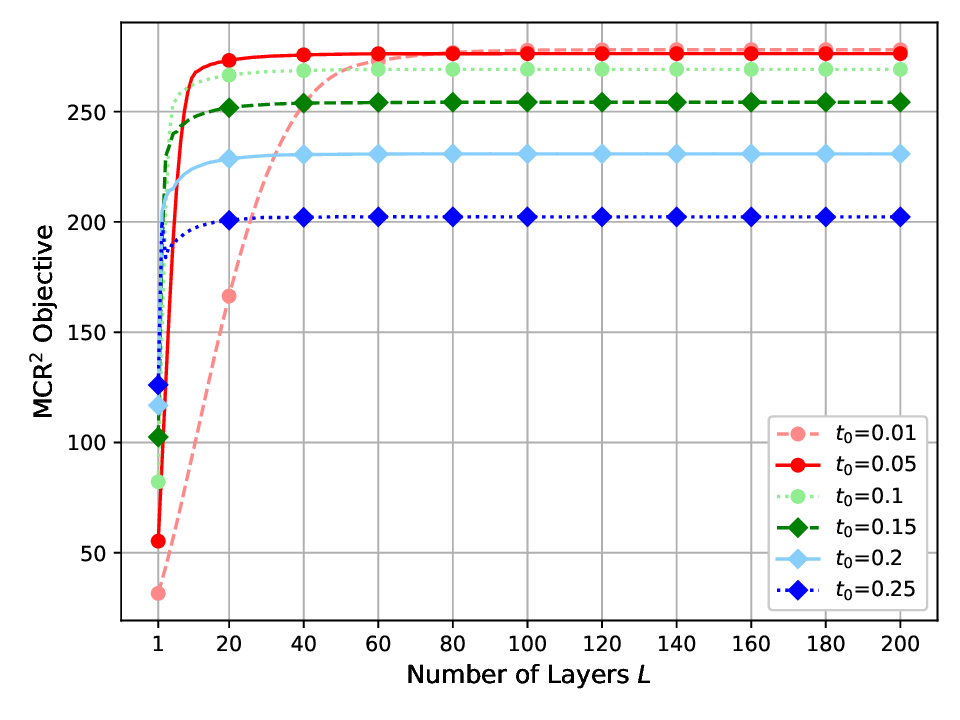}
\caption{CINIC-10}
\label{fig:11_a}
\end{subfigure}
\caption{Effect of parameter $t_0$ on the convergence of the MCR$^2$ objective of LA-ReduNet.}
\label{fig:11}
\end{figure}

\subsubsection{Impact of Hyperparameter $t_0$ on the Objective Convergence of LA-ReduNet}

\begin{figure}[t]
\centering
\captionsetup[subfigure]{labelformat=parens}

\begin{subfigure}{0.47\textwidth}
\centering
\includegraphics[width=\linewidth]{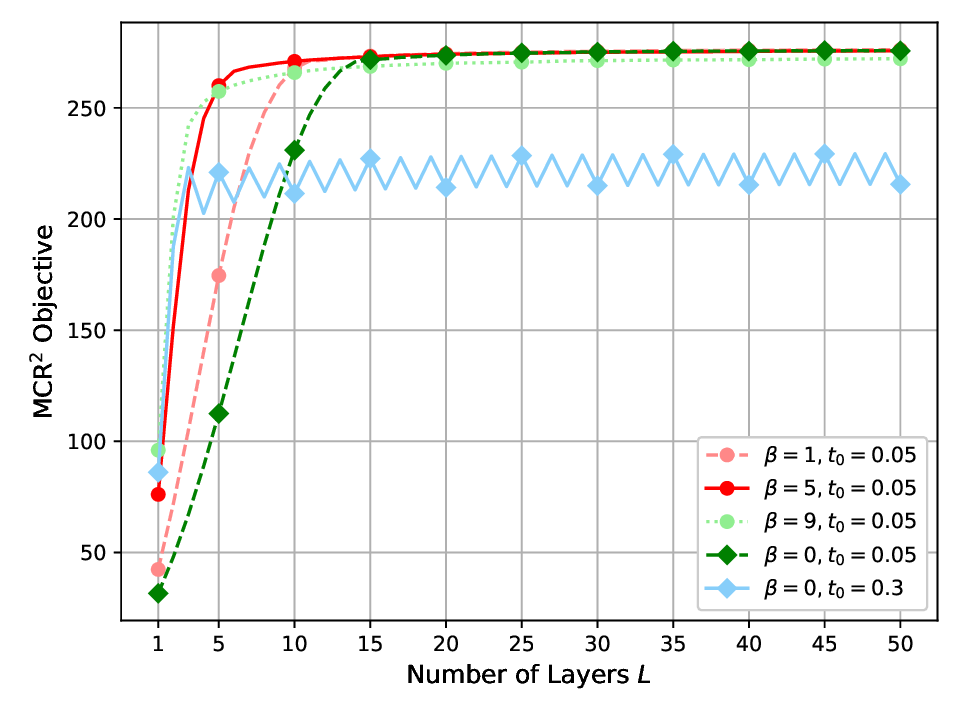}
\caption{CIFAR-10}
\label{fig:12_b}
\end{subfigure}
\hfill
\begin{subfigure}{0.47\textwidth}
\centering
\includegraphics[width=\linewidth]{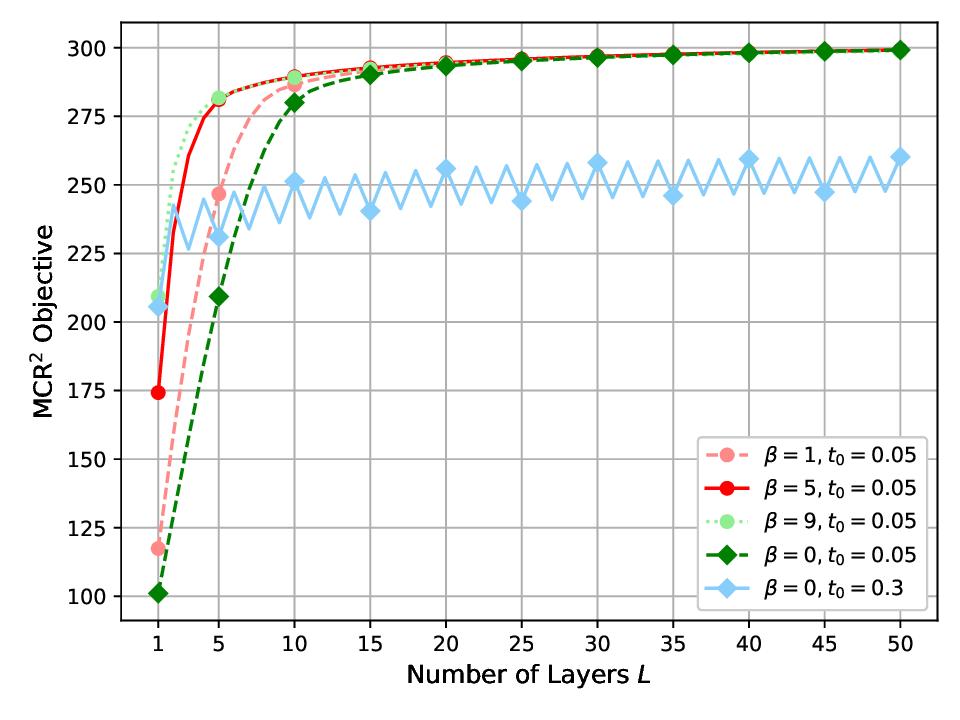}
\caption{CIFAR-100}
\label{fig:12_c}
\end{subfigure}

\vspace{0.5em}

\begin{subfigure}{0.47\textwidth}
\centering
\includegraphics[width=\linewidth]{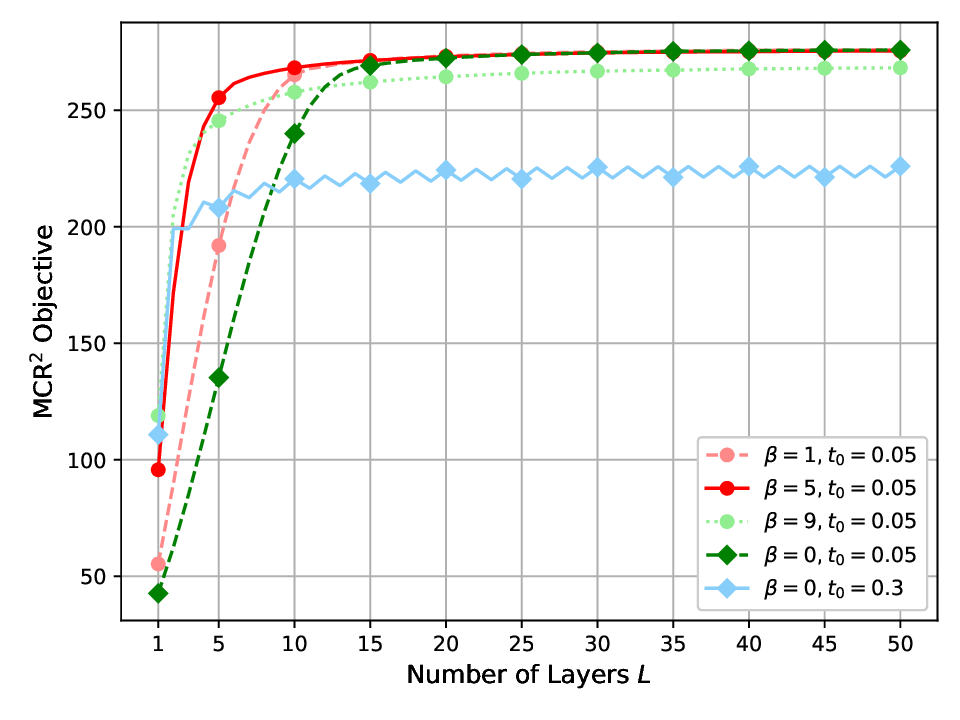}
\caption{CINIC-10}
\label{fig:12_a}
\end{subfigure}

\caption{Effect of parameter $\beta$ on the convergence of the MCR$^2$ objective of LA-ReduNet.}
\label{fig:12}
\end{figure}

Since the adaptive step $t$ in \eqref{Ad_t} depends on both $t_0$ and $\beta$, we first evaluate the influence of $t_0$ on objective convergence. As shown in Fig. \ref{fig:11}, we evaluate the effect of different values of $t_0$ on the objective convergence speed of LA-ReduNet. It can be observed that a relatively small $t_0$ yields a higher MCR$^2$ objective value, which indicates better discriminability of the extracted features. Nevertheless, an excessively small $t_0$ such as $t_0=0.01$ increases the number of layers $L$ required to achieve objective convergence. By contrast, when $t_0$ is excessively large, the MCR$^2$ objective converges to a lower value. Overall, our results demonstrate that a moderate value of $t_0$ ranging from 0.05 to 0.1 provides a favorable trade-off between objective convergence speed and the attained objective value, making it the most suitable choice.

\subsubsection{Impact of Hyperparameter $\beta$ on the Objective Convergence of LA-ReduNet}

\begin{table}[t!]
\centering
\begin{tabular}{lccc}
\toprule
Weight coefficient $\mu$ & CIFAR-10 & CIFAR-100 & CINIC-10 \\
\midrule
$10^{-1}$ & 44.71\% & 9.28\% & 30.13\% \\
$10^{-2}$ & 84.66\% & 55.97\% & 72.86\% \\
$\mathbf{10^{-3}}$ & \textbf{87.58\%} & \textbf{62.78\%} & \textbf{76.79\%} \\
$10^{-4}$ & 86.48\% & 61.84\% & 74.49\% \\
$0$ & 86.18\% & 61.67\% & 75.20\% \\
\bottomrule
\end{tabular}
\caption{Ablation study on the weight coefficient $\mu$ in terms of classification accuracy on CIFAR-10, CIFAR-100, and CINIC-10.}
\label{tab:2}
\end{table}

\begin{figure}[t]
\centering
\captionsetup[subfigure]{labelformat=parens}

\begin{subfigure}{0.47\textwidth}
\centering
\includegraphics[width=\linewidth]{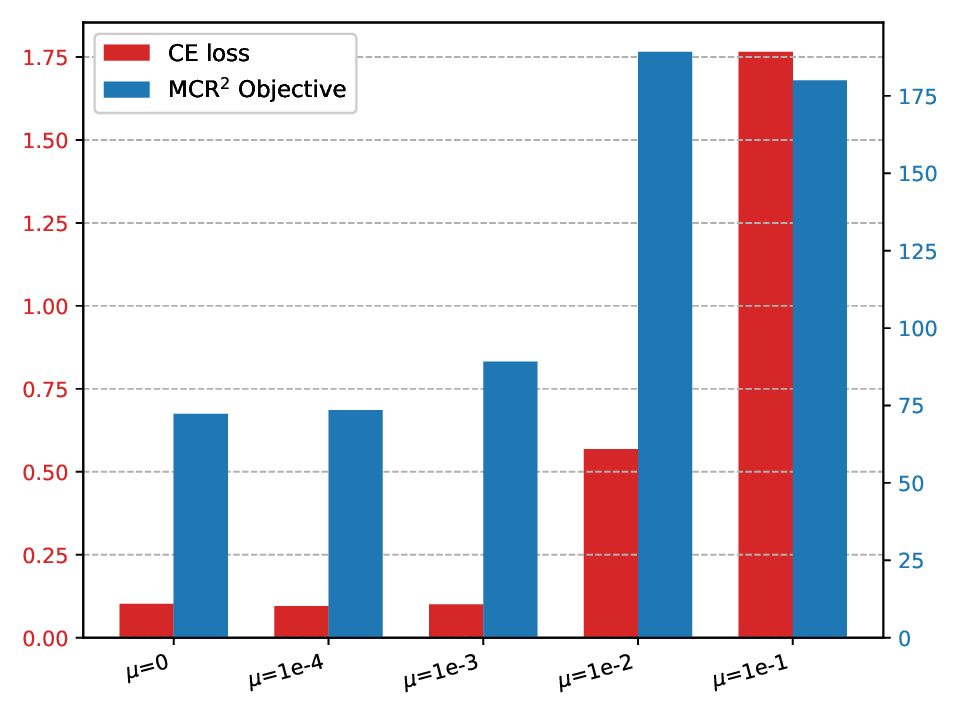}
\caption{CIFAR-10}
\label{fig:13_a}
\end{subfigure}
\hfill
\begin{subfigure}{0.47\textwidth}
\centering
\includegraphics[width=\linewidth]{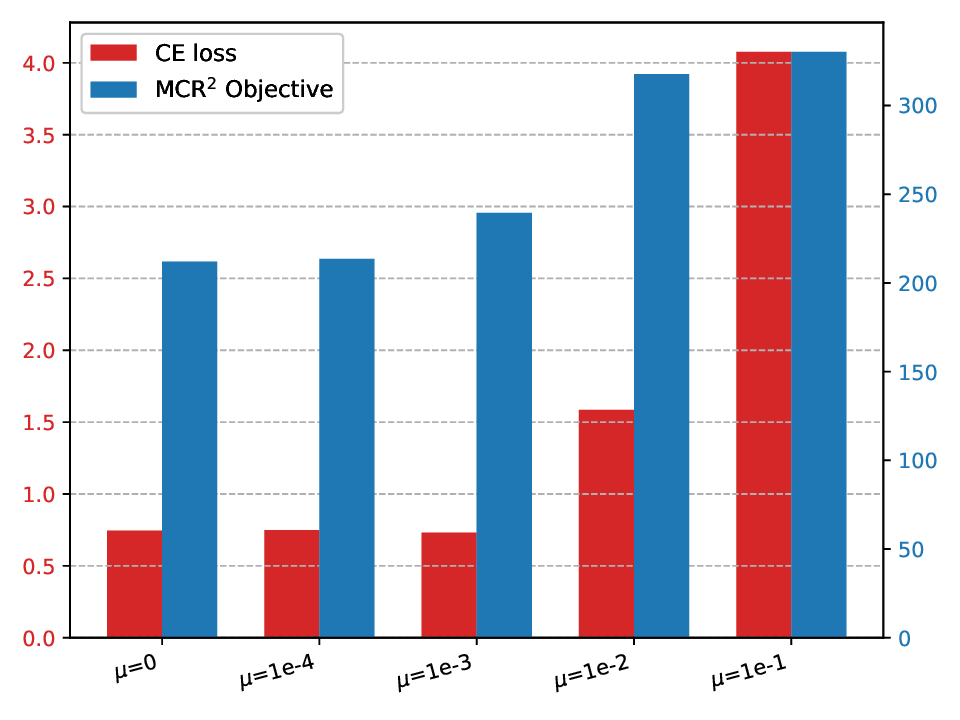}
\caption{CIFAR-100}
\label{fig:13_b}
\end{subfigure}

\vspace{0.5em}

\begin{subfigure}{0.47\textwidth}
\centering
\includegraphics[width=\linewidth]{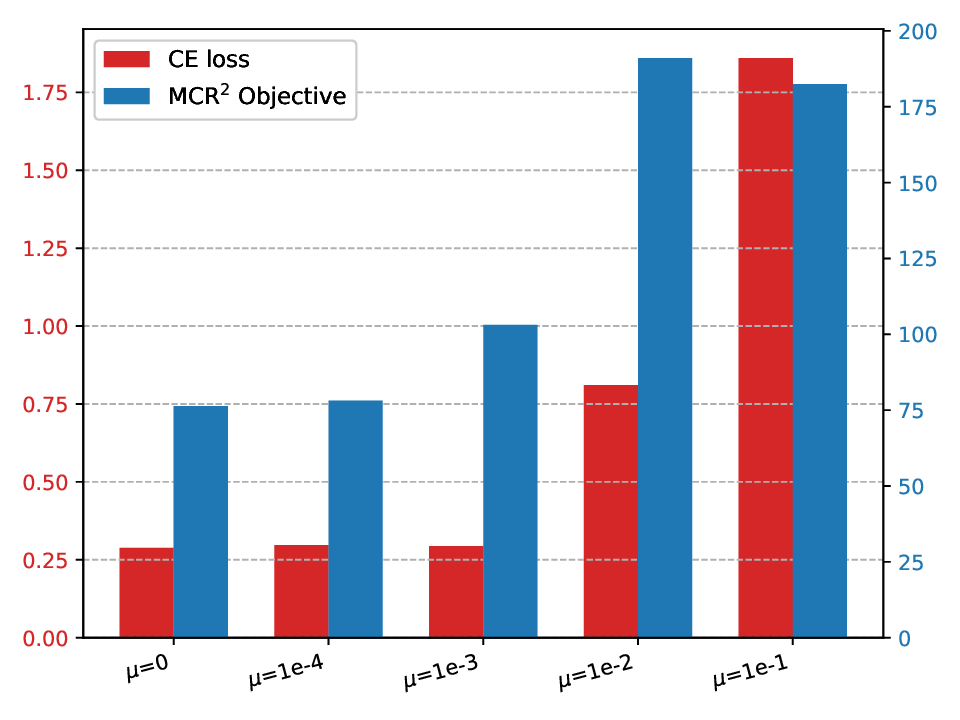}
\caption{CINIC-10}
\label{fig:13_c}
\end{subfigure}

\caption{Ablation study on the weight coefficient $\mu$, illustrating how the final CE loss and MCR$^2$ objective value vary with $\mu$.}
\label{fig:13}
\end{figure}

In this subsection, we evaluate the effect of hyperparameter $ \beta $ on the objective convergence speed of LA-ReduNet. As shown in Fig. \ref{fig:12}, it can be observed that $\beta$ effectively accelerates the growth rate of the MCR$^2$ objective during the training of LA-ReduNet, and this advantage is particularly prominent at shallow depths where $L \leq 15$. Furthermore, we observe that increasing the value of $\beta$ barely affects the peak value of MCR$^2$ that LA-ReduNet can attain upon objective convergence. In particular, with $ \beta=5 $ as a representative case, we compare the adaptive step-size parameter $ t $ against two fixed step-size settings, $t_0=0.05$ and $t_0=0.3$, which respectively correspond to the lower and upper bounds of $ t $ when $ \beta=5 $. As illustrated in Fig. \ref{fig:12}, the adaptive strategy accelerates objective convergence relative to the smaller fixed step size. In contrast, when compared to the larger fixed step size, the adaptive step-size strategy not only achieves more stable optimization but also yields a higher peak value of the MCR$^2$ objective.

In summary, LA-ReduNet benefits from a larger $ \beta $ for faster objective convergence, coupled with a relatively small $t_0$ to safeguard the peak MCR$^2$ value.

\subsubsection{Ablation Study on the Influence of $ \mu $ in the Convolutional Module}
To investigate the impact of the hyperparameter $ \mu $ on the classification performance of the convolutional module, we tune $ \mu $ with all other experimental settings fixed, and evaluate the test set classification accuracy of extracted features under the NS classifier. As shown in Table \ref{tab:2} and Fig. \ref{fig:13}, a lower CE loss is generally associated with higher classification accuracy. Moreover, when the CE loss is kept nearly unchanged, increasing the MCR$^2$ objective value can further enhance classification performance. In particular, for $\mu \in [0, 10^{-3}]$, the MCR$^2$ objective value increases while the CE loss remains almost constant, resulting in improved classification accuracy. By contrast, when $\mu=10^{-2}$, the CE loss becomes substantially higher, and the classification accuracy is markedly lower compared with settings using smaller $\mu$. Thus, we set $10^{-3}$ in the preceding experiments, as it increases the MCR$^2$ objective value without noticeably increasing the CE loss and consequently yields the best overall classification performance.

\section{Conclusions and Discussions}\label{sec:conclusions}
In this paper, we propose LA-ReduNet, a lightweight ReduNet architecture with an adaptive step-size mechanism. On the theoretical side, we establish a finite-termination guarantee for the proposed Riemannian update scheme under explicit sufficient threshold and step-size conditions. Experimentally, compared with the gradient ascent-based ReduNet and AR-ReduNet, LA-ReduNet requires substantially fewer layers for the MCR$^2$ objective to reach a stable value. Simulation results demonstrate that the classification accuracy of LA-ReduNet generally stabilizes within approximately 5--10 layers, while the MCR$^2$ objective reaches a stable value in approximately $35$ layers under the considered experimental settings. We expect that LA-ReduNet can be more effectively applied to a broader range of downstream tasks and can be integrated as a modular component into different neural network architectures. Our experimental results further demonstrate that LA-ReduNet, in conjunction with the convolutional module, yields features with superior classification accuracy and better-separated class boundaries.

\appendix
\section*{Appendices} 
\addcontentsline{toc}{section}{Appendices} 
The proof of Proposition \ref{pro:1} relies on four supporting lemmas, whose detailed proofs are provided in Appendices A--D. The detailed proof of Proposition \ref{pro:1} is then presented in \ref{app:proof-proposition}.
\section{Positive Definiteness and Uniform Eigenvalue Bounds}
\label{app:proof-lemma-first}

\begin{lemma}\label{lem:1}
For any $\bm Z \in \mathbb{R}^{n \times m}$ with unit-norm columns, let $ \alpha \in [0, 1] $ satisfy
\begin{align}
\log\det\left(\alpha\bm I+ \frac{n}{\operatorname{tr}(\bm Z\bm Z^{\rm T})}\bm Z\bm Z^{\rm T}\right)=0. \label{compute:alpha_copy}
\end{align}
Then, the matrix $\alpha \bm{I} + n/\tr(\bm{Z}\bm{Z}^{\rm T})\bm{ZZ}^{\rm T}$ is positive definite. Moreover, all its eigenvalues satisfy
\begin{align}
\frac{1}{n+1} \leq \lambda_{i} \leq n+1,  \qquad i=1,\ldots,n.
\end{align}
Here $ \lambda_{i} $ denotes the $i$-th eigenvalue of matrix $\alpha \bm{I} + n / \tr(\bm{Z}\bm{Z}^{\rm T})\bm{ZZ}^{\rm T}$.
\end{lemma}

\begin{proof}
Let $ \lambda_{i}= \alpha + q_{i} $, where $ q_{i} $ denotes the $i$-th eigenvalue of $ n / \tr(\bm{Z}\bm{Z}^{\rm T})\bm{ZZ}^{\rm T} $. From \eqref{compute:alpha_copy} and $ \tr(\bm{ZZ}^{\rm T})=m $, it follows that
\begin{align}
\det(\alpha\bm{I}+\frac{n}{m}\bm{ZZ}^{\rm T})=\prod_{i=1}^{n}\lambda_{i}=1.
\end{align}
From the fact that $ q_{i} \leq \sum_{s=1}^{n}q_{s} = n $, and $ \alpha \leq 1 $, we obtain $ \lambda_{i} \leq n+1 $.

Moreover, since $\alpha\geq0$ and $q_i\geq0$, we have $\lambda_i\geq0$. Together with $\prod_{i=1}^{n}\lambda_i=1$, this implies that $\lambda_i>0$ for all $i=1,\ldots,n$. For $n=1$, \eqref{compute:alpha_copy} reduces to $\log(\alpha+1)=0$, which gives $\alpha=0$ and hence $\lambda_1=1$. Therefore, the desired bound holds. For $n\geq2$, let
\begin{align}
\lambda_r=\min_{1\leq i\leq n}\lambda_i.
\end{align}
Since $\lambda_r=\alpha+q_r$ and $q_r\geq0$, we have
\begin{align}
\alpha\leq\lambda_r.
\end{align}
Suppose, by contradiction, that
\begin{align}
\lambda_r<\frac{1}{n+1}.
\label{lambda_contradiction}
\end{align}
Since the sum of the eigenvalues of a matrix equals its trace and $\sum_{i=1}^{n}q_i=n$, we have
\begin{align}
\sum_{i=1}^{n}\lambda_i=n\alpha+\sum_{i=1}^{n}q_i=n(1+\alpha).
\end{align}
Then, it follows from $\alpha\leq\lambda_r$ that
\begin{align}
\sum_{i\neq r}\lambda_i&=n(1+\alpha)-\lambda_r \\
&\leq n+(n-1)\lambda_r \\
&< n+\frac{n-1}{n+1}  \\
&< n+1.
\label{remaining_eigenvalue_sum}
\end{align}
By the arithmetic--geometric mean inequality,
\begin{align}
\prod_{i\neq r}\lambda_i
&\leq
\left(
\frac{\sum_{i\neq r}\lambda_i}{n-1}
\right)^{n-1} \\
&<
\left(
\frac{n+1}{n-1}
\right)^{n-1}.
\label{remaining_eigenvalue_product}
\end{align}

Moreover, for all $n\geq2$,
\begin{align}
\left(
\frac{n+1}{n-1}
\right)^{n-1}
\leq n+1.
\label{auxiliary_eigenvalue_bound}
\end{align}
Indeed, the inequality holds directly for $2\leq n\leq6$. For $n\geq7$, using $1+x<e^x$ for $x>0$, we have
\begin{align}
\left(\frac{n+1}{n-1}\right)^{n-1}&=\left(1+\frac{2}{n-1}\right)^{n-1} \\
&<\left(e^{\frac{2}{n-1}}\right)^{n-1} \\
&=e^2 \\
&\leq n+1.
\end{align}

Combining \eqref{remaining_eigenvalue_product} and \eqref{auxiliary_eigenvalue_bound} yields
\begin{align}
\prod_{i\neq r}\lambda_i<n+1.
\end{align}
On the other hand, since $\prod_{i=1}^{n}\lambda_i=1$ and $\lambda_r<1/(n+1)$, we have
\begin{align}
\prod_{i\neq r}\lambda_i
=\frac{1}{\lambda_r}>n+1,
\end{align}
which is a contradiction. Therefore,
\begin{align}
\lambda_r\geq\frac{1}{n+1},
\end{align}
and hence
\begin{align}
\frac{1}{n+1}
\leq
\lambda_i
\leq
n+1,
\qquad
i=1,\ldots,n.
\end{align}

Then, $\alpha \bm{I} + n / \tr(\bm{Z}\bm{Z}^{\rm T})\bm{ZZ}^{\rm T}$ is positive definite, since all its eigenvalues are strictly positive.
\end{proof}

Notice that we can similarly prove that the matrix $ \alpha_{j}\bm{I}+n / \tr(\bm{Z}\bm{\Pi}_{j}\bm{Z}^{\rm T})\bm{Z}\bm{\Pi}_{j}\bm{Z}^{\rm T} $ is positive definite, and its eigenvalues share the same upper and lower bounds as those of $\alpha \bm{I} + n / \tr(\bm{Z}\bm{Z}^{\rm T})\bm{ZZ}^{\rm T}$.

\section{Lipschitz Continuity of the Implicit Parameter Mapping}
\label{app:proof-lemma-second}
At each iteration, $\alpha$ and $\alpha_j$ are first recomputed from the current feature matrix $\bm Z$ and then treated as fixed quantities when computing the update direction with respect to $\bm Z$. To characterize how these recomputed parameters vary with $\bm Z$ across iterations, we denote their values by $\alpha(\bm Z)$ and $\alpha_j(\bm Z^j)$, which are the unique solutions to \eqref{compute:alpha} and \eqref{compute:alphaj}, respectively. The following lemma establishes the Lipschitz continuity of these parameter mappings.

\begin{lemma} \label{lem:2}
Let \(\mathcal{M} = (\mathbb{S}^{n-1})^m\) denote the set of all \(n\times m\) matrices with unit-norm columns. For any \(\bm Z_1, \bm Z_2 \in \mathcal{M}\), there exists a constant $ L_\alpha > 0 $ such that
\begin{align}
|\alpha(\bm{Z}_1)-\alpha(\bm{Z}_2)| &\leq L_{\alpha}\|\bm{Z}_1-\bm{Z}_2\|_F,
\end{align}
where
\begin{align}
L_\alpha = \frac{\pi}{\sqrt{m}}.
\end{align}
\end{lemma}

\begin{proof}
Let $ \bm{Z}(s) $ denote a minimizing geodesic curve on $\mathcal M$ connecting $\bm{Z}_1$ and $\bm{Z}_2$,
\begin{align}
    \bm{Z}(0)=\bm{Z}_{1},  \qquad \bm{Z}(1)=\bm{Z}_2,
\end{align}
where $ s \in [0, 1] $. Define $ \alpha(s)=\alpha(\bm{Z}(s)) $. Then, from the definition in \eqref{compute:alpha}, we have
\begin{align}
    \log \det \left(\alpha(s) \bm{I}+\frac{n}{\tr(\bm{Z}(s)\bm{Z}(s)^{\rm T})}\bm{Z}(s)\bm{Z}(s)^{\rm T}\right)=0. \label{As_0}
\end{align}
Define
\begin{align}
    \bm{A}(s)&=\alpha(s) \bm{I}+\frac{n}{\tr(\bm{Z}(s)\bm{Z}(s)^{\rm T})}\bm{Z}(s)\bm{Z}(s)^{\rm T}, \label{AS} \\
    \bm{B}(s)&=\frac{n}{\tr(\bm{Z}(s)\bm{Z}(s)^{\rm T})} \bm{Z}(s)\bm{Z}(s)^{\rm T}. \label{BS}
\end{align}

By Lemma \ref{lem:1}, $\bm A(s)$ is positive definite. Since
\begin{align}
\frac{\partial \bm A(s)}{\partial \alpha(s)}=\bm I,
\end{align}
we have
\begin{align}
\frac{\partial}{\partial \alpha(s)}
\log\det\bm A(s)
&=
\tr\left(\bm A(s)^{-1}\right)>0.  \label{alpha_partial}
\end{align}
Since $\bm Z(s)$ varies smoothly along the geodesic and \eqref{alpha_partial} shows that the partial derivative with respect to
$\alpha(s)$ is strictly positive, the unique solution $\alpha(s)$ varies differentiably with $s$. Then, it follows from \eqref{As_0} that $\log\det\bm A(s)=0$ for all $s\in[0,1]$. Differentiating $\log\det\bm A(s)$ with respect to $s$ yields
\begin{align}
&\frac{\mathrm{d}}{\mathrm{d}s}\log\det\bm{A}(s) \notag \\
&=\tr\left(\bm{A}(s)^{-1} \frac{\mathrm{d}\bm{A}(s)}{\mathrm{d}s} \right) \\
&=\tr\left(\bm{A}(s)^{-1}(\alpha'(s)\bm{I}+\bm{B}'(s))\right)  \\
&=\tr\left(\bm{A}(s)^{-1}(\alpha'(s)\bm{I}+\frac{n}{m}(\bm{Z}'(s)\bm{Z}(s)^{\rm T}+\bm{Z}(s)\bm{Z}'(s)^{\rm T}))\right) \label{Bm} \\
&=\alpha'(s)\tr\left(\bm{A}(s)^{-1}\right)+\frac{n}{m}\tr\left(\bm{A}(s)^{-1}(\bm{Z}'(s)\bm{Z}(s)^{\rm T}+\bm{Z}(s)\bm{Z}'(s)^{\rm T})\right) \\
&=0,
\end{align}
where \eqref{Bm} follows from $\tr(\bm Z(s)\bm Z(s)^{\rm T})=\tr(\bm Z(s)^{\rm T}\bm Z(s))=m$. Therefore,
\begin{align}
\alpha'(s)&= -\frac{n}{m}\frac{\tr\left(\bm{A}(s)^{-1}(\bm{Z}'(s)\bm{Z}(s)^{\rm T}+\bm{Z}(s)\bm{Z}'(s)^{\rm T})\right)}{\tr\left(\bm{A}(s)^{-1}\right)}  \\
\alpha'(s)&= -\frac{2n}{m}\frac{\tr\left(\bm{Z}(s)^{\rm T}\bm{A}(s)^{-1}\bm{Z}'(s)\right)}{\tr\left(\bm{A}(s)^{-1}\right)}.  \label{alpha_s_de}
\end{align}
By Lemma \ref{lem:1}, $\bm A(s)$ is positive definite. Let $\lambda_i(s)>0$, $i=1,\ldots,n$, denote the eigenvalues of $\bm A(s)$. From \eqref{As_0}, we have
\begin{align}
\prod_{i=1}^{n}\lambda_i(s)=1.
\end{align}
Since the eigenvalues of $\bm A(s)^{-1}$ are $1/\lambda_i(s)$, the arithmetic--geometric mean inequality gives
\begin{align}
\tr\left(\bm A(s)^{-1}\right)
&=
\sum_{i=1}^{n}\frac{1}{\lambda_i(s)} \notag\\
&\geq
n\left(
\prod_{i=1}^{n}\frac{1}{\lambda_i(s)}
\right)^{1/n} \notag\\
&=n.
\label{numerator_1}
\end{align}
Moreover, we obtain
\begin{align}
&|\tr\left(\bm{Z}(s)^{\rm T}\bm{A}(s)^{-1}\bm{Z}'(s)\right)|  \notag \\
&\leq \|\bm{A}(s)^{-1}\bm{Z}(s)\|_F\|\bm{Z}'(s)\|_F
\label{Cauchy--Schwarz} \\
&=
\left[
\tr\left(
\bm{Z}(s)^{\rm T}\bm{A}(s)^{-2}\bm{Z}(s)
\right)
\right]^{1/2}
\|\bm{Z}'(s)\|_F
\label{Frobenius_identity} \\
&=
\left[
\tr\left(
\bm{A}(s)^{-2}\bm{Z}(s)\bm{Z}(s)^{\rm T}
\right)
\right]^{1/2}
\|\bm{Z}'(s)\|_F \\
&=
\left[
\frac{m}{n}
\tr\left(
\bm{A}(s)^{-2}
\left(\bm{A}(s)-\alpha(s)\bm I\right)
\right)
\right]^{1/2}
\|\bm{Z}'(s)\|_F \\
&=
\left\{
\frac{m}{n}
\left[
\tr\left(\bm{A}(s)^{-1}\right)
-\alpha(s)\tr\left(\bm{A}(s)^{-2}\right)
\right]
\right\}^{1/2}
\|\bm{Z}'(s)\|_F \\
&\leq
\sqrt{
\frac{m}{n}
\tr\left(\bm{A}(s)^{-1}\right)
}
\|\bm{Z}'(s)\|_F.
\label{denominator_1}
\end{align}

Step \eqref{Cauchy--Schwarz} follows from the Cauchy-Schwarz inequality for the Frobenius inner product. Step \eqref{Frobenius_identity} follows from the definition of the Frobenius norm and the symmetry of $\bm A(s)^{-1}$. Step \eqref{denominator_1} follows from $\alpha(s)\geq0$ and $\tr(\bm A(s)^{-2})\geq0$. Thus, combining \eqref{numerator_1} and \eqref{denominator_1}, we obtain an upper bound for $ |\alpha'(s)| $ as follows:
\begin{align}
|\alpha'(s)|
&=
\frac{2n}{m}
\frac{
|\tr\left(\bm{Z}(s)^{\rm T}\bm{A}(s)^{-1}\bm{Z}'(s)\right)|
}{
\tr\left(\bm{A}(s)^{-1}\right)
} \\
&\leq
\frac{2n}{m}
\frac{
\sqrt{
\frac{m}{n}
\tr\left(\bm{A}(s)^{-1}\right)
}
}{
\tr\left(\bm{A}(s)^{-1}\right)
}
\|\bm{Z}'(s)\|_F \\
&=
2\sqrt{\frac{n}{m}}
\frac{1}{
\sqrt{\tr\left(\bm{A}(s)^{-1}\right)}
}
\|\bm{Z}'(s)\|_F \\
&\leq
\frac{2}{\sqrt{m}}
\|\bm{Z}'(s)\|_F.
\label{alpha_derivative_bound}
\end{align}
The above result bounds the derivative of $\alpha(s)$ along the geodesic curve. Therefore, by integrating over $s$, we obtain
\begin{align}
\alpha(\bm Z_{2})-\alpha(\bm Z_{1})
&=
\int_{0}^{1}\alpha'(s)\mathrm{d}s \\
|\alpha(\bm Z_{1})-\alpha(\bm Z_{2})|
&=
\left|
\int_{0}^{1}\alpha'(s)\mathrm{d}s
\right| \\
&\leq
\int_{0}^{1}|\alpha'(s)|\mathrm{d}s
\label{integrals} \\
&\leq
\frac{2}{\sqrt m}
\int_{0}^{1}
\|\bm Z'(s)\|_F
\mathrm{d}s \\
&=
\frac{2}{\sqrt m}
d_{\mathcal M}(\bm Z_{1},\bm Z_{2}),
\end{align}
where $d_{\mathcal{M}}(\bm{Z}_{1}, \bm{Z}_{2})$ denotes the geodesic distance on the product manifold $\mathcal M=(\mathbb S^{n-1})^m$. Step \eqref{integrals} follows from the triangle inequality for integrals. Since $\mathcal{M}=(\mathbb{S}^{n-1})^m$ is a product of unit spheres, the geodesic distance can be bounded by the Euclidean distance as
\begin{align}
d_{\mathcal M}(\bm Z_1,\bm Z_2) \leq \frac{\pi}{2} \|\bm Z_1-\bm Z_2\|_F.
\end{align}
Thus,
\begin{align}
|\alpha(\bm{Z}_{1})-\alpha(\bm{Z}_{2})| &\leq \frac{\pi}{\sqrt{m}} \|\bm Z_1-\bm Z_2\|_F.
\end{align}
By setting
\begin{align}
L_\alpha = \frac{\pi}{\sqrt{m}},
\end{align}
the conclusion is proved.
\end{proof}
In addition, analogous results can be derived for the parameter $\alpha_j(\bm Z^j)$. That is,
\begin{align}
|\alpha_j(\bm{Z}_{1}^j)-\alpha_j(\bm{Z}_{2}^j)| \leq \frac{\pi}{\sqrt{m_j}} \|\bm Z_1^j-\bm Z_2^j\|_F,
\end{align}
where $\bm{Z}_1^j, \bm{Z}_2^j \in \mathbb{R}^{n \times m_j}$, and $ m_j $ denotes the number of samples in the $j$-th class. It follows that the constant term satisfies
\begin{align}
L_{\alpha_j} = \frac{\pi}{\sqrt{m_j}}.
\end{align}
Since $\bm\Pi_j$ retains the columns corresponding to the $j$-th class and sets all other columns to zero, we have
\begin{align}
\|\bm Z_1^j-\bm Z_2^j\|_F
&=\|(\bm Z_1-\bm Z_2)\bm\Pi_j\|_F \\
&\leq\|\bm Z_1-\bm Z_2\|_F \|\bm\Pi_j\|_2 \\
&=\|\bm Z_1-\bm Z_2\|_F, \label{alpha_j_inequality}
\end{align}
where \eqref{alpha_j_inequality} follows from $\|\bm\Pi_j\|_2=1$. Therefore,
\begin{align}
|\alpha_j(\bm Z_1^j)-\alpha_j(\bm Z_2^j)|
\leq \frac{\pi}{\sqrt{m_j}}
\|\bm Z_1-\bm Z_2\|_F.
\end{align}

Furthermore, by applying the same argument to the $j$-th class, the derivative of the corresponding implicit parameter along the geodesic satisfies
\begin{align}
|\alpha_j'(s)|
\leq
\frac{2}{\sqrt{m_j}}
\|(\bm Z^j)'(s)\|_F
\leq
\frac{2}{\sqrt{m_j}}
\|\bm Z'(s)\|_F.
\label{alphaj_derivative_bound}
\end{align}

\section{Lipschitz Continuity of the Inverse Matrices in the MCR$^2$ Objective}
\label{app:proof-lemma-third}
Based on Lemma \ref{lem:2}, we further establish the Lipschitz continuity of the inverse matrices appearing in the MCR$^2$ update with respect to $\bm Z$.
\begin{lemma} \label{lem:3}
For any $\bm Z_1,\bm Z_2 \in\mathcal M$, where $\mathcal M=(\mathbb S^{n-1})^m$, there exists a constant $ L_{\rm inv} > 0 $ such that
\begin{align}
&
\left\|
\left(
\alpha(\bm Z_1)\bm I+
\frac{n}{m\epsilon^2}\bm Z_1\bm Z_1^{\rm T}
\right)^{-1}
-
\left(
\alpha(\bm Z_2)\bm I+
\frac{n}{m\epsilon^2}\bm Z_2\bm Z_2^{\rm T}
\right)^{-1}
\right\|_F
\nonumber\\
&\qquad\leq
L_{\rm inv}
\|\bm Z_1-\bm Z_2\|_F,
\end{align}
where
\begin{align}
L_{\rm inv}
=\frac{(n+1)^2}{\sqrt{m}}
\left(\pi\sqrt{n}+\frac{2n}{\epsilon^2}\right).
\label{L_inv}
\end{align}
\end{lemma}

\begin{proof}
Let
\begin{align}
\bm Q_1&=\alpha(\bm Z_1)\bm I+\frac{n}{m\epsilon^2}\bm Z_1\bm Z_1^{\rm T},  \label{def_Q1} \\
\bm Q_2&=\alpha(\bm Z_2)\bm I+\frac{n}{m\epsilon^2}\bm Z_2\bm Z_2^{\rm T}.  \label{def_Q2}
\end{align}
Using the identity for the difference of two inverse matrices and the norm inequality for matrix products, we obtain
\begin{align}
\|\bm Q_1^{-1}-\bm Q_2^{-1}\|_F
&=
\|\bm Q_1^{-1}(\bm Q_2-\bm Q_1)\bm Q_2^{-1}\|_F \\
&\leq
\|\bm Q_1^{-1}\|_2
\|\bm Q_2-\bm Q_1\|_F
\|\bm Q_2^{-1}\|_2.   \label{Q_sum}
\end{align}
The eigenvalues of $\bm{Q}_i$ are bounded below by a positive constant. Specifically,
\begin{align}
\bm Q_i - \left(\alpha(\bm Z_i)\bm I+ \frac{n}{\tr(\bm Z_i\bm Z_i^{\rm T})}\bm Z_i\bm Z_i^{\rm T}\right)
&= \left( \frac{n}{m\epsilon^2} - \frac{n}{m} \right) \bm{Z}_i \bm{Z}_i^{\rm T} \\
&= \frac{n}{m} \left( \frac{1}{\epsilon^2} - 1 \right) \bm{Z}_i \bm{Z}_i^{\rm T}.
\end{align}
Since $ 0 < \epsilon^2 < \tr(\bm{Z}_{i}\bm{Z}_{i}^{\rm T}) / m = 1 $ and $ \bm{Z}_{i}\bm{Z}_{i}^{\rm T} \succeq 0 $, we have
\begin{align}
\frac{n}{m} \left( \frac{1}{\epsilon^2} - 1 \right) \bm{Z}_i \bm{Z}_i^{\rm T} \succeq 0.
\end{align}
Thus,
\begin{align}
\bm Q_i \succeq \left(\alpha(\bm Z_i)\bm I+ \frac{n}{\tr(\bm Z_i\bm Z_i^{\rm T})}\bm Z_i\bm Z_i^{\rm T}\right). \label{eigen_Qi}
\end{align}
By Lemma \ref{lem:1} and \eqref{eigen_Qi}, we have
\begin{align}
\lambda_{\min}(\bm Q_i)
&\geq
\lambda_{\min}\left(
\alpha(\bm Z_i)\bm I+
\frac{n}{\tr(\bm Z_i\bm Z_i^{\rm T})}
\bm Z_i\bm Z_i^{\rm T}
\right)\\
&\geq
\frac{1}{n+1}.
\end{align}
Then, it follows that
\begin{align}
\|\bm Q_i^{-1}\|_2
\leq
n+1,
\quad i=1,2.
\label{Q_norm}
\end{align}
It remains to bound $\|\bm Q_2-\bm Q_1\|_F$. We have
\begin{align}
\bm Q_2-\bm Q_1=& (\alpha(\bm Z_2)-\alpha(\bm Z_1))\bm I+\frac{n}{m\epsilon^2}(\bm Z_2\bm Z_2^{\rm T}-\bm Z_1\bm Z_1^{\rm T}).
\end{align}
For the first term, Lemma \ref{lem:2} gives
\begin{align}
\|(\alpha(\bm Z_2)-\alpha(\bm Z_1))\bm I\|_F &= \sqrt n|\alpha(\bm Z_2)-\alpha(\bm Z_1)| \\
&\leq\sqrt n L_\alpha\|\bm Z_2-\bm Z_1\|_F.    \label{first_term}
\end{align}
For the second term, we have
\begin{align}
\|\bm Z_2\bm Z_2^{\rm T}-\bm Z_1\bm Z_1^{\rm T}\|_F
&=
\|(\bm Z_2-\bm Z_1)\bm Z_2^{\rm T}
+\bm Z_1(\bm Z_2-\bm Z_1)^{\rm T}\|_F
\\
&\leq
\|(\bm Z_2-\bm Z_1)\bm Z_2^{\rm T}\|_F
+\|\bm Z_1(\bm Z_2-\bm Z_1)^{\rm T}\|_F
\\
&\leq
2\sqrt m\|\bm Z_2-\bm Z_1\|_F, \label{second_term}
\end{align}
where the unit-norm column constraint is used. Therefore, combining \eqref{first_term} and \eqref{second_term}, it follows that
\begin{align}
\|\bm Q_2-\bm Q_1\|_F &\leq \|(\alpha(\bm Z_2)-\alpha(\bm Z_1))\bm I\|_F+\frac{n}{m\epsilon^2}\|(\bm Z_2\bm Z_2^{\rm T}-\bm Z_1\bm Z_1^{\rm T})\|_F  \\
&\leq \left(\sqrt n L_\alpha+\frac{2n\sqrt m}{m\epsilon^2}\right)\|\bm Z_2-\bm Z_1\|_F.  \label{Q_F}
\end{align}
Finally, combining \eqref{Q_sum}, \eqref{Q_norm}, and \eqref{Q_F}, it can be obtained that
\begin{align}
&\|\bm Q_1^{-1}-\bm Q_2^{-1}\|_F  \notag \\
&\leq (n+1)^2
\left(
\sqrt n L_\alpha+\frac{2n\sqrt m}{m\epsilon^2}
\right)
\|\bm Z_1-\bm Z_2\|_F \\
&=
\frac{(n+1)^2}{\sqrt m}
\left(
\pi\sqrt n+\frac{2n}{\epsilon^2}
\right)
\|\bm Z_1-\bm Z_2\|_F \\
&=
L_{\rm inv}
\|\bm Z_1-\bm Z_2\|_F.
\end{align}
where $L_{\rm inv} $ is defined as in \eqref{L_inv}. This completes the proof.
\end{proof}
Similarly, applying Lemma \ref{lem:3} to the class-wise submatrices and using $\|\bm Z_1^j-\bm Z_2^j\|_F \leq\|\bm Z_1-\bm Z_2\|_F$, we obtain
\begin{align}
&
\left\|
\left(
\alpha_j(\bm Z_1^j)\bm I
+
\frac{n}{\tr(\bm\Pi_j)\epsilon^2}
\bm Z_1\bm\Pi_j\bm Z_1^{\rm T}
\right)^{-1}
-
\left(
\alpha_j(\bm Z_2^j)\bm I
+
\frac{n}{\tr(\bm\Pi_j)\epsilon^2}
\bm Z_2\bm\Pi_j\bm Z_2^{\rm T}
\right)^{-1}
\right\|_F
\nonumber\\
&\qquad\leq
L^j_{{\rm inv}}
\|\bm Z_1-\bm Z_2\|_F .
\end{align}
where
\begin{align}
L^j_{{\rm inv}}=\frac{(n+1)^2}{\sqrt{m_j}}\left(\pi\sqrt{n}+\frac{2n}{\epsilon^2}\right).
\end{align}
The above results establish the Lipschitz continuity of the inverse-matrix mappings induced by the recomputed parameters $\alpha(\bm{Z})$ and $\alpha_j(\bm{Z}^j)$ with respect to $\bm{Z}$.

\section{Lipschitz Continuity of the Riemannian Update Mapping}
\label{app:proof-lemma-fourth}
Based on Lemmas \ref{lem:1} and \ref{lem:3}, the Lipschitz continuity of the Riemannian update mapping on the product manifold is established. At each $\bm{Z}$, the parameters $\alpha(\bm{Z})$ and $\alpha_j(\bm{Z}^j)$ are first recomputed and then treated as fixed when differentiating the MCR$^2$ objective with respect to $\bm{Z}$. The resulting Euclidean update direction is given by
\begin{align}
\bm{G}(\bm{Z})=&\frac{n}{m\epsilon^2}\left(\alpha(\bm Z)\bm I+\frac{n}{m\epsilon^2}\bm{Z}\bm{Z}^{\rm T}\right)^{-1}\bm{Z}  \notag \\
&-\sum_{j=1}^{k}\frac{n}{m\epsilon^2}\left(\alpha_j(\bm Z^j)\bm I+\frac{n}{\tr(\bm{\Pi}_j)\epsilon^2}\bm{Z}\bm{\Pi}_{j}\bm{Z}^{\rm T}\right)^{-1}\bm{Z}\bm{\Pi}_{j}.
\end{align}
This yields the following lemma.
\begin{lemma}\label{lem:4}
For any $\bm Z_1,\bm Z_2\in\mathcal M=(\mathbb S^{n-1})^m$, there exists a constant
$L_{\rm grad}>0$ such that
\begin{align}
\|\bm{G}_{\rm T}(\bm Z_1)-\bm{G}_{\rm T}(\bm Z_2)\|_F
\leq L_{\rm grad} \|\bm Z_1-\bm Z_2\|_F,
\end{align}
where $ \bm{G}_{\rm T}(\bm Z_1), \bm{G}_{\rm T}(\bm Z_2) $ denote the Riemannian update directions at $\bm{Z}_1$ and $\bm{Z}_2$, respectively.
\end{lemma}

\begin{proof}
The Riemannian update direction on the product of spheres is obtained by projecting the Euclidean update direction onto the tangent space. It follows from \eqref{gt_formula} that
\begin{align}
\bm{G}_{\rm T}(\bm Z) = \bm{G}(\bm{Z}) - \bm{Z} \operatorname{diag}(\bm Z^{\rm T} \bm{G}(\bm{Z})).
\end{align}
For any $\bm Z_1,\bm Z_2\in\mathcal M$, we have
\begin{align}
\|\bm G_{\rm T}(\bm Z_1)-\bm G_{\rm T}(\bm Z_2)\|_F\leq&\|\bm G(\bm Z_1)-\bm G(\bm Z_2)\|_F+\|\bm Z_1\operatorname{diag}(\bm Z_1^{\rm T}\bm G(\bm Z_1))  \notag \\
&-\bm Z_2\operatorname{diag}(\bm Z_2^{\rm T}\bm G(\bm Z_2))\|_F. \label{grad_Gt}
\end{align}

We first establish a Lipschitz bound for the Euclidean update direction $ \bm{G}(\bm{Z}) $. Define $ \bm Q_i^j=\alpha_j(\bm Z_i^j)\bm I+ n / (\tr(\bm{\Pi}_j)\epsilon^2) \bm{Z}_i \bm{\Pi}_j \bm{Z}_i^{\rm T} $. Combining \eqref{def_Q1}, \eqref{def_Q2}, and the triangle inequality for the Frobenius norm yields the following upper bound:
\begin{align}
&\big\|\bm G(\bm Z_1)-\bm G(\bm Z_2)\big\|_F  \notag \\
&= \frac{n}{m\epsilon^2} \Bigg\|
\bm{Q}_1^{-1}\bm{Z}_1 - \bm{Q}_2^{-1}\bm{Z}_2
-\sum_{j=1}^k \Big((\bm{Q}_1^j)^{-1}\bm{Z}_1\bm{\Pi}_j - (\bm{Q}_2^j)^{-1}\bm{Z}_2\bm{\Pi}_j\Big)
\Bigg\|_F \\
&\leq \frac{n}{m\epsilon^2} \Big\| \bm{Q}_1^{-1}\bm{Z}_1 - \bm{Q}_2^{-1}\bm{Z}_2 \Big\|_F + \frac{n}{m\epsilon^2}\sum_{j=1}^k \Big\| (\bm{Q}_1^j)^{-1}\bm{Z}_1\bm{\Pi}_j - (\bm{Q}_2^j)^{-1}\bm{Z}_2\bm{\Pi}_j \Big\|_F. \label{grad_G}
\end{align}
By analyzing the first term in \eqref{grad_G}, we can obtain
\begin{align}
&\frac{n}{m\epsilon^2}\|\bm Q_1^{-1}\bm Z_1-\bm Q_2^{-1}\bm Z_2\|_F  \notag \\
&=\frac{n}{m\epsilon^2}\left\|\bm Q_1^{-1}(\bm Z_1-\bm Z_2)+(\bm Q_1^{-1}-\bm Q_2^{-1})\bm Z_2\right\|_F \\
&\leq\frac{n}{m\epsilon^2}\left(\|\bm Q_1^{-1}(\bm Z_1-\bm Z_2)\|_F+\|(\bm Q_1^{-1}-\bm Q_2^{-1})\bm Z_2\|_F\right)  \\
&\leq\frac{n}{m\epsilon^2}\left(\|\bm Q_1^{-1}\|_2\|\bm Z_1-\bm Z_2\|_F+\|\bm Q_1^{-1}-\bm Q_2^{-1}\|_F\|\bm Z_2\|_2\right)\\
&\leq\frac{n}{m\epsilon^2}\left(\|\bm Q_1^{-1}\|_2+\sqrt m\,L_{\rm inv}\right)\|\bm Z_1-\bm Z_2\|_F. \label{LE_first_term_1} \\
&\leq
\frac{n}{m\epsilon^2}
\biggl[
(n+1)
+
(n+1)^2
\left(
\pi\sqrt{n}
+
\frac{2n}{\epsilon^2}
\right)
\biggr]
\|\bm Z_1-\bm Z_2\|_F
\label{LE_first_term_2} \\
&=
\frac{n(n+1)}{m\epsilon^2}
\biggl[
1
+
(n+1)
\left(
\pi\sqrt{n}
+
\frac{2n}{\epsilon^2}
\right)
\biggr]
\|\bm Z_1-\bm Z_2\|_F \\
&=
L_{\rm E}
\|\bm Z_1-\bm Z_2\|_F,
\label{LE}
\end{align}
where
\begin{align}
L_{\rm E}=\frac{n(n+1)}{m\epsilon^2}\biggl[1+(n+1)\left(\pi\sqrt{n}+\frac{2n}{\epsilon^2}\right)\biggr].
\end{align}
is a constant. Step \eqref{LE_first_term_1} follows from Lemma \ref{lem:3} and the fact that $\|\bm Z_2\|_2\leq\|\bm Z_2\|_F=\sqrt{m}$, while Step \eqref{LE_first_term_2} follows from Lemma \ref{lem:1} and the definition of $L_{\rm inv}$
in Lemma \ref{lem:3}. Similarly, from an analysis of the second term, we can obtain
\begin{align}
\frac{n}{m\epsilon^2}\sum_{j=1}^k \Big\| (\bm{Q}_1^j)^{-1}\bm{Z}_1\bm{\Pi}_j - (\bm{Q}_2^j)^{-1}\bm{Z}_2\bm{\Pi}_j \Big\|_F
\leq k L_{\rm E}  \|\bm Z_1-\bm Z_2\|_F.  \label{KLE}
\end{align}
Here, we use $\|\bm Z\bm\Pi_j\|_2\leq\sqrt{m_j}$ and $\sqrt{m_j}L_{\rm inv}^j=\sqrt m\,L_{\rm inv}$, so that each class-wise term is bounded by $L_{\rm E}\|\bm Z_1-\bm Z_2\|_F$.

Then the second term in \eqref{grad_Gt} is analyzed. For notational simplicity, set $ \bm{D}_i = \mathrm{diag}\big(\bm{Z}_i^\mathrm{T} \bm{G}(\bm{Z}_i)\big), \ i = 1,2 $. By the triangle inequality of the Frobenius norm, we have
\begin{align}
\| \bm{Z}_1 \bm{D}_1 - \bm{Z}_2 \bm{D}_2 \|_F&= \| \bm{Z}_1 (\bm{D}_1 - \bm{D}_2) + (\bm{Z}_1 - \bm{Z}_2) \bm{D}_2 \|_F  \\
&\leq \| \bm{Z}_1 (\bm{D}_1 - \bm{D}_2) \|_F + \|(\bm{Z}_1 - \bm{Z}_2) \bm{D}_2 \|_F \\
&\leq \|\bm{Z}_1\|_2 \big\| \bm{D}_1 - \bm{D}_2 \big\|_F + \|\bm{Z}_1 - \bm{Z}_2\|_F \|\bm{D}_2\|_2. \label{D1}
\end{align}
Consider the first term of \eqref{D1},
\begin{align}
&\|\bm{Z}_1\|_2 \big\| \bm{D}_1 - \bm{D}_2 \big\|_F  \notag \\
&\leq \sqrt{m}\big\| \bm{D}_1 - \bm{D}_2 \big\|_F  \\
&\leq \sqrt{m} \left\| \bm{Z}_1^T \bm{G}(\bm{Z}_1) - \bm{Z}_2^T \bm{G}(\bm{Z}_2) \right\|_F   \\
&= \sqrt{m}  \left\| (\bm{Z}_1 - \bm{Z}_2)^T \bm{G}(\bm{Z}_1) + \bm{Z}_2^T \big(\bm{G}(\bm{Z}_1) - \bm{G}(\bm{Z}_2)\big)\right\|_F \\
&\leq \sqrt{m} \| \bm{G}(\bm{Z}_1) \|_2 \| \bm{Z}_1 - \bm{Z}_2 \|_F + \sqrt{m}\| \bm{Z}_2 \|_2 \| \bm{G}(\bm{Z}_1) - \bm{G}(\bm{Z}_2) \|_F \\
&\leq \sqrt{m} \| \bm{G}(\bm{Z}_1) \|_F \| \bm{Z}_1 - \bm{Z}_2 \|_F + \sqrt{m}\| \bm{Z}_2 \|_2 \| \bm{G}(\bm{Z}_1) - \bm{G}(\bm{Z}_2) \|_F. \label{D2}
\end{align}

The upper bound of $\|\bm G(\bm Z_1)\|_F$ can be derived using the inverse-matrix bounds established above.
\begin{align}
\|\bm G(\bm Z_1)\|_F
&=\frac{n}{m\epsilon^2}\left\|\bm Q_1^{-1}\bm Z_1-\sum_{j=1}^{k}(\bm Q_1^j)^{-1}\bm Z_1\bm\Pi_j\right\|_F\\
&\leq\frac{n}{m\epsilon^2}\left(\|\bm Q_1^{-1}\bm Z_1\|_F+\left\|\sum_{j=1}^{k}(\bm Q_1^j)^{-1}\bm Z_1\bm\Pi_j\right\|_F\right)\\
&\leq\frac{n}{m\epsilon^2}\left(\|\bm Q_1^{-1}\|_2\|\bm Z_1\|_F+(n+1)\|\bm Z_1\|_F\right)\label{Qj}\\
&\leq\frac{2n(n+1)}{m\epsilon^2}\|\bm Z_1\|_F\\
&=\frac{2n(n+1)}{\sqrt m\,\epsilon^2}.
\label{G_bound_value}
\end{align}
Step \eqref{Qj} follows from $ \|(\bm Q_1^j)^{-1}\|_2\leq n+1 $, $j=1,2, \dots, k$, together with $\sum_{j=1}^k\bm\Pi_j=\bm I$, and $\bm\Pi_i\bm\Pi_j=\bm 0$ for $i\neq j$. Therefore, combining \eqref{LE} and \eqref{KLE} gives the following explicit upper bound for \eqref{D2}:
\begin{align}
&\sqrt{m}\,\|\bm G(\bm Z_1)\|_F\|\bm Z_1-\bm Z_2\|_F
+\sqrt{m}\,\|\bm Z_2\|_2
\|\bm G(\bm Z_1)-\bm G(\bm Z_2)\|_F
\notag\\
&\qquad\leq
\left(
\frac{2n(n+1)}{\epsilon^2}
+
m(k+1)L_{\rm E}
\right)
\|\bm Z_1-\bm Z_2\|_F.
\label{KE}
\end{align}

To bound the second term in \eqref{D1}, let $\bm z_{2,i}$ and $\bm g_{2,i}$ denote the $i$-th columns of $\bm Z_2$ and $\bm G(\bm Z_2)$, respectively. Since
\begin{align}
\bm D_2
&=
\operatorname{diag}
\left(
\bm Z_2^{\rm T}\bm G(\bm Z_2)
\right)\\
&=
\operatorname{diag}
\left(
\bm z_{2,1}^{\rm T}\bm g_{2,1},
\ldots,
\bm z_{2,m}^{\rm T}\bm g_{2,m}
\right),
\end{align}
the spectral norm of $\bm D_2$ satisfies
\begin{align}
\|\bm D_2\|_2
&=
\max_{1\leq i\leq m}
\left|\bm z_{2,i}^{\rm T}\bm g_{2,i} \right|\\
&\leq
\max_{1\leq i\leq m}
\left(
\|\bm z_{2,i}\|_2
\|\bm g_{2,i}\|_2
\right)
\label{D2_1}\\
&=
\max_{1\leq i\leq m}
\|\bm g_{2,i}\|_2\\
&\leq
\sqrt{
\sum_{i=1}^{m}
\|\bm g_{2,i}\|_2^2
}
\label{D2_2}\\
&=
\|\bm G(\bm Z_2)\|_F\\
&\leq
\frac{2n(n+1)}
{\sqrt m\,\epsilon^2}.
\label{D2_norm}
\end{align}
Step \eqref{D2_1} follows from the Cauchy-Schwarz inequality. Step \eqref{D2_2} holds since each term is nonnegative, and the square of the maximum term does not exceed the sum of the squares of all terms. The last inequality follows by applying the same bound as in \eqref{G_bound_value} to $\bm G(\bm Z_2)$. Consequently,
\begin{align}
\|\bm Z_1-\bm Z_2\|_F\|\bm D_2\|_2
\leq
\frac{2n(n+1)}
{\sqrt m\,\epsilon^2}
\|\bm Z_1-\bm Z_2\|_F.
\label{D2_second}
\end{align}

Finally, combining Equations \eqref{grad_Gt}, \eqref{LE}, \eqref{KLE}, \eqref{KE} and \eqref{D2_second}, the upper bound of $ \|\bm G_{\rm T}(\bm Z_1)-\bm G_{\rm T}(\bm Z_2)\|_F $ is given by
\begin{align}
&\|\bm G_{\rm T}(\bm Z_1)-\bm G_{\rm T}(\bm Z_2)\|_F \notag \\
&\leq
\Bigg[
(m+1)(k+1)L_{\rm E}
+
\frac{2n(n+1)}{\epsilon^2}
\left(
1+\frac{1}{\sqrt m}
\right)
\Bigg]
\|\bm Z_1-\bm Z_2\|_F \\
&=
L_{\rm grad}\|\bm Z_1-\bm Z_2\|_F,
\end{align}
where
\begin{align}
L_{\rm grad}
=
(m+1)(k+1)L_{\rm E}
+
\frac{2n(n+1)}{\epsilon^2}
\left(
1+\frac{1}{\sqrt m}
\right).
\end{align}
This completes the proof.
\end{proof}

\section{Proof of Proposition \ref{pro:1}}
\label{app:proof-proposition}
Now, we are ready to establish Proposition \ref{pro:1}. Fix an arbitrary point $\bm Z=[\bm z_1,\ldots,\bm z_m]\in\mathcal M$ and an arbitrary tangent vector $\bm\xi=[\bm\xi_1,\ldots,\bm\xi_m]$. Since $\bm\xi\in T_{\bm Z}\mathcal M$, we have
\begin{align}
\bm z_i^{\rm T}\bm\xi_i=0, \qquad i=1,\ldots,m.
\label{tangent_condition}
\end{align}
Consider
\begin{align}
\bm Z(s)=\operatorname{Geo}(\bm Z, s\bm\xi), \quad s\in[0,1],
\label{geodesic_Z}
\end{align}
which satisfies
\begin{align}
\bm Z(0)=\bm Z, \quad \bm Z(1)=\operatorname{Geo}(\bm Z, \bm\xi).
\end{align}
For $\bm\xi_i \neq \bm 0$, the $i$-th column of $\bm Z(s)$ is
\begin{align}
\bm z_i(s)=\cos(s\|\bm\xi_i\|_2)\bm z_i + \sin(s\|\bm\xi_i\|_2)\frac{\bm\xi_i}{\|\bm\xi_i\|_2},
\label{column_geodesic}
\end{align}
whereas $\bm z_i(s)=\bm z_i$ when $\bm\xi_i= \bm 0$. Differentiating \eqref{column_geodesic} with respect to $s$ gives
\begin{align}
\frac{\mathrm d\bm z_i(s)}{\mathrm ds} = -\|\bm\xi_i\|_2 \sin\bigl(s\|\bm\xi_i\|_2\bigr)\bm z_i + \bm\xi_i\cos\bigl(s\|\bm\xi_i\|_2\bigr).
\label{column_geodesic_derivative}
\end{align}
Using $\|\bm z_i\|_2=1$ and $\bm z_i^{\rm T}\bm\xi_i=0$, we obtain
\begin{align}
\left\| \frac{\mathrm d\bm z_i(s)}{\mathrm ds} \right\|_2^2
&= \|\bm\xi_i\|_2^2 \sin^2\bigl(s\|\bm\xi_i\|_2\bigr) \|\bm z_i\|_2^2+\|\bm\xi_i\|_2^2\cos^2\bigl(s\|\bm\xi_i\|_2\bigr)
 \notag\\
&\quad-2\|\bm\xi_i\|_2 \sin\bigl(s\|\bm\xi_i\|_2\bigr) \cos\bigl(s\|\bm\xi_i\|_2\bigr)\bm z_i^{\rm T}\bm\xi_i\\
&=\|\bm\xi_i\|_2^2\left[\sin^2\bigl(s\|\bm\xi_i\|_2\bigr)+\cos^2\bigl(s\|\bm\xi_i\|_2\bigr)\right]\\
&=\|\bm\xi_i\|_2^2.
\end{align}
Consequently,
\begin{align}
\left\|
\frac{\mathrm d\bm Z(s)}{\mathrm ds}
\right\|_F^2
&=
\sum_{i=1}^{m}
\left\|
\frac{\mathrm d\bm z_i(s)}{\mathrm ds}
\right\|_2^2\\
&=
\sum_{i=1}^{m}
\|\bm\xi_i\|_2^2\\
&=
\|\bm\xi\|_F^2.  \label{Zsd_ori}
\end{align}
It follows from \eqref{Zsd_ori} that
\begin{align}
\left\| \frac{\mathrm{d}\bm Z(s)}{\mathrm{d}s} \right\|_F = \|\bm\xi\|_F . \label{Zsd}
\end{align}
Moreover, using $ 2\left|\sin\bigl(s\|\bm\xi_i\|_2 / 2\bigr)\right|\leq s\|\bm\xi_i\|_2 $ and \eqref{column_geodesic}, we obtain
\begin{align}
\|\bm Z(s)-\bm Z\|_F^2&=\sum_{i=1}^{m}\|\bm z_i(s)-\bm z_i\|_2^2\\
&=\sum_{i=1}^{m}\left( \left[\cos\bigl(s\|\bm\xi_i\|_2\bigr)-1\right]^2+\sin^2\bigl(s\|\bm\xi_i\|_2\bigr) \right)\\
&=\sum_{i=1}^{m}\bigl(2-2\cos\bigl(s\|\bm\xi_i\|_2\bigr)\bigr)\\
&=\sum_{i=1}^{m}\left(4\sin^2\left(\frac{s\|\bm\xi_i\|_2}{2}\right)\right) \\
&\leq s^2\sum_{i=1}^{m}\|\bm\xi_i\|_2^2 \\
&=s^2\|\bm\xi\|_F^2.  \label{ZsZ_ori}
\end{align}
It follows from \eqref{ZsZ_ori} that
\begin{align}
\|\bm Z(s)-\bm Z\|_F \leq s\|\bm\xi\|_F. \label{ZsZ}
\end{align}
Similarly, it follows from \eqref{column_geodesic_derivative} that
\begin{align}
\left\| \frac{\mathrm{d}\bm z_i(s)}{\mathrm{d}s}-\bm\xi_i \right\|_2^2
&=\|\bm\xi_i\|_2^2\left[\sin^2\big(s\|\bm\xi_i\|_2\big)+\big(\cos(s\|\bm\xi_i\|_2)-1\big)^2\right]\\
&=\|\bm\xi_i\|_2^2\Big(2-2\cos\big(s\|\bm\xi_i\|_2\big)\Big)\\
&=4\|\bm\xi_i\|_2^2
\sin^2\left(\frac{s\|\bm\xi_i\|_2}{2}\right)\\
&\leq s^2\|\bm\xi_i\|_2^4.
\end{align}
Then,
\begin{align}
\left\| \frac{\mathrm{d}\bm Z(s)}{\mathrm{d}s}-\bm\xi \right\|_F^2
&= \sum_{i=1}^{m}\left\| \frac{\mathrm{d}\bm z_i(s)}{\mathrm{d}s}-\bm\xi_i \right\|_2^2 \\
&\leq \sum_{i=1}^{m} s^2 \|\bm\xi_i\|_2^4 \\
&= s^2\sum_{i=1}^{m} \|\bm\xi_i\|_2^4 \\
&\leq s^2\left( \sum_{i=1}^{m} \|\bm\xi_i\|_2^2 \right)^2 \\
&= s^2 \|\bm\xi\|_F^4  \label{Zsxi_ori}
\end{align}
It follows from \eqref{Zsxi_ori} that
\begin{align}
\left\| \frac{\mathrm{d}\bm Z(s)}{\mathrm{d}s}-\bm\xi \right\|_F \leq s \|\bm\xi\|_F^2. \label{Zsxi}
\end{align}

Notice that $\bm Z(s)$, $\alpha(s)$, and $\alpha_j(s)$, $j=1,\ldots,k$, are all functions of $s$. Thus, by the definition of \eqref{composite_objective} and the chain rule, we have
\begin{align}
\frac{\mathrm d}{\mathrm ds}F(\bm Z(s))
&=
\left\langle
\frac{\partial \Delta R}{\partial \bm Z(s)},
\frac{\mathrm d\bm Z(s)}{\mathrm ds}
\right\rangle_F
+
\frac{\partial \Delta R}{\partial \alpha(s)}
\frac{\mathrm d\alpha(s)}{\mathrm ds}
+
\sum_{j=1}^{k}
\frac{\partial \Delta R}{\partial \alpha_j(s)}
\frac{\mathrm d\alpha_j(s)}{\mathrm ds}
\\
&=
\left\langle
\bm G(\bm Z(s)),
\frac{\mathrm d\bm Z(s)}{\mathrm ds}
\right\rangle_F
+
\frac{1}{2}
\tr\left(\bm Q(s)^{-1}\right)\alpha'(s)
\notag\\
&\quad
-
\sum_{j=1}^{k}
\frac{m_j}{2m}
\tr\left((\bm Q^j(s))^{-1}\right)
\alpha_j'(s).
\end{align}
where $\langle\cdot,\cdot\rangle_F$ denotes the Frobenius inner product. Define
\begin{align}
E_{\alpha}(s)
&=
\frac{1}{2}
\tr\left(\bm Q(s)^{-1}\right)\alpha'(s)
-
\sum_{j=1}^{k}
\frac{m_j}{2m}
\tr\left((\bm Q^j(s))^{-1}\right)
\alpha_j'(s). \label{E_alpha_s}
\end{align}
Since $\mathrm d\bm Z(s)/\mathrm ds\in T_{\bm Z(s)}\mathcal M$ and $\bm G_{\rm T}(\bm Z(s))$ is the orthogonal projection of $\bm G(\bm Z(s))$ onto $T_{\bm Z(s)}\mathcal M$, we have
\begin{align}
\frac{\mathrm d}{\mathrm ds}F(\bm Z(s))=\left\langle\bm G_{\rm T}(\bm Z(s)),\frac{\mathrm d\bm Z(s)}{\mathrm ds}\right\rangle_F+E_{\alpha}(s). \label{objective_derivative}
\end{align}

To bound $E_{\alpha}(s)$, we first consider the term associated with $\alpha(s)$. By \eqref{eigen_Qi}, we have $\bm Q(s)\succeq\bm A(s)$, and hence $\bm Q(s)^{-1}\preceq\bm A(s)^{-1}$. Therefore, using the expression for $\alpha'(s)$ derived in \eqref{alpha_s_de}, we obtain
\begin{align}
\frac{1}{2}
\tr\left(\bm Q(s)^{-1}\right)|\alpha'(s)|
&=\frac{n}{m}\frac{\tr\left(\bm Q(s)^{-1}\right)}{\tr\left(\bm A(s)^{-1}\right)}
\left|\tr\left(\bm Z(s)^{\rm T}\bm A(s)^{-1}\bm Z'(s)\right)\right|  \\
&\leq
\frac{n}{m}
\left|
\tr\left(
\bm Z(s)^{\rm T}
\bm A(s)^{-1}
\bm Z'(s)
\right)
\right| \label{chain_alpha_1} \\
&\leq
\frac{n}{m}
\sqrt{
\frac{m}{n}
\tr\left(\bm A(s)^{-1}\right)
}
\|\bm Z'(s)\|_F  \label{chain_alpha_2} \\
&\leq
\frac{n\sqrt{n+1}}{\sqrt m}
\|\bm Z'(s)\|_F. \label{chain_alpha_3}
\end{align}
Step \eqref{chain_alpha_1} follows from $ \tr(\bm Q(s)^{-1}) \leq \tr(\bm A(s)^{-1}) $, step \eqref{chain_alpha_2} from \eqref{denominator_1}, and step \eqref{chain_alpha_3} from Lemma \ref{lem:1}.  For the class-wise terms, we have
\begin{align}
\left|\sum_{j=1}^{k}\frac{m_j}{2m}\tr\left((\bm Q^j(s))^{-1}\right)\alpha_j'(s)\right|  \leq\sum_{j=1}^{k}\frac{m_j}{2m}\tr\left((\bm Q^j(s))^{-1}\right)|\alpha_j'(s)| \label{chain_alphaj_0}
\end{align}
Step \eqref{chain_alphaj_0} follows from the triangle inequality and $ \tr\big((\bm{Q}^j(s))^{-1}\big) > 0 $. Then, applying the same argument as in \eqref{chain_alpha_1}--\eqref{chain_alpha_3} to each class-wise term, we obtain
\begin{align}
&\left|
\sum_{j=1}^{k}
\frac{m_j}{2m}
\tr\left((\bm Q^j(s))^{-1}\right)
\alpha_j'(s)
\right|
\notag\\
&\leq\sum_{j=1}^{k}\frac{n\sqrt{(n+1)m_j}}{m}\,\|(\bm Z^j)'(s)\|_F\\
&=\frac{n\sqrt{n+1}}{m}\sum_{j=1}^{k}\sqrt{m_j}\,\|(\bm Z^j)'(s)\|_F\\
&\leq\frac{n\sqrt{n+1}}{m}\sqrt{\sum_{j=1}^{k}m_j}
\sqrt{\sum_{j=1}^{k}\|(\bm Z^j)'(s)\|_F^2}  \label{chain_alphaj_1} \\
&=\frac{n\sqrt{n+1}}{\sqrt m}\|\bm Z'(s)\|_F. \label{chain_alphaj_2}
\end{align}
Step \eqref{chain_alphaj_1} follows from the Cauchy-Schwarz inequality, while step \eqref{chain_alphaj_2} follows from $\sum_{j=1}^{k}m_j=m$ and $\sum_{j=1}^{k}\|(\bm Z^j)'(s)\|_F^2=\|\bm Z'(s)\|_F^2$, since the class-wise selection matrices $\{\bm\Pi_j\}_{j=1}^{k}$ have mutually disjoint supports and partition all samples.

Combining \eqref{E_alpha_s}, \eqref{chain_alpha_3} and \eqref{chain_alphaj_2}, we obtain
\begin{align}
|E_{\alpha}(s)|
&\leq
\left|
\frac{1}{2}
\tr\left(\bm Q(s)^{-1}\right)\alpha'(s)
\right|+
\left|
\sum_{j=1}^{k}
\frac{m_j}{2m}
\tr\left((\bm Q^j(s))^{-1}\right)
\alpha_j'(s)
\right|  \\
&\leq
\frac{2n\sqrt{n+1}}{\sqrt m}
\|\bm Z'(s)\|_F. \label{Es_bound_0}
\end{align}
Let $ C_{\alpha}=2n\sqrt{n+1}\,m^{-1/2} $. From \eqref{Es_bound_0} and \eqref{Zsd}, we obtain
\begin{align}
E_{\alpha}(s) \geq -C_{\alpha}\|\bm\xi\|_F. \label{Es_bound_1}
\end{align}

Integrating \eqref{objective_derivative} over $s\in[0,1]$ and using $\bm Z(0)=\bm Z$ and $\bm Z(1)=\operatorname{Geo}(\bm Z,\bm\xi)$, we obtain
\begin{align}
F\bigl(\operatorname{Geo}(\bm Z,\bm\xi)\bigr)-F(\bm Z)
&=
\int_0^1
\frac{\mathrm d}{\mathrm ds}F(\bm Z(s))
\,\mathrm ds
\notag\\
&=
\int_0^1
\left[
\left\langle
\bm G_{\rm T}\bigl(\bm Z(s)\bigr),
\frac{\mathrm d\bm Z(s)}{\mathrm ds}
\right\rangle_F
+
E_{\alpha}(s)
\right]
\,\mathrm ds.
\label{objective_integral}
\end{align}

Subtracting
$\langle\bm G_{\rm T}(\bm Z),\bm\xi\rangle_F$
from both sides of \eqref{objective_integral} yields
\begin{align}
&F\bigl(
\operatorname{Geo}(\bm Z,\bm\xi)
\bigr)-F(\bm Z)-\left\langle\bm G_{\rm T}(\bm Z),\bm\xi
\right\rangle_F  \notag\\
&=\int_{0}^{1}
\left[
\left\langle
\bm G_{\rm T}(\bm Z(s)),
\frac{\mathrm d\bm Z(s)}{\mathrm ds}
\right\rangle_F+E_{\alpha}(s)
\right]\mathrm ds - \int_{0}^{1}
\left\langle
\bm G_{\rm T}(\bm Z),\bm\xi
\right\rangle_F
\mathrm ds  \\
&=
\int_{0}^{1}
\Bigg[
\left\langle
\bm G_{\rm T}(\bm Z(s)),
\frac{\mathrm d\bm Z(s)}{\mathrm ds}
\right\rangle_F
-
\left\langle
\bm G_{\rm T}(\bm Z),
\frac{\mathrm d\bm Z(s)}{\mathrm ds}
\right\rangle_F \notag \\
&\qquad+
\left\langle
\bm G_{\rm T}(\bm Z),
\frac{\mathrm d\bm Z(s)}{\mathrm ds}
\right\rangle_F
-
\left\langle
\bm G_{\rm T}(\bm Z),\bm\xi
\right\rangle_F
+
E_{\alpha}(s)
\Bigg]
\mathrm ds  \\
&=
\int_{0}^{1}
\left\langle
\bm G_{\rm T}(\bm Z(s))
-
\bm G_{\rm T}(\bm Z),
\frac{\mathrm d\bm Z(s)}{\mathrm ds}
\right\rangle_F
\mathrm ds  \notag \\
&\quad+
\int_{0}^{1}
\left\langle
\bm G_{\rm T}(\bm Z),
\frac{\mathrm d\bm Z(s)}{\mathrm ds}
-
\bm\xi
\right\rangle_F
\mathrm ds  \notag \\
&\quad+
\int_{0}^{1}
E_{\alpha}(s)\,
\mathrm ds.
\label{smoothness_decomposition}
\end{align}
For the first term in \eqref{smoothness_decomposition}, we have
\begin{align}
&\left|\left\langle\bm G_{\rm T}(\bm Z(s))-\bm G_{\rm T}(\bm Z),\frac{\mathrm{d}\bm Z(s)}{\mathrm{d}s}\right\rangle_F\right|\notag \\
&\leq\|\bm G_{\rm T}(\bm Z(s))-\bm G_{\rm T}(\bm Z)\|_F\left\| \frac{\mathrm{d}\bm Z(s)}{\mathrm{d}s} \right\|_F \label{first_integral_bound_1}\\
&\leq L_{\rm grad}\|\bm Z(s)-\bm Z\|_F\left\| \frac{\mathrm{d}\bm Z(s)}{\mathrm{d}s} \right\|_F \label{first_integral_bound_2} \\
&\leq sL_{\rm grad}\|\bm\xi\|_F^2.
\label{first_integral_bound}
\end{align}
Step \eqref{first_integral_bound_1} follows from the Cauchy-Schwarz inequality, step \eqref{first_integral_bound_2} follows from
Lemma \ref{lem:4}, and step \eqref{first_integral_bound} follows from \eqref{Zsd} and \eqref{ZsZ}. Furthermore, since $\bm G_{\rm T}(\bm Z)$ is the orthogonal projection of $\bm G(\bm Z)$ onto $T_{\bm Z}\mathcal M$, using \eqref{G_bound_value}, we have
\begin{align}
\|\bm G_{\rm T}(\bm Z)\|_F
\leq
\|\bm G(\bm Z)\|_F
\leq
\frac{2n(n+1)}
{\sqrt m\,\epsilon^2}.
\label{Gt_uniform_bound}
\end{align}
For the second term in \eqref{smoothness_decomposition}, combining \eqref{Gt_uniform_bound} with \eqref{Zsxi}, we obtain
\begin{align}
\left|
\left\langle
\bm G_{\rm T}(\bm Z),
\frac{\mathrm{d}\bm Z(s)}{\mathrm{d}s}-\bm\xi
\right\rangle_F
\right|
&\leq
\|\bm G_{\rm T}(\bm Z)\|_F
\left\|
\frac{\mathrm{d}\bm Z(s)}{\mathrm{d}s}-\bm\xi
\right\|_F\\
&\leq
\frac{2n(n+1)s}
{\sqrt m\,\epsilon^2}
\|\bm\xi\|_F^2.
\label{second_integral_bound}
\end{align}

Combining \eqref{Es_bound_1}, \eqref{first_integral_bound}, and \eqref{second_integral_bound}, we obtain
\begin{align}
&F\bigl(\operatorname{Geo}(\bm Z,\bm\xi)\bigr)
-F(\bm Z)
-
\left\langle
\bm G_{\rm T}(\bm Z),\bm\xi
\right\rangle_F
\notag\\
&\geq
-\int_{0}^{1}
s\left[
L_{\rm grad}
+
\frac{2n(n+1)}
{\sqrt{m}\,\epsilon^2}
\right]
\|\bm\xi\|_F^2
\,\mathrm ds
-
\int_{0}^{1}
C_{\alpha}\|\bm\xi\|_F
\,\mathrm ds
\notag\\
&=
-\frac{1}{2}
\left[
L_{\rm grad}
+
\frac{2n(n+1)}
{\sqrt{m}\,\epsilon^2}
\right]
\|\bm\xi\|_F^2
-
C_{\alpha}\|\bm\xi\|_F.
\end{align}
Therefore, setting $L_{\rm s}$ as in \eqref{Lsm_definition} yields \eqref{Riemannian_smoothness}. This completes the proof.

\footnotesize
\bibliographystyle{elsarticle-num}
\bibliography{myreference}

\end{document}